\documentclass{article}

     \usepackage[final]{neurips_2021}

\usepackage[utf8]{inputenc} 
\usepackage[T1]{fontenc}    
\usepackage{hyperref}       
\hypersetup{
	colorlinks = true, 
	urlcolor   = black, 
	linkcolor  = blue, 
	citecolor  = blue 
}
\usepackage{url}            
\usepackage{booktabs}       
\usepackage{amsfonts}       
\usepackage{nicefrac}       
\usepackage{microtype}      
\usepackage{xcolor}         

\usepackage{url}
\usepackage{algorithm}
\usepackage{algorithmic}
\usepackage{amsthm}
\newtheorem{theorem}{Theorem}
\newtheorem{lemma}{Lemma}

\usepackage{ifthen}
\usepackage{thm-restate}
\usepackage{tikz}
\usetikzlibrary{positioning,chains,fit,shapes,calc}

\usepackage{amsmath}
\usepackage{amssymb}
\usepackage{amsfonts}
\usepackage{enumerate}
\usepackage{flushend}
\usepackage{mathrsfs}
\usepackage{subfigure}
\usepackage{booktabs}
\usepackage{makecell}
\usepackage{setspace}
\usepackage{xcolor}
\usepackage{wrapfig}
\usepackage{multicol}
\usepackage{geometry}
\usepackage{bm}

\newcommand{\E}{\mathbb{E}}
\newcommand{\I}{\mathbb{I}}
\newcommand{\R}{\mathbb{R}}
\newcommand{\N}{\mathbb{N}}

\newcommand{\ba}{\boldsymbol{a}}

\newcommand{\be}{\boldsymbol{e}}

\newcommand{\bu}{\boldsymbol{u}}

\newcommand{\bv}{\boldsymbol{v}}
\newcommand{\bw}{\boldsymbol{w}}

\newcommand{\by}{\boldsymbol{y}}
\newcommand{\bz}{\boldsymbol{z}}

\newcommand{\cA}{\mathcal{A}}

\newcommand{\cD}{\mathcal{D}}
\newcommand{\cE}{\mathcal{E}}
\newcommand{\cF}{\mathcal{F}}
\newcommand{\cG}{\mathcal{G}}
\newcommand{\cH}{\mathcal{H}}

\newcommand{\cJ}{\mathcal{J}}
\newcommand{\cK}{\mathcal{K}}
\newcommand{\cL}{\mathcal{L}}
\newcommand{\cM}{\mathcal{M}}
\newcommand{\cN}{\mathcal{N}}

\newcommand{\cR}{\mathcal{R}}

\newcommand{\argmax}{\operatornamewithlimits{argmax}}

\mathchardef\mhyphen="2D
\newcommand{\empiricalmax}{\mathtt{MC \mhyphen Empirical}}
\newcommand{\algsemibandit}{\mathtt{MC \mhyphen UCB}}
\newcommand{\algfullbandit}{\mathtt{MC \mhyphen ETE}}
\newcommand{\ex}{\mathbb{E}}

\newcommand{\bs}{\boldsymbol}

\newcommand{\var}{\textup{Var}}

\newcommand{\unitvec}{\textbf{1}}
\newcommand{\sbr}[1]{\left( #1 \right)}
\newcommand{\mbr}[1]{\left[ #1 \right]}
\newcommand{\lbr}[1]{\left\{ #1 \right\}}
\newcommand{\abr}[1]{\left| #1 \right|}
\newcommand{\quadopt}{\textsf{QuadOpt}}

\newcommand{\compilehidecomments}{true}
\ifthenelse{ \equal{\compilehidecomments}{true} }{%
	\newcommand{\longbo}[1]{}
	\newcommand{\zhixuan}[1]{}
	\newcommand{\siwei}[1]{}
	\newcommand{\yihan}[1]{}
}{
	\newcommand{\longbo}[1]{{\color{red}  [\text{Longbo:} #1]}}
	\newcommand{\zhixuan}[1]{{\color{purple} [\text{Zhixuan:} #1]}}
	\newcommand{\siwei}[1]{{\color{blue} [\text{Siwei:} #1]}}
	\newcommand{\yihan}[1]{{\color{teal} [\text{Yihan:} #1]}}
}

\allowdisplaybreaks

\title{Continuous Mean-Covariance Bandits}

\author{%
  Yihan Du\\
  IIIS, Tsinghua University\\
  Beijing, China\\
  \texttt{duyh18@mails.tsinghua.edu.cn} \\
   \And
   Siwei Wang \\
   CST, Tsinghua University\\
   Beijing, China\\
   \texttt{wangsw2020@mail.tsinghua.edu.cn} \\
   \AND
   Zhixuan Fang \\
   IIIS, Tsinghua University, Beijing, China\\
   Shanghai Qi Zhi Institute, Shanghai, China\\
   \texttt{zfang@mail.tsinghua.edu.cn} \\
   \And
   Longbo Huang\thanks{Corresponding author.} \\
   IIIS, Tsinghua University\\
   Beijing, China\\
   \texttt{longbohuang@mail.tsinghua.edu.cn} \\
}

\begin{document}

\maketitle

\begin{abstract}
Existing risk-aware multi-armed bandit models typically focus on risk measures of individual options such as variance. As a result, they cannot be directly applied to important real-world online decision making problems with correlated options. 
In this paper, we propose a novel Continuous Mean-Covariance Bandit (CMCB) model to explicitly take into account option correlation. Specifically, in CMCB, there is a learner who sequentially chooses weight vectors on given options and observes random feedback according to the decisions. The agent's  objective is to  achieve the best trade-off between reward and risk, measured with option covariance. 
To capture different reward observation scenarios in practice, we consider three feedback settings, i.e., full-information, semi-bandit and full-bandit feedback. We propose novel algorithms with rigorous regret analysis, and provide nearly matching lower bounds to validate their optimalities (in terms of the number of timesteps $T$). The experimental results also demonstrate the superiority of our algorithms.  
To the best of our knowledge, this is the first work that considers option correlation in risk-aware bandits and explicitly quantifies how arbitrary covariance structures impact the learning performance.
The novel analytical techniques we developed for exploiting the estimated covariance to build concentration and bounding the risk of selected actions based on sampling strategy properties can likely find applications in other bandit analysis and be of independent interests. 
\end{abstract}

\section{Introduction}
The stochastic Multi-Armed Bandit (MAB)~\cite{UCB_auer2002,thompson1933,agrawal2012analysis} problem is a classic online learning model, which characterizes the exploration-exploitation trade-off in decision making.
%
%
Recently, due to the increasing requirements of risk guarantees in practical applications, the Mean-Variance Bandits (MVB) \cite{sani2012risk,vakili2016risk,thompson_sampling_mean_variance2020} which aim at balancing the rewards and performance variances have received extensive attention.
%
While MVB provides a successful risk-aware model, it only considers discrete decision space and focuses on the variances of individual arms (assuming independence among arms).

However, in many real-world scenarios,
a decision often involves multiple options with  certain correlation structure, which can heavily influence risk management and cannot be ignored.  
For instance, in finance, investors can select portfolios on multiple correlated assets, and the investment risk is closely related to the correlation among the chosen assets. The well-known ``risk diversification'' strategy~\cite{risk_diversification1991} embodies the importance of correlation to investment decisions. 
In clinical trials, a treatment often  consists of different drugs with certain ratios,
and the correlation among drugs plays an important role in the treatment risk. 
Failing to handle the correlation among multiple options, existing MVB results cannot be directly applied to these important  real-world tasks.

Witnessing the above limitation of existing risk-aware results, in this paper, we propose a novel Continuous Mean-Covariance Bandit (CMCB) model, which considers a set of options (base arms) with continuous decision space and measures the risk of decisions with the option correlation. Specifically, 
in this model, a learner is given $d$ base arms, which are associated with an unknown joint reward distribution with a mean vector and covariance. At each timestep, the environment generates an underlying random reward for each base arm according to the joint distribution. Then, the learner selects a weight vector of base arms and observes the rewards.
The goal of the learner is to minimize the expected cumulative regret, i.e., the total difference of the reward-risk (mean-covariance) utilities between the chosen actions and the optimal action, where the optimal action is defined as the weight vector that achieves the best trade-off between the expected reward and covariance-based risk. 
To capture important observation scenarios in practice, we consider three feedback settings in this model, i.e., full-information (CMCB-FI), semi-bandit (CMCB-SB) and full-bandit (CMCB-FB) feedback, which vary from seeing  rewards of all options to receiving  rewards of the selected options to only observing a weighted sum of rewards. 

The CMCB framework finds a wide range of real-world applications, including finance~\cite{Markowitz1952}, company operation~\cite{computer_operation2004} and online advertising~\cite{online_advertsing2017}. 
%
%
For example, in stock markets, investors choose portfolios
based on the observed prices of all stocks (full-information feedback),
with the goal of earning high returns and meanwhile minimizing risk. 
In company operation, managers allocate  investment budgets to several correlated business
and only observe the returns of the invested business (semi-bandit feedback), with the objective of achieving high returns and low risk. 
In clinical trials, clinicians select a treatment comprised of different drugs and only observe an overall therapeutic effect (full-bandit feedback), where good therapeutic effects and high stability are both desirable.


For both CMCB-FI and CMCB-SB, we propose novel  algorithms and establish nearly matching lower bounds for the problems, and contribute novel techniques in analyzing the risk of chosen actions and exploiting the covariance information. 
For CMCB-FB, we develop a novel algorithm which adopts a carefully designed action set to estimate the expected rewards and covariance, with non-trivial regret guarantees.
Our theoretical results offer an explicit quantification of the influences of arbitrary covariance structures on learning performance, and our empirical evaluations also demonstrate the superior performance of our algorithms.

Our work differs from previous works on bandits with  covariance~\cite{online_variance_minimization2006,online_variance_minimization2012,known_covariance2016,covariance-adapting2020} in the following aspects. 
(i) 
We consider the reward-risk objective under  continuous decision space and stochastic environment, while existing works study either combinatorial bandits, where the decision space is discrete and risk is not considered in the objective, or  adversarial online optimization. 
%
(ii) We do not assume a prior knowledge or direct feedback on the covariance matrix as in~\cite{online_variance_minimization2006,online_variance_minimization2012,known_covariance2016}. (iii) Our results for  full-information and full-bandit feedback explicitly characterize the impacts of   arbitrary  covariance structures, whereas prior results, e.g., \cite{known_covariance2016,covariance-adapting2020}, only focus on independent or positively-correlated cases. 
These differences pose new challenges in algorithm design and analysis, and demand new  analytical techniques.

%
We summarize the main contributions as follows. 
\begin{itemize}
	\item We propose a novel risk-aware bandit model called continuous mean-covariance bandit (CMCB), which considers correlated options with continuous decision space, and characterizes the trade-off between reward and covariance-based risk. 
	Motivated by practical reward observation scenarios, three feedback settings are considered under CMCB, i.e., full-information (CMCB-FI), semi-bandit (CMCB-SB) and full-bandit (CMCB-FB).

	\item We design an algorithm $\empiricalmax$ for CMCB-FI with a  $\tilde{O}(\sqrt{T})$ regret, and develop a novel analytical technique to build a relationship on risk between chosen actions and the optimal one
	using properties of the sampling strategy. 
	We also derive a nearly matching lower bound, by analyzing the gap between hindsight knowledge and available empirical information under a Bayesian environment.

	\item For CMCB-SB, we develop  $\algsemibandit$, an algorithm that exploits the estimated covariance information to construct confidence intervals and achieves a $\tilde{O}(\sqrt{T})$ regret. 
	A regret lower bound  is also established, by  investigating the necessary regret paid to differentiate  two  well-chosen distinct instances.

	\item We propose a novel algorithm $\algfullbandit$ for CMCB-FB, which employs a well-designed action set to carefully estimate the reward means and covariance, and achieves a $\tilde{O}(T^\frac{2}{3})$ regret guarantee under the severely limited feedback. 
	%
\end{itemize}
To our best knowledge, our work is the \emph{first} to explicitly characterize the influences of \emph{arbitrary} covariance structures on learning performance in risk-aware bandits. 
Our results shed light into risk management in online decision making with correlated options.
%
Due to space limitation, we defer all detailed proofs to the supplementary material.

\section{Related Work} \label{sec:related_work}
\textbf{(Risk-aware Bandits)}
Sani et al.~\cite{sani2012risk} initiate the classic mean-variance paradigm~\cite{Markowitz1952,Markowitz2017} in bandits, and formulate the mean-variance bandit problem, where the learner plays a single arm each time and the risk is measured by the variances of individual arms.
Vakili \& Zhao~\cite{vakili2015mean,vakili2016risk} further study this problem under a different metric and complete the regret analysis.
Zhu \& Tan~\cite{thompson_sampling_mean_variance2020} provide a Thompson Sampling-based algorithm for mean-variance bandits.
In addition to variance, several works consider other risk criteria.
The VaR measure is studied in \cite{VaR2016}, and CVaR is also investigated to quantify the risk in \cite{CVaR_galichet2013,distribution_oblivious_risk_aware2019}.
Cassel et al.~\cite{general_risk_criteria2018} propose a general risk measure named empirical distributions performance measure (EDPM) and present an algorithmic framework for EDPM.
All existing studies on risk-aware bandits only consider discrete decision space and assume independence among arms, and thus they cannot be applied to our CMCB problem.

\textbf{(Bandits with Covariance)}
In the stochastic MAB setting, while there have been several works~\cite{known_covariance2016,covariance-adapting2020} on covariance, they focus on the combinatorial bandit problem without considering risk.
Degenne \& Perchet~\cite{known_covariance2016} study the combinatorial semi-bandits with correlation, which assume a known upper bound on the covariance, and design an algorithm with this prior knowledge of covariance.
Perrault et al.~\cite{covariance-adapting2020} further investigate this problem without the assumption on covariance under the sub-exponential distribution framework, and propose an algorithm with a tight asymptotic regret analysis. 
In the adversarial setting, Warmuth \& Kuzmin~\cite{online_variance_minimization2006,online_variance_minimization2012} consider an online variance minimization problem, where at each timestep the learner chooses a weight vector and receives a covariance matrix, and propose the exponentiated gradient based algorithms. 
Our work differs from the above works in the following aspects:
compared to  \cite{known_covariance2016,covariance-adapting2020}, we consider a continuous decision space instead of combinatorial space, study the reward-risk objective instead of only maximizing the expected reward, and investigate two more feedback settings other than the semi-bandit feedback.
Compared to \cite{online_variance_minimization2006,online_variance_minimization2012}, we consider the stochastic environment  and in our case, the covariance cannot be directly observed and needs to be estimated.

\section{Continuous Mean-Covariance Bandits (CMCB)} \label{sec:formulation}
Here we present the formulation for the Continuous Mean-Covariance Bandits (CMCB) problem. 
%
Specifically, a learner is given $d$ base arms labeled $1, \dots, d$ and a decision (action) space $\cD \subseteq \triangle_{d}$, where  $\triangle_{d}=\{\bw \in \R^d:  0 \leq w_i \leq 1, \forall i \in [d], \  \sum_i w_i=1\}$ denotes the probability simplex in $\R^d$. 
The base arms are associated with an unknown $d$-dimensional joint reward distribution with mean vector $\bs{\theta}^*$
and positive semi-definite covariance matrix $\Sigma^*$, where  $\Sigma^*_{ii} \leq 1$ for any $i \in [d]$ without loss of generality. 
%
For any action $\bw \in \cD$, which can be regarded as a weight vector placed on the base arms, 
the instantaneous reward-risk utility is given by the following \emph{mean-covariance} function 

\begin{eqnarray}
f(\bw)=\bw^\top \bs{\theta}^* - \rho \bw^\top \Sigma^* \bw, \label{eq:mean-cov-def}
\end{eqnarray}
where $\bw^\top \bs{\theta}^*$ denotes the expected reward, $\bw^\top \Sigma^* \bw$ represents the risk, i.e., reward variance, and  $\rho>0$ is a risk-aversion parameter that controls the weight placed on the risk. 
We define the optimal action as $\bw^*=\argmax_{\bw \in \cD} f(\bw)$. 
Compared to linear bandits~\cite{improved_linear_bandit2011,conservative_linear_bandit2017}, the additional quadratic term in $f(\bw)$ raises significant challenges in estimating the covariance, bounding the risk of chosen actions and deriving  covariance-dependent regret bounds. 

At each timestep $t$, the environment generates an underlying (unknown to the learner) random reward vector
$\bs{\theta}_t=\bs{\theta}^*+\bs{\eta}_t$ according to the joint distribution, where $\bs{\eta}_t$ is a zero-mean noise vector and it is independent among different timestep $t$.
Note that here we consider an additive vector noise to the parameter $\bs{\theta}^*$, instead of the simpler scalar noise added in the observation (i.e., $y_t=\bw_t^\top \bs{\theta}^*+\eta_t$) as in linear bandits~\cite{improved_linear_bandit2011,conservative_linear_bandit2017}. 
Our noise setting better models the real-world scenarios where distinct actions incur different risk, and enables us to explicitly quantify the correlation effects. Following the standard assumption in the bandit literature \cite{APT2016,known_covariance2016,thompson_sampling_mean_variance2020},
we assume the noise is sub-Gaussian, i.e.,
$\forall \bu \in \R^d$, $\E[\exp(\bu^\top \bs{\eta}_t)] \leq \exp(\frac{1}{2} \bu^\top \Sigma^* \bu)$, where $\Sigma^*$ is unknown.
The learner selects an action $\bw_t  \in \cD$ and observes the feedback according to a certain structure (specified later).
For any time horizon $T>0$, define the expected cumulative regret as
$$
\ex\mbr{\cR(T)}=\sum_{t=1}^{T} \ex \mbr{ f(\bw^*)-f(\bw_t) }. 
$$
The objective of the learner is to minimize  $\ex[\cR(T)]$.
Note that our mean-covariance function Eq.~\eqref{eq:mean-cov-def} extends the popular mean-variance measure~\cite{sani2012risk,vakili2016risk,thompson_sampling_mean_variance2020} to the continuous decision space. 


In the following,  we consider three feedback settings motivated by reward observation scenarios in practice, including (i) full-information (CMCB-FI),  observing random rewards of all base arms after a pull, (ii) semi-bandit (CMCB-SB), only observing  random rewards of the selected base arms,  and (iii) full-bandit (CMCB-FB),  only seeing a weighted sum of the random rewards from base arms.
%
We will present the formal definitions of these three feedback settings in the following sections. 

\textbf{Notations.} 
For action $\bw \in \cD$, let $I_{\bw}$ be a diagonal matrix such that $I_{\bw,ii}=\I\{w_i>0\}$. 
For a matrix $A$, let $A_{\bw}=I_{\bw} A I_{\bw}$ and $\Lambda_A$ be a diagonal matrix with the same diagonal as $A$.

\section{CMCB with Full-Information Feedback (CMCB-FI)}
\label{sec:full_information}
We start with CMCB with full-information feedback (CMCB-FI). In this setting, at each timestep $t$, the learner selects $\bw_t \in \triangle_{d}$ and observes the random reward $\theta_{t,i}$ for all $i \in [d]$.
CMCB-FI provides an online learning model for the celebrated Markowitz~\cite{Markowitz1952,Markowitz2017} problem in finance, where investors select portfolios and can observe the prices of all stocks at the end of the trading days. 

Below, we propose an efficient  Mean-Covariance Empirical algorithm ($\empiricalmax$)  for CMCB-FI, and provide a novel regret analysis that  fully characterizes how an arbitrary covariance structure affects the regret performance. 
We also present a nearly matching lower bound for CMCB-FI to demonstrate the optimality of  $\empiricalmax$.

\subsection{Algorithm for CMCB-FI}

\begin{algorithm}[t!]
	\caption{$\empiricalmax$} \label{alg:empirical_max}
	\begin{multicols}{2}
		\begin{algorithmic}[1]
			\STATE \textbf{Input:} Risk-aversion parameter $\rho>0$.
			\STATE Initialization: Pull action $\bw_1=(\frac{1}{d}, \dots, \frac{1}{d})$, and observe $\bs{\theta}_1=(\theta_{1,1}, \dots, \theta_{d,1})^\top \!$. 
			$\hat{\theta}^*_{1,i} \leftarrow  \theta_{1,i}, \  \forall i \in [d]$.
			$\hat{\Sigma}_{1,ij}= (\theta_{1,i}-\hat{\theta}^*_{1,i})(\theta_{1,j}-\hat{\theta}^*_{1,j}), \  \forall i,j \in [d]$.
			\FOR{$t=2,3,\dots$}
			\STATE $\bw_t=\argmax \limits_{\bw \in \triangle_{d}} ( \bw^\top \bs{\hat{\theta}}^*_{t-1} - \rho \bw^\top \hat{\Sigma}_{t-1} \bw )$
			\STATE Pull $\bw_t$, observe $\bs{\theta}_t=(\theta_{t,1}, \dots, \theta_{t,d})^\top$
			\STATE $\hat{\theta}^*_{t,i} \leftarrow \frac{1}{t} \sum_{s=1}^{t} \theta_{s,i}, \  \forall i \in [d]$
			\STATE $\!\! \hat{\Sigma}_{t,ij} \!\!= \!\! \frac{1}{t} \!\! \sum \limits_{s=1}^{t} \! (\theta_{s,i} \!-\! \hat{\theta}^*_{t,i}) \! (\theta_{s,j} \!-\! \hat{\theta}^*_{t,j}), \! \forall i,\! j \!\in\! [d]$
			\ENDFOR
		\end{algorithmic}
	\end{multicols}
	\vspace*{-1em}
\end{algorithm}


Algorithm~\ref{alg:empirical_max} shows the detailed steps of $\empiricalmax$. 
%
%
Specifically, at each timestep $t$, we use the empirical mean $\bs{\hat{\theta}}_t$ and  covariance $\hat{\Sigma}_t$ to estimate $\bs{\theta}^*$ and $\Sigma^*$, respectively. Then, we form $\hat{f}_{t}(\bw)= \bw^\top \bs{\hat{\theta}}_{t} - \rho \bw^\top \hat{\Sigma}_{t} \bw$, an empirical mean-covariance function of $\bw \in \triangle_{d}$, and  always choose the action with the maximum empirical objective value. 

Although $\empiricalmax$ appears to be intuitive, its analysis is highly non-trivial due to  covariance-based risk in the objective. In this case, a naive universal bound cannot characterize the impact of covariance, 
and prior gap-dependent analysis (e.g.,~\cite{known_covariance2016,covariance-adapting2020}) cannot be applied to solve our continuous space analysis with gap approximating to zero.
%
Instead, we develop two novel techniques
to handle the covariance, including 
using the actual covariance to analyze the confidence region of the expected rewards, 
and exploiting the empirical information of the sampling strategy to bound the risk gap between selected actions and the optimal one.
Different from prior works~\cite{known_covariance2016,covariance-adapting2020}, which  assume  a prior knowledge on covariance or only focus on the independent and positively-related cases, our analysis does not require extra knowledge of covariance and explicitly quantifies the effects of arbitrary covariance structures.
%
The regret performance of $\empiricalmax$ is summarized in Theorem~\ref{thm:ub_full_information}. 

\begin{theorem}[Upper Bound for CMCB-FI]
	\label{thm:ub_full_information}
	Consider the continuous mean-covariance bandits with full-information feedback (CMCB-FI).
	For any $T \geq 1+\frac{\Sigma^*_{\textup{max}}}{{\bw^*}^\top \Sigma^* {\bw^*}}$, algorithm $\empiricalmax$ (Algorithm~\ref{alg:empirical_max}) achieves an expected cumulative regret bounded by
	\begin{align}
	O \Bigg( \bigg(  \bm{\min}  \Big\{ \sqrt{{\bw^*}^{\top}  \Sigma^* \bw^*} + \rho^{-\frac{1}{2}}  \sqrt{  {\theta}^*_{\textup{max}}-{\theta}^*_{\textup{min}} } , \sqrt{ \Sigma^*_{\max} } \Big\} + \rho \bigg) \ln T  \sqrt{d T} \Bigg), \label{eq:ub_fi} 
	\end{align}
	where ${\theta}^*_{\textup{max}}=\max_{i \in [d]} \theta^*_{i}$, ${\theta}^*_{\textup{min}}=\min_{i \in [d]} \theta^*_{i}$ and $\Sigma^*_{\max} = \max_{i \in [d]} \Sigma^*_{ii} $.
\end{theorem}

\emph{Proof sketch.}
Let $D_t$ be the diagonal matrix which takes value $t$ at each diagonal entry.
We first build confidence intervals for the expected rewards of actions and the covariance as 
$|\bw^\top \bs{\theta}^*-\bw^\top \bs{\hat{\theta}}_{t-1}| \leq p_t(\bw) \triangleq c_1 \sqrt{\beta_t} \sqrt{\bw^\top D_{t-1}^{-1}(\lambda \Lambda_{\Sigma^*} D_{t-1} + \sum_{s=1}^{t-1} \Sigma^* ) D_{t-1}^{-1} \bw}$ and $|\Sigma^*_{ij}-\hat{\Sigma}_{ij,t-1}|  \leq q_{t} \triangleq c_2 \frac{\ln t}{\sqrt{t-1}}$. 
Here $\beta_t \triangleq \ln t + d\ln(1+\lambda^{-1})$, $\lambda \triangleq \frac{{\bw^*}^\top \Sigma^* {\bw^*}}{\Sigma^*_{\textup{max}}}$, and $c_1$ and $c_2$ are positive constants. 
Then, we obtain the confidence interval of $f(\bw)$ as $|\hat{f}_{t-1}(\bw)-f(\bw)|\leq r_t(\bw) \triangleq p_t(\bw)+ \rho {\bw}^\top Q_t {\bw}$, where $Q_t$ is a matrix with all entries equal to $q_{t}$. 
Since algorithm $\empiricalmax$ always plays the empirical best action, we have $f(\bw^*)-f(\bw_t) \leq \hat{f}_{t-1}(\bw^*)+ r_t(\bw^*)-f(\bw_t) \leq \hat{f}_{t-1}(\bw_t)+ r_t(\bw^*)-f(\bw_t) \leq r_t(\bw^*)+r_t(\bw_t)$. 
Plugging the definitions of $f(\bw)$ and $r_t(\bw)$, we have 
\begin{align}
\!\! - \! \Delta_{\theta^*} \!\! + \!\! \rho \sbr{\bw_t^\top \Sigma^* \bw_t -  {\bw^*}^\top \Sigma^* {\bw^*}}  \! \leq \! f(\bw^*) \! - \! f(\bw_t) 
\! \overset{\textup{(a)}}{\leq} \! c_3 \frac{ \!\! \sqrt{\beta_t} ( \! \sqrt{    {\bw^*}^\top  \Sigma^* \bw^* }  \!\!+\!\!  \sqrt{ \bw_t^\top  \Sigma^* \bw_t } ) + \rho \ln t }{\sqrt{t-1}},
\!\! \label{eq:proof_sketch_ineq}
\end{align}
where $\Delta_{\theta^*}=\theta^*_{\textup{max}}-\theta^*_{\textup{min}}$ and $c_3$ is a positive constant. Since our goal is to bound the regret $f(\bw^*) - f(\bw_t)$ and in inequality~(a) only the $\sqrt{ \bw_t^\top  \Sigma^* \bw_t }$ term is a  variable, the challenge falls on bounding $\bw_t^\top  \Sigma^* \bw_t$. Note that the left-hand-side of Eq.~\eqref{eq:proof_sketch_ineq} is linear with respect to $\bw_t^\top  \Sigma^* \bw_t$ and the right-hand-side only contains $\sqrt{ \bw_t^\top  \Sigma^* \bw_t }$. 
%
Then, using the property of sampling strategy on $\bw_t$, i.e.,  Eq.~\eqref{eq:proof_sketch_ineq},  again, after some algebraic analysis, 
we obtain $\bw_t^\top  \Sigma^* \bw_t \leq c_4( {\bw^*}^\top \Sigma^* {\bw^*}+  \frac{1}{\rho}  \Delta_{\theta^*}    +  \frac{1}{\rho} \sqrt{ \frac{\beta_t}{t-1} }    \sqrt{ {\bw^*}^\top  \Sigma^* \bw^* }    +   \frac{\ln t}{\sqrt{t-1}} + \frac{\beta_t}{\rho^2(t-1)}) $ for some constant $c_4$. 
Plugging it into inequality~(a) and doing a summation over $t$, we obtain the theorem.
\hfill $\qed$  


\textbf{Remark 1.}
Theorem \ref{thm:ub_full_information}  fully characterizes how an \emph{arbitrary} covariance structure impacts the regret bound. 
To see this, note that in Eq. (\ref{eq:ub_fi}), under the $\bm{\min}$ operation, the first $\sqrt{{\bw^*}^\top  \Sigma^* \bw^*}$-related term dominates under reasonable $\rho$, and shrinks from positive to negative correlation, which implies that the more the base arms are negatively (positively) correlate, the lower (higher) regret the learner suffers.  
%
The intuition behind is that the negative (positive) correlation diversifies (intensifies) the risk of estimation error and narrows (enlarges) the confidence region for the expected reward of an action, which leads to a reduction (an increase) of regret. 
%

Also note that when $\rho=0$, the CMCB-FI problem reduces to a $d$-armed bandit problem with full-information feedback,  
and Eq.~(\ref{eq:ub_fi}) becomes $\tilde{O}( \sqrt{ d \Sigma^*_{\max} T})$.
For this degenerated case, the optimal gap-dependent regret is $O(\frac{\Sigma^*_{\max}}{\Delta})$ for constant gap $\Delta>0$. By setting $\Delta=\sqrt{\Sigma^*_{\max} / T}$ at this gap-dependent result, one obtains the optimal gap-independent regret $O(\sqrt{\Sigma^*_{\max} T})$. Hence, when $\rho=0$,  Eq.~(\ref{eq:ub_fi}) still offers a tight gap-independent regret bound with respect to $T$.

\subsection{Lower Bound for CMCB-FI} \label{section:lb_full_information}
Now we provide a regret lower bound for CMCB-FI to corroborate the optimality of  $\empiricalmax$.

Since CMCB-FI considers full-information feedback and continuous decision space where the reward gap $\Delta$ (between the optimal action and the nearest optimal action) approximates to zero, existing lower bound analysis for linear~\cite{dani2008stochastic,dani_price2008} or discrete~\cite{lai_robbins1985,known_covariance2016,covariance-adapting2020} bandit problems  cannot be applied to this problem.

To tackle this challenge, we contribute a new analytical procedure to establish the lower bound for continuous and full-information bandit problems from the Bayesian perspective.
The main idea is to construct an instance distribution, where $\bs{\theta}^*$ is drawn from a well-chosen prior Gaussian distribution. After $t$ pulls the posterior of $\bs{\theta}^*$ is still Gaussian with a mean vector $\bu_t$ related to sample outcomes. 
Since the hindsight strategy simply selects the action which maximizes the mean-covariance function with respect to $\bs{\theta}^*$ while a feasible strategy can only utilize the sample information ($\bu_t$), we show that 
any algorithm must suffer $\Omega(\sqrt{T})$ regret due to the gap between random $\bs{\theta}^*$ and its mean $\bu_t$. 
%
%
%
%
%
Theorem~\ref{thm:lb_full_information} below formally states this lower bound.

\begin{theorem}[Lower Bound for CMCB-FI]
	\label{thm:lb_full_information}
	There exists an instance distribution of the continuous mean-covariance bandits with full-information feedback problem (CMCB-FI), for which any algorithm has an expected cumulative regret bounded by $\Omega(\sqrt{T})$.
\end{theorem}

\yihan{Revised the statement of the lower bound to ``There exists an instance distribution of ...'', and removed the detailed parameters of constructed instances.}

\textbf{Remark 2.}
This parameter-free lower bound demonstrates that the regret upper bound of $\empiricalmax$ (Theorem~\ref{thm:ub_full_information}) is near-optimal with respect to $T$.
Unlike discrete bandit problems~\cite{lai_robbins1985,known_covariance2016,covariance-adapting2020} where the optimal regret is usually $\frac{\log T}{\Delta}$ for constant gap $\Delta>0$, 
CMCB-FI has a continuous decision space with gap $\Delta \rightarrow 0$ and a polylogarithmic regret is not achievable in general. 

\begin{algorithm}[t!]
	\caption{$\algsemibandit$} \label{alg:semi_bandit}
	\begin{multicols}{2}
		\begin{algorithmic}[1]
			\STATE \textbf{Input:} $\! \rho \! > \! 0 $, $c \in (0,\frac{1}{2}]$ and regularization parameter $\! \lambda \! > \! 0$.
			\STATE Initialize: $\forall i \in [d]$, pull $\be_i$ that has $1$ at the $i$-th entry and $0$ elsewhere. $\forall i,j \in [d], i \neq j$, pull $\be_{ij}$ that has $\frac{1}{2}$ at the $i$-th and the $j$-th entries, and $0$ elsewhere. 
			Update $N_{ij}(d^2), \  \forall i,j \in [d]$, $\bs{\hat{\theta}}_{d^2}$ and $\hat{\Sigma}_{d^2}$.
			\label{line:sb_initilization}
			\FOR{$t=d^2+1, \dots$}
			\STATE $\underline{\Sigma}_{t,ij} \leftarrow \hat{\Sigma}_{t-1,ij} - g_{ij}(t)$
			\STATE $\bar{\Sigma}_{t,ij} \leftarrow \hat{\Sigma}_{t-1,ij} + g_{ij}(t)$
			\STATE $\! \bw_t \!\! \leftarrow  \!\! \argmax \limits_{\bw \in \triangle^c_{d}}  (\! \bw^{\! \!\top}  \! \bs{\hat{\theta}}_{t-1} \! + \! E_t(\bw) \! - \! \rho \bw^{\! \!\top} \!  \underline{\Sigma}_{t} \! \bw ) $
			\STATE Pull $\bw_t$ and observe all $\theta_{t,i}$ s.t. $w_{t,i}>0$
			\vspace*{0.2em}
			\STATE $J_{t, ij} \leftarrow \I\{w_{t,i}, w_{t,j}>0\}, \  \forall i,\! j \!\in\! [d]$
			\vspace*{0.3em}
			\STATE $N_{ij}(t) \leftarrow N_{ij}(t-1)+J_{t, ij}, \  \forall i,\! j \!\in\! [d]$
			\vspace*{0.2em}
			\STATE $\hat{\theta}^*_{t,i} \leftarrow \frac{\sum_{s=1}^{t} J_{t, ii}  \theta_{s,i}}{ N_{ii}(t) }, \  \forall i \in [d]$
			\label{line:sb_hat_theta}
			\vspace*{-0.3em}
			\STATE $\! \hat{\Sigma}_{t,ij}  \!\!  \leftarrow  \!\!  \frac{\sum \limits_{s=1}^{t} \!\! J_{t, ij} \! (\theta_{s,i}-\hat{\theta}^*_{t,i}) \! (\theta_{s,j}-\hat{\theta}^*_{t,j})}{N_{ij}(t)} \!, \! \forall i,\! j \! \in \! [d]$
			\label{line:sb_hat_Sigma}
			\ENDFOR
		\end{algorithmic}
	\end{multicols}
	\vspace*{-0.9em}
\end{algorithm}

\section{CMCB with Semi-Bandit Feedback (CMCB-SB)} \label{sec:semi_bandit}
In many practical tasks, the learner may not be able to simultaneously select (place positive weights on) all options and observe full information. Instead, the weight of each option is  usually  lower bounded and cannot be arbitrarily small. As a result, the learner  only selects a subset of options and obtains their feedback, e.g., company investments~\cite{company_investment2019} on multiple  business.

Motivated by such tasks, in this section we consider the CMCB problem with semi-bandit feedback (CMCB-SB), where the decision space is a restricted probability simplex $\triangle^c_{d}=\{\bw \in \R^d: w_i=0 \textup{ or } c \leq w_i \leq 1, \forall i \in [d] \textup{ and }  \sum_i w_i=1\}$ for some constant $0<c\leq \frac{1}{2}$.\footnote{When $c>\frac{1}{2}$, the learner can only place all weight on one option, and
	the problem trivially reduces to the mean-variance bandit setting~\cite{sani2012risk,thompson_sampling_mean_variance2020}. In this case, our Theorem~\ref{thm:ub_semi_bandit} still provides a tight gap-independent bound.
}
In this scenario, at timestep $t$, the learner selects $\bw_t \in \triangle^c_{d}$ and only observes the rewards $\{\theta_{t,i}: w_i \geq c \}$ from the base arms that are placed positive weights on. 
%
Below, we propose the Mean-Covariance Upper Confidence Bound algorithm ($\algsemibandit$) for CMCB-SB, 
and provide a regret lower bound, which shows that 
$\algsemibandit$ achieves the optimal performance with respect to $T$. 

\subsection{Algorithm for CMCB-SB}

Algorithm $\algsemibandit$  for CMCB-SB is described in Algorithm~\ref{alg:semi_bandit}.
The main idea is to use the optimistic covariance to construct a  confidence region for the expected reward of an action and calculate an upper confidence bound of the mean-covariance function, and then select the action with the maximum optimistic mean-covariance value.

In Algorithm~\ref{alg:semi_bandit}, 
$N_{ij}(t)$ denotes the number of times  $w_{s,i},w_{s,j}>0$ occurs among timestep $s \in [t]$.
%
$J_{t, ij}$ is an indicator variable that takes value $1$ if $w_{t,i},w_{t,j}>0$ and $0$ otherwise.
%
$D_t$ is a diagonal matrix such that $D_{t,ii}=N_{ii}(t)$.
In Line~\ref{line:sb_initilization}, we update the number of observations by $N_{ii}(d^2) \leftarrow 2d-1$ for all $i \in [d]$ and $N_{ij}(d^2) \leftarrow 2$ for all $i,j \in [d], i \neq j$ (due to the initialized $d^2$ pulls), and calculate the empirical mean $\hat{\boldsymbol{\theta}}^*_{d^2}$ and empirical covariance $\hat{\Sigma}^*_{d^2}$ using the equations in Lines~\ref{line:sb_hat_theta},\ref{line:sb_hat_Sigma}. 
\yihan{Clarified the initialization steps.}

For any $t > 1$ and $i,j \in [d]$, we define the confidence radius of covariance $\Sigma^*_{ij}$ as
$g_{ij}(t) \triangleq  16 \big( \frac{3 \ln t}{N_{ij}(t-1)} \vee \sqrt{\frac{3 \ln t}{N_{ij}(t-1)}} \big)   + \sqrt{\frac{48 \ln^2 t}{N_{ij}(t-1) N_{ii}(t-1)}} + \sqrt{\frac{36 \ln^2 t}{N_{ij}(t-1) N_{j}(t-1)}}$,
and the confidence region for the expected reward $\bw^\top \bs{\theta}^*$ of action $\bw$ as
$$
E_t(\bw)  \triangleq  \sqrt{  2\beta(\delta_t)  \bigg(  \bw^\top  D_{t-1}^{-1}  \bigg(  \lambda \Lambda_{\bar{\Sigma}_{t}}  D_{t-1} + \sum_{s=1}^{t-1}  \bar{\Sigma}_{s, \bw_s}  \bigg)  D_{t-1}^{-1} \bw  \bigg) } ,
$$
where $\lambda>0$ is the regularization parameter, $\beta(\delta_t)=\ln(\frac{1}{\delta_t}) + d\ln \ln t + \frac{d}{2} \ln(1+\frac{e}{\lambda})$ is the confidence term and $\delta_t = \frac{1}{t \ln^2 t}$ is the confidence parameter.
At each timestep $t$, algorithm $\algsemibandit$ calculates the upper confidence bound of $f(\bw)$ using $g_{ij}(t)$ and $E_t(\bw)$, and selects the action $\bw_t$ that maximizes this upper confidence bound. Then, the learner observes rewards $\theta_{t,i}$ with  $w_{t,i}>0$ and update the statistical information according to the  feedback.

In regret analysis, unlike \cite{known_covariance2016} which uses a universal upper bound to analyze confidence intervals,   we incorporate the estimated covariance into the confidence region for the expected reward of an action, which enables us to derive tighter regret bound and explictly quantify the impact of the covariance structure on algorithm performance.  
We also contribute a new technique for handling the challenge raised by having different numbers of observations among base arms, in order to obtain an optimal $\tilde{O}(\sqrt{T})$ regret (here prior gap-dependent analysis~\cite{known_covariance2016,covariance-adapting2020} still cannot be applied to solve this continuous problem).  
%
Theorem~\ref{thm:ub_semi_bandit} gives the regret upper bound of algorithm $\algsemibandit$. 

\begin{theorem}[Upper Bound for CMCB-SB]
	\label{thm:ub_semi_bandit}
	Consider the continuous mean-covariance bandits with semi-bandit feedback problem (CMCB-SB). Then, for any $T>0$, algorithm $\algsemibandit$ (Algorithm~\ref{alg:semi_bandit}) with regularization parameter $\lambda>0$ has an expected cumulative regret bounded by
	$$
	O \bigg(  \sqrt{ L(\lambda) (\|\Sigma^*\|_{+}   +  d^2) d \ln^2 T    \cdot T }   + \rho  d \ln T  \sqrt{ T } \bigg) ,
	$$
	where $L(\lambda)= (\lambda+1)( \ln(1+ \lambda^{-1}) + 1 ) $ and $\|\Sigma^*\|_{+}=\sum_{i,j \in [d]} \sbr{\Sigma^*_{ij} \vee 0}$ for any $i,j \in [d]$.
\end{theorem}

\textbf{Remark 3.}
Theorem~\ref{thm:ub_semi_bandit} captures the effects of covariance structures in CMCB-SB, i.e., positive correlation renders a larger $\|\Sigma^*\|_{+}$ factor than the negative correlation or independent case, since the covariance influences the rate of estimate concentration for the expected rewards of actions.
%
%
The regret bound for CMCB-SB has a heavier dependence on $d$ than that for CMCB-FI. This matches the fact that semi-bandit feedback only reveals rewards of the queried dimensions, and provides less information than full-information feedback in terms of observable dimensions.
\yihan{Added some sentences to compare the regret for SB with that for FI.}

\subsection{Lower Bound for CMCB-SB} \label{section:lb_semi_bandit}

In this subsection, we establish  a lower bound for CMCB-SB, and show that algorithm $\algsemibandit$ achieves the optimal regret with respect to $T$ up to logarithmic factors. 

The insight of the lower bound analysis is to construct two instances with a gap in the expected reward vector $\bs{\theta}^*$, where the optimal actions under these two instances place positive weights on different base arms. 
Then, when the gap is  set to $\sqrt{  \ln T/ T  }$,
any algorithm must suffer $\Omega \sbr{ \sqrt{ T \ln T} }$ regret for differentiating  these two instances. Theorem~\ref{thm:lb_semi_bandit} summarizes the lower bound for CMCB-SB.

\begin{theorem}[Lower Bound for CMCB-SB]
	\label{thm:lb_semi_bandit}
	There exists an instance distribution of the continuous mean-covariance bandits with semi-bandit feedback (CMCB-SB) problem, for which any algorithm has an expected cumulative regret bounded by $\Omega \sbr{ \sqrt{c d T } }$.
\end{theorem}

\yihan{Fixed Theorem 4 and its proof.}

\textbf{Remark 4.} Theorem~\ref{thm:lb_semi_bandit} demonstrates that the regret upper bound of   $\algsemibandit$ (Theorem~\ref{thm:ub_semi_bandit}) is near-optimal in terms of $T$.
Similar to CMCB-FI, CMCB-SB considers continuous decision space with $\Delta \rightarrow 0$, and thus the lower bound differs from those gap-dependent results $\frac{\log T}{\Delta}$ in discrete bandit problems~\cite{lai_robbins1985,known_covariance2016,covariance-adapting2020}. 

\section{CMCB with Full-Bandit Feedback (CMCB-FB)}
\label{sec:full_bandit}

\begin{algorithm}[t]
	\caption{$\algfullbandit$} \label{alg:full_bandit}
	\begin{multicols}{2}
		\begin{algorithmic}[1]
			\STATE \textbf{Input:} {$\rho>0$, $\tilde{d}=\frac{d(d+1)}{2}$ and design action set $\pi=\{ \bv_1, \dots, \bv_{\tilde{d}} \}$.}
			\STATE Initialize: $N_\pi(0) \leftarrow 0$. $t \leftarrow 1$.\;
			\STATE \textbf{Repeat} lines~\ref{line:ete_loop_start}-\ref{line:ete_loop_end}:
			\IF{ $N_\pi(t-1)>t^{\frac{2}{3}} / d$ } \label{line:ete_loop_start}
			\label{line:ete_check_exploration_times}
			\STATE $\bw_t=\argmax \limits_{\bw \in \triangle_{d}} \ ( \bw^\top \bs{\hat{\theta}}^*_{t-1} - \rho \bw^\top \hat{\Sigma}_{t-1} \bw )$\; \label{line:ete_exploit}
			\vspace*{-0.2em}
			\STATE $t \leftarrow t+1$\;
			\ELSE
			\STATE $N_\pi(t) \leftarrow N_\pi(t-1)+1$\;
			\label{line:ete_exploration_start}
			\FOR{$k=1, \dots, \tilde{d}$}
			\STATE Pull $\bv_k$ and observe $y_{N_\pi(t),k}$\;
			\IF{$k=\tilde{d}$}
			\STATE $\by_{N_\pi(t)} \leftarrow (y_{N_\pi(t),1}, \dots, y_{N_\pi(t),\tilde{d}})^\top$\;
			\vspace*{-0.7em}
			\STATE $\hat{\by}_{t}  \leftarrow  \frac{\sum_{s=1}^{N_\pi(t)} \by_s}{N_\pi(t)} $\;
			\STATE $\hat{z}_{t, k}=\frac{\sum_{s=1}^{N_\pi(t)} (y_{s,k} - \hat{y}_{t,k} )^2}{N_\pi(t)}, \forall k \in [\tilde{d}]$\;
			\STATE $\hat{\bz}_t \leftarrow (\hat{z}_{t, 1}, \dots, \hat{z}_{t, \tilde{d}})^\top$ \;
			\STATE $\bs{\hat{\theta}}_t \leftarrow B_\pi^{+} \hat{\by}_{t}$\;\label{line:theta}
			\STATE $\bs{\hat{\sigma}}_{t} \leftarrow C_\pi^{+} \hat{\bz}_{t}$\;\label{line:sigma}
			\STATE Reshape $\bs{\hat{\sigma}}_{t}$ to $d \times d$ matrix $\hat{\Sigma}_{t}$\;
			\ENDIF
			\STATE $t \leftarrow t+1$\;
			\ENDFOR \label{line:ete_exploration_end}
			\ENDIF \label{line:ete_loop_end}
		\end{algorithmic}
	\end{multicols}
	\vspace*{-1em}
\end{algorithm}

In this section, we further study the CMCB problem with full-bandit feedback (CMCB-FB), where at timestep $t$, the learner selects $\bw_t \in \triangle_{d}$ and only observes the weighted sum of random rewards, i.e., $y_t=\bw_t^\top \bs{\theta}_t$. 
This setting models many real-world decision making tasks, where the learner can only attain an aggregate feedback from the chosen options, such as clinical trials~\cite{clinical_trials2015}. 


\subsection{Algorithm for CMCB-FB}

We propose the Mean-Covariance Exploration-Then-Exploitation algorithm ($\algfullbandit$) for CMCB-FB in Algorithm~\ref{alg:full_bandit}. 
%
%
Specifically, we first choose a design action set $\pi=\{ \bv_1, \dots, \bv_{\tilde{d}} \}$ which contains $\tilde{d}=d(d+1)/2$ actions and satisfies that $B_\pi = ( \bv_1^\top; \dots; \bv_{\tilde{d}}^\top )$ and  $C_\pi=( v_{1,1}^2, \dots, v_{1,d}^2, 2 v_{1,1} v_{1,2}, \dots, 2 v_{1,d-1} v_{1,d} ;  \dots;  v_{\tilde{d},1}^2, \dots, v_{\tilde{d},d}^2, 2 v_{\tilde{d},1} v_{\tilde{d},2}, \dots, 2 v_{\tilde{d},d-1} v_{\tilde{d},d} ) $ are of full column rank. We also denote their Moore-Penrose inverses by $B^+_\pi$ and $C^+_\pi$, and it holds that $B^+_\pi B_\pi =I^{d \times d}$ and $C^+_\pi C_\pi =I^{\tilde{d} \times \tilde{d}}$.
There exist more than one feasible $\pi$, and for simplicity and good performance we choose $\bv_1, \dots, \bv_{d}$ as standard basis vectors in $\R^d$ and $\{\bv_{d+1}, \dots, \bv_{\tilde{d}}\}$ as the set of all $\tbinom{d}{2}$ vectors where each vector has two entries equal to $\frac{1}{2}$ and others equal to $0$.

In an exploration round (Lines \ref{line:ete_exploration_start}-\ref{line:ete_exploration_end}), we pull the designed actions in $\pi$ and maintain their empirical rewards and variances. Through linear transformation by $B^+_\pi$ and $C^+_\pi$, we obtain the estimators of the
expected rewards and covariance of base arms (Lines \ref{line:theta}-\ref{line:sigma}).  
When the estimation confidence is high enough, we exploit the attained information to select the empirical best action (Lines \ref{line:ete_exploit}).
Theorem~\ref{thm:ub_full_bandit} presents the regret guarantee of $\algfullbandit$.

\begin{theorem}[Upper Bound for CMCB-FB]
	\label{thm:ub_full_bandit}
	Consider the continuous mean-covariance bandits with full-bandit feedback problem (CMCB-FB). Then, for any $T>0$, algorithm $\algfullbandit$ (Algorithm~\ref{alg:full_bandit}) achieves an expected cumulative regret bounded by 
	$$
	O \sbr{ Z(\rho, \pi)  \sqrt{ d(\ln T + d^2 ) } \cdot  T^{\frac{2}{3}}+   d \Delta_{\textup{max}} \cdot  T^{\frac{2}{3}} } ,
	$$
	where $Z (\rho, \pi) = \max_{\bw \in \triangle_{d}}  (\sqrt{ \bw^\top B_\pi^+   \Sigma^*_\pi  (B_\pi^{+}) \top \bw  } + \rho \| C^+_\pi \| )  $, $\Sigma^*_\pi=\textup{diag}(\bv_1^\top \Sigma^* \bv_1, \dots, \bv_{\tilde{d}}^\top \Sigma^* \bv_{\tilde{d}} )$, $\| C_\pi^+ \| = \max_{i \in [\tilde{d}]}  {\sum_{j \in [\tilde{d}]} |C^+_{\pi, ij}| } $ and $\Delta_{\textup{max}}=f(\boldsymbol{w}^*)-\min_{\boldsymbol{w} \in \triangle_d} f(\boldsymbol{w})$.
\end{theorem}

\textbf{Remark 5.}
The choice of $\pi$ will affect the regret factor $\Sigma^*_\pi$ contained in $Z (\rho, \pi)$. 
Under our construction, $\Sigma^*_\pi$ can be regarded as a uniform representation of covariance $\Sigma^*$, and
thus our regret bound demonstrates how the learning performance is influenced by the covariance structure, i.e., negative (positive) correlation shrinks (enlarges) the factor and leads to a lower (higher) regret. 

\textbf{Discussion on the ETE strategy.}
In contrast to common ETE-type algorithms, $\algfullbandit$ requires novel analytical techniques in handling the transformed estimate concentration while preserving the covariance information in regret bounds. 
In analysis, we build a novel concentration using key matrices $B^+_\pi$ and $C^+_\pi$ to adapt to the actual covariance structure, and construct a super-martingale which takes the aggregate noise in an exploration round as analytical basis to prove the concentration. 
These techniques allow us to capture the correlations in the results, and are new compared to both the former FI/SB settings and covariance-related bandit literature~\cite{online_variance_minimization2012,known_covariance2016,covariance-adapting2020}. 

In fact, under the full-bandit feedback, it is highly challenging to estimate the covariance without using a fixed exploration (i.e., ETE) strategy. Note that even for its simplified offline version, where one uses given (non-fixed) full-bandit data to estimate the covariance, there is \emph{no available solution} in the  statistics literature to our best knowledge.
%
Hence, for such online tasks with severely limited feedback, ETE is the most viable strategy currently available, as used in many partial observation works~\cite{LinTian2014,Phased_Exploration,Contextual_Volatile_Submodular}.
We remark that our contribution in this setting focuses on designing a practical solution and deriving regret guarantees which explicitly characterize the correlation impacts. 
The lower bound for CMCB-FB remains open, which we leave for future work.

\section{Experiments} \label{sec:experiments}

\begin{figure} [t!]
	\centering     
	\subfigure[FI, synthetic, $d=5$, $\rho=0.1$] { 
		\label{fig:fi_syn}        
		\includegraphics[width=0.3\textwidth]{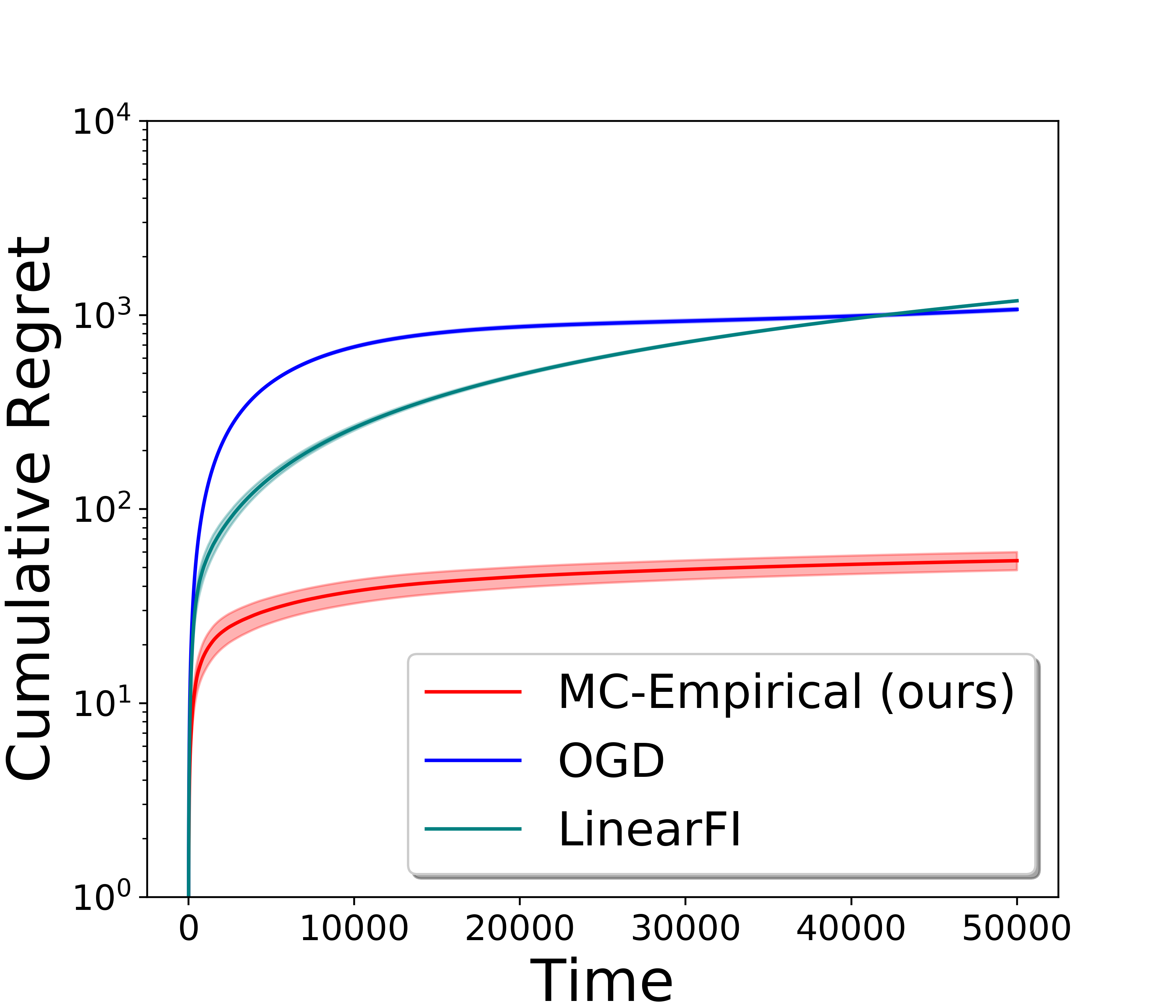} 
	}    
	\subfigure[SB, synthetic, $d=5$, $\rho=0.1$] { 
		\label{fig:sb_syn}       
		\includegraphics[width=0.3\textwidth]{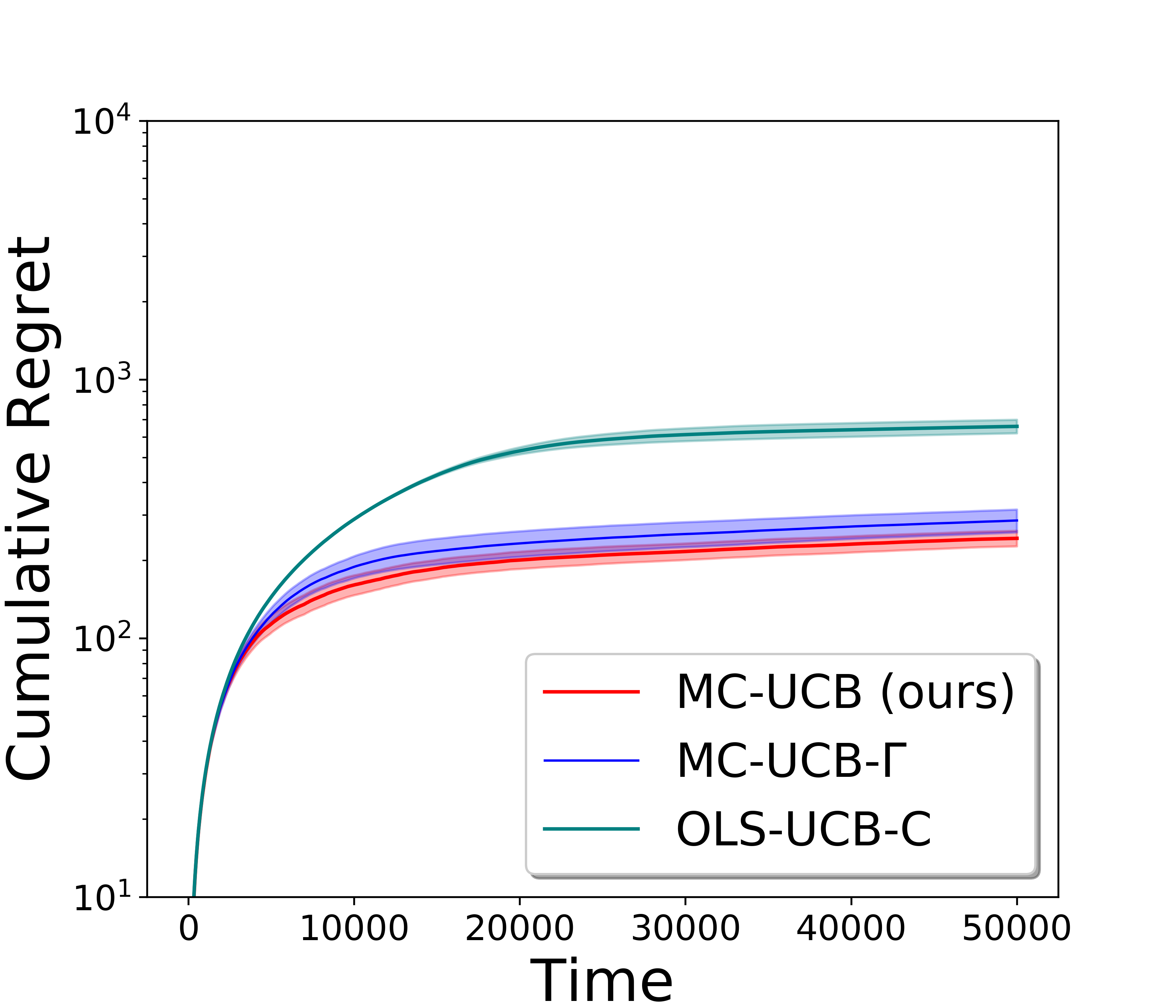} 
	}   
	\subfigure[FB, synthetic, $d=5$, $\rho=10$] { 
		\label{fig:fb_syn}       
		\includegraphics[width=0.3\textwidth]{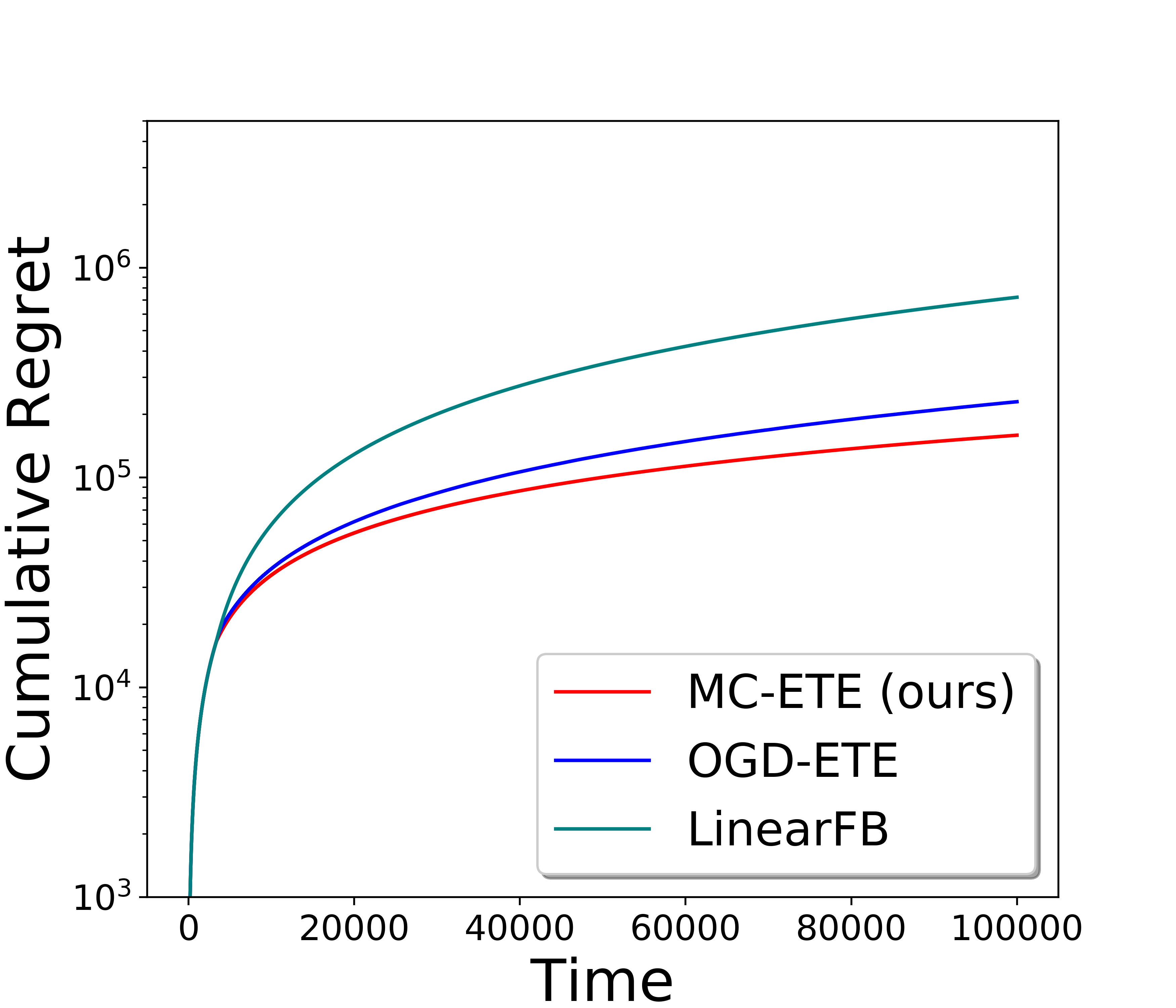} 
	}
	\subfigure[FI, real-world, $d=5$, $\rho=0.1$] { 
		\label{fig:fi_real}        
		\includegraphics[width=0.3\textwidth]{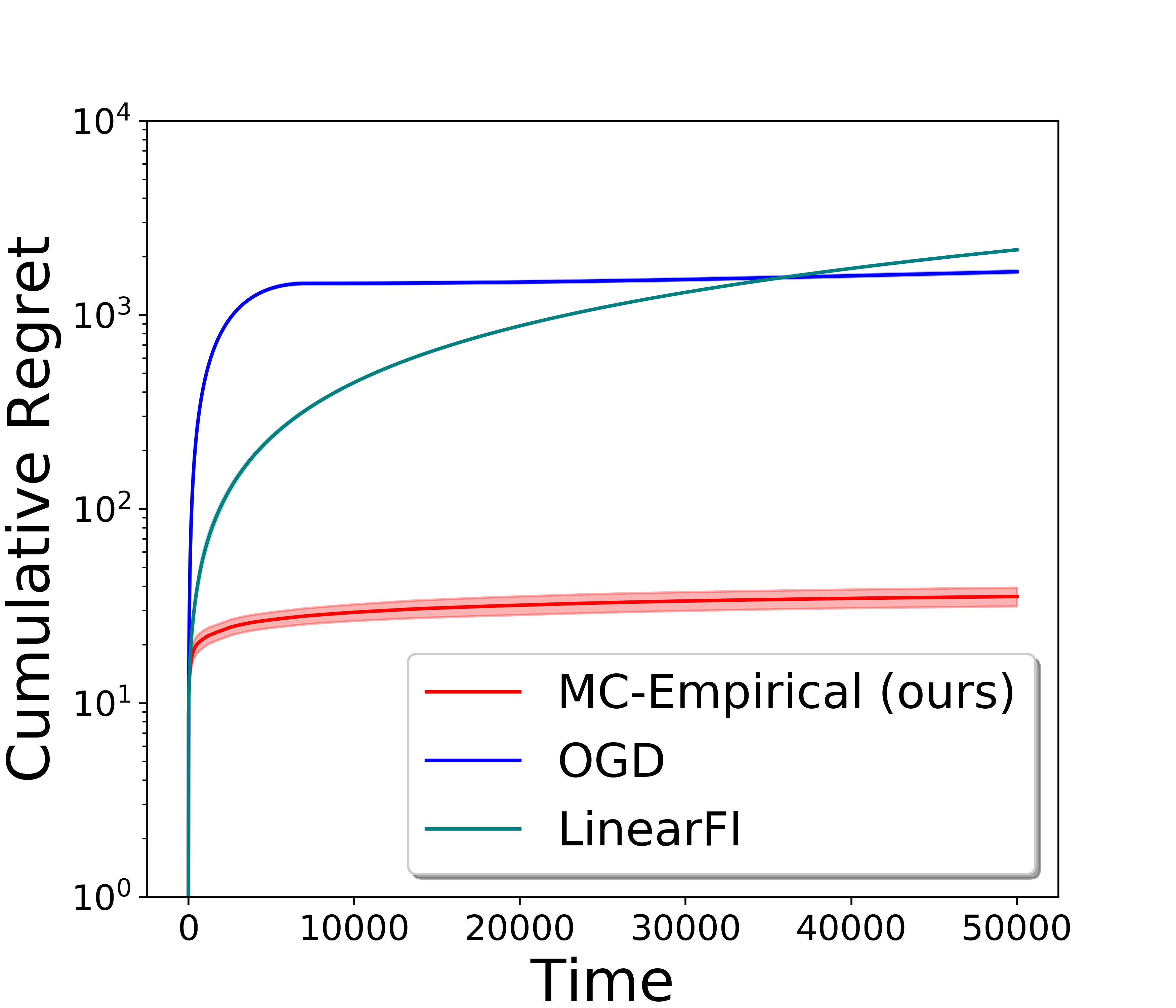} 
	} 
	\subfigure[SB, real-world, $d=5$, $\rho=0.1$] { 
		\label{fig:sb_real}        
		\includegraphics[width=0.3\textwidth]{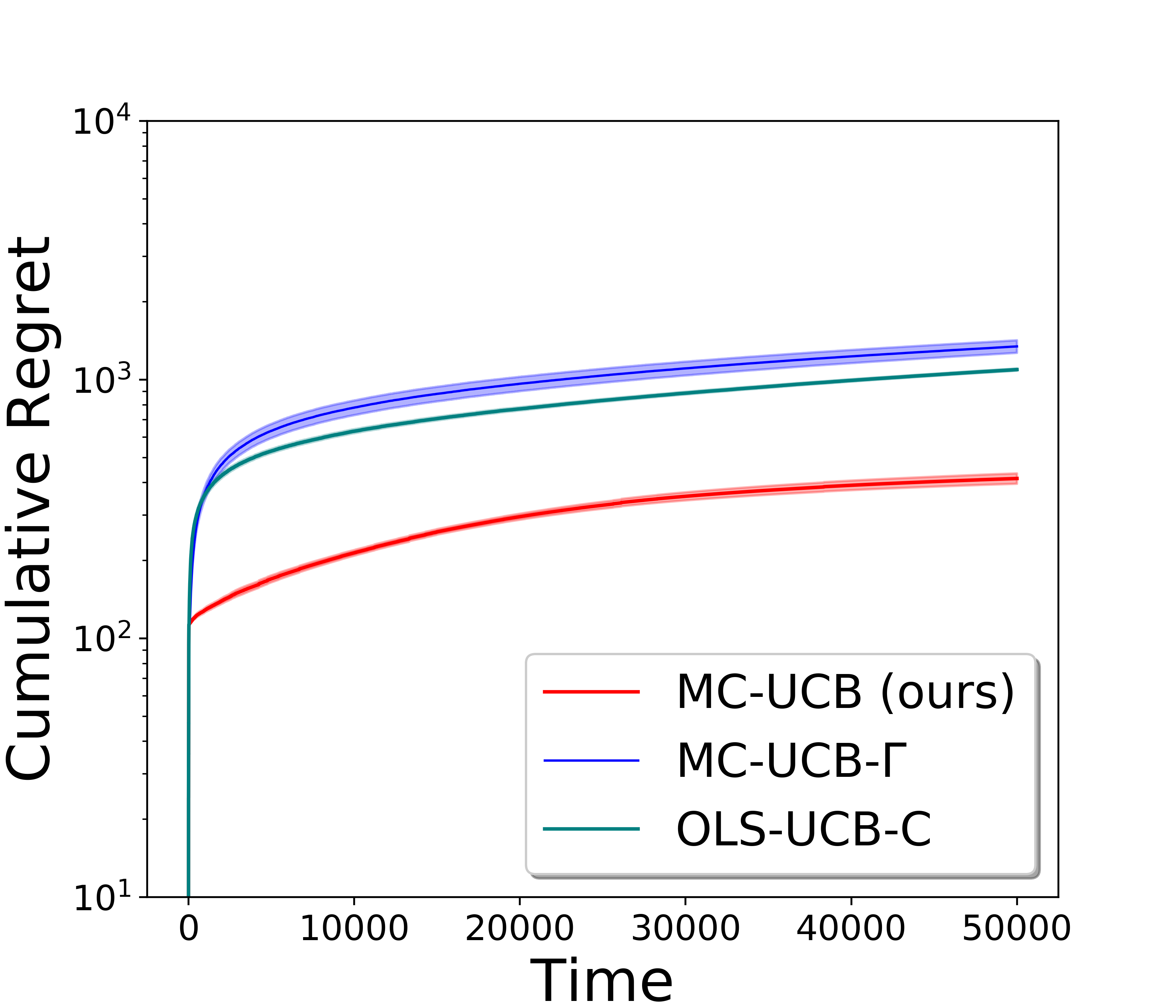} 
	} 
	\subfigure[FB, real-world, $d=5$, $\rho=10$] { 
		\label{fig:fb_real}        
		\includegraphics[width=0.3\textwidth]{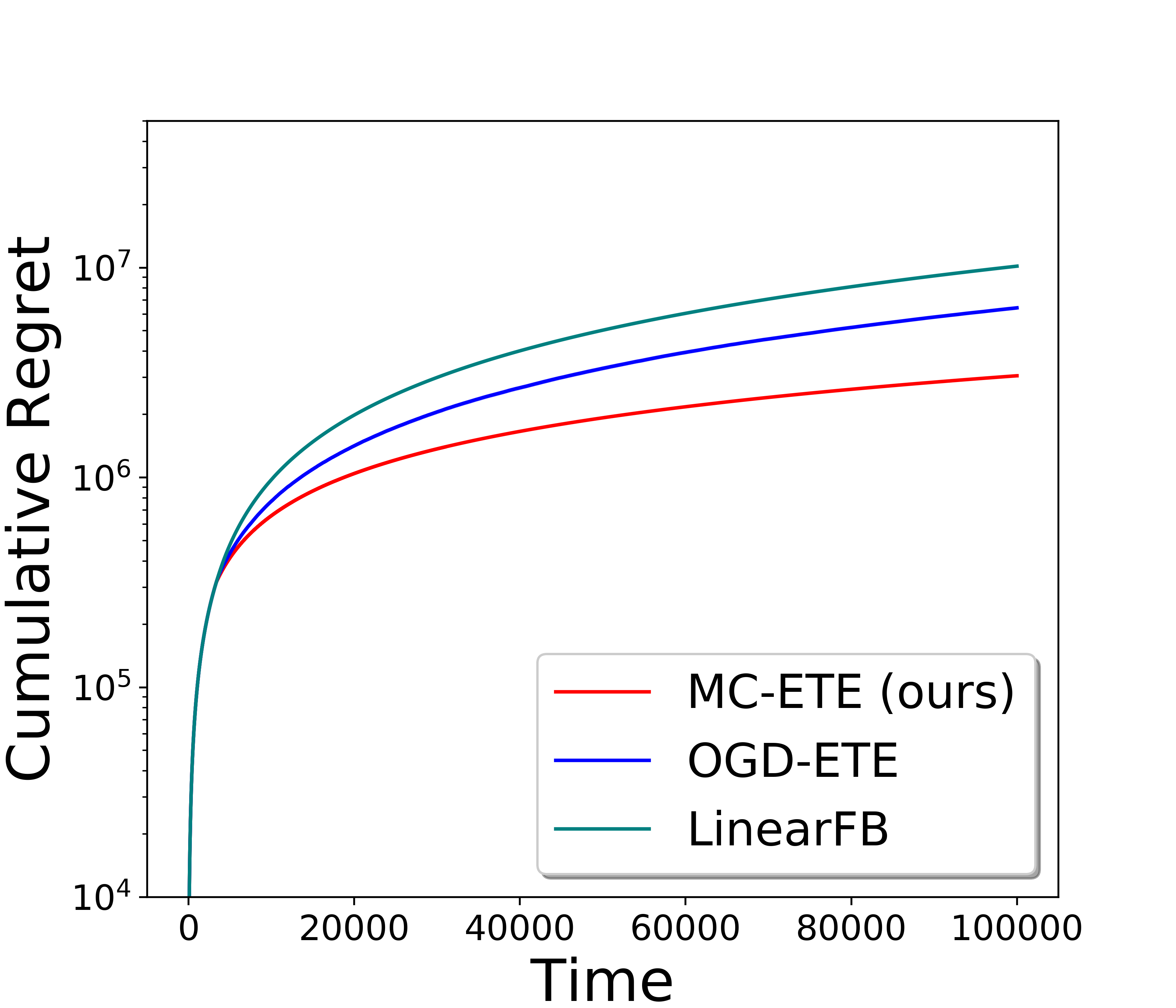} 
	} 
	\caption{Experiments for CMCB-FI, CMCB-SB and CMCB-FB on the synthetic and real-world datasets.\yihan{Added the experiments on the real-world dataset.}}
	\label{fig:experiments_main}
\end{figure} 

In this section, we present experimental results for our algorithms on both synthetic and real-world  \cite{kaggle_us_fund} datasets.
For the synthetic dataset, we set $\bs{\theta}^*=[0.2, 0.3, 0.2, 0.2, 0.2]^\top$, and 
$\Sigma^*$ has all diagonal entries equal to $1$ and all off-diagonal entries equal to $-0.05$. 
For the real-world dataset, we use an open dataset \emph{US Funds from Yahoo Finance} on Kaggle~\cite{kaggle_us_fund}, which provides financial data of 1680 ETF funds in 2010-2017. We select five funds and generate a stochastic distribution ($\bs{\theta}^*$ and $\Sigma^*$) from the data of returns (since we study a stochastic bandit problem). 
For both datasets, we set $d=5$ and $\rho \in \{0.1, 10\}$. The random reward $\bs{\theta}_t$ is drawn i.i.d. from  Gaussian distribution $\cN(\bs{\theta}^*, \Sigma^*)$. We perform $50$ independent runs for each algorithm and show the average regret and $95\%$ confidence interval across runs,\footnote{In some cases, since algorithms are doing similar procedures (e.g., in Figures~\ref{fig:fb_syn},\ref{fig:fb_real}, the algorithms are exploring the designed actions) and have low performance variance, the confidence intervals are narrow and indistinguishable.} with logarithmic y-axis for clarity of magnitude comparison.

\yihan{Added the experiments on the real-world dataset.}

\textbf{(CMCB-FI)}  
We compare our algorithm $\empiricalmax$ with  two algorithms $\mathtt{OGD}$~\cite{hazan2007logarithmic} and $\mathtt{LinearFI}$. 
$\mathtt{OGD}$ (Online Gradient Descent)~\cite{hazan2007logarithmic} is designed for general online convex optimization with also a $O(\sqrt{T})$ regret guarantee, but its result cannot capture the covariance impacts as ours. 
$\mathtt{LinearFI}$ is a linear adaption of $\empiricalmax$ that only aims to maximize the expected rewards.
Figures~\ref{fig:fi_syn},\ref{fig:fi_real} show that our $\empiricalmax$ enjoys multiple orders of magnitude reduction in regret compared to the benchmarks, since it efficiently exploits the empirical observations to select actions and well handles the covariance-based risk. 
In particular, the performance superiority of $\empiricalmax$ over $\mathtt{OGD}$ demonstrates that our sample strategy sufficiently utilize the observed information than conventional gradient descent based policy.


\textbf{(CMCB-SB)} 
For CMCB-SB, we compare $\algsemibandit$ with two adaptions of $\mathtt{OLS \mhyphen UCB}$~\cite{known_covariance2016} (state-of-the-art for combinatorial bandits with covariance), named $\mathtt{MC \mhyphen UCB \mhyphen \Gamma}$ and $\mathtt{OLS \mhyphen UCB \mhyphen C}$.
$\mathtt{MC \mhyphen UCB \mhyphen \Gamma}$ uses the confidence region with a universal covariance upper bound $\Gamma$, instead of the adapting one used in our $\algsemibandit$.
$\mathtt{OLS \mhyphen UCB \mhyphen C}$ directly adapts $\mathtt{OLS \mhyphen UCB}$~\cite{known_covariance2016} to the continuous decision space and only considers maximizing the expected rewards in its objective.
As shown in Figures~\ref{fig:sb_syn},\ref{fig:sb_real},  $\algsemibandit$ achieves the lowest regret since it utilizes the covariance information to accelerate the estimate  concentration.
Due to lack of a covariance-adapting confidence interval, $\mathtt{MC \mhyphen UCB \mhyphen \Gamma}$ shows an inferior regret performance than $\algsemibandit$, and $\mathtt{OLS \mhyphen UCB \mhyphen C}$ suffers the highest regret due to its ignorance of risk.

\textbf{(CMCB-FB)} 
We compare $\algfullbandit$ with  two baselines, $\mathtt{OGD \mhyphen ETE}$, which adopts $\mathtt{OGD}$~\cite{hazan2007logarithmic} during the exploitation phase, and $\mathtt{LinearFB}$, which only investigates the expected reward maximization. 
From Figures~\ref{fig:fb_syn},\ref{fig:fb_real}, 
one can see that, $\algfullbandit$ achieves the best regret performance due to its effective estimation of the covariance-based risk and efficiency in exploitation. 
Due to the inefficiency of gradient descent based policy in utilizing information, $\mathtt{OGD \mhyphen ETE}$ has a higher regret than $\algfullbandit$, whereas $\mathtt{LinearFB}$ shows the worst performance owing to the unawareness of the risk.

\section{Conclusion and Future Work}
\label{sec:conclusion}
In this paper, we propose a novel continuous mean-covariance bandit (CMCB) model, which investigates the reward-risk trade-off measured by option correlation.
Under this model, we consider three feedback settings, i.e., full-information, semi-bandit and full-bandit feedback, to formulate different real-world reward observation scenarios.
We propose novel algorithms for CMCB with rigorous regret analysis, and provide lower bounds for the problems to demonstrate our optimality. We also  present empirical evaluations to show the superior performance of our algorithms. 
To our best knowledge, this is the first work to fully characterize the impacts of arbitrary covariance structures on learning performance for risk-aware bandits.
There are several interesting directions for future work. For example, how to design an adaptive algorithm for CMCB-FB is a challenging open problem, and the lower bound for CMCB-FB is also worth further investigation.

\begin{ack}
The work of Yihan Du and Longbo Huang is  supported  in  part  by  the  Technology  and Innovation  Major  Project  of  the  Ministry  of  Science  and Technology  of  China  under  Grant 2020AAA0108400 and 2020AAA0108403. 
\end{ack}

\bibliographystyle{plainnat}
\bibliography{neurips_2021_CMCB_ref}

\clearpage
\newgeometry{left=2cm,right=2cm}
\appendix
\section*{Appendix}

\section{Technical Lemmas}
In this section, we introduce two technical lemmas which will be used in our analysis. 

Lemmas~\ref{lemma:semi_concentration_covariance} and \ref{lemma:semi_concentration_means} give the concentration guarantees of algorithm $\algsemibandit$ for CMCB-SB, which sets up a foundation for the concentration guarantees in CMCB-FI.

\begin{lemma}[Concentration of Covariance for CMCB-SB] \label{lemma:semi_concentration_covariance}
	Consider the CMCB-SB problem and algorithm $\algsemibandit$ (Algorithm~\ref{alg:semi_bandit}).
	Define the event 
	\begin{align*}
		\cG_t  \triangleq  \Bigg\{    |\Sigma^*_{ij}-\hat{\Sigma}_{ij,t-1}|  \leq  & 16   \left(  \frac{3 \ln t}{N_{ij}(t-1)}   \vee   \sqrt{\frac{3 \ln t}{N_{ij}(t-1)}} \right)    \\& +   \sqrt{\frac{61 \ln^2 t}{N_{ij}(t-1) N_{i}(t-1)}}   +   \sqrt{\frac{36 \ln^2 t}{N_{ij}(t-1) N_{j}(t-1)}}, \forall i,j \in [d] \Bigg\}
	\end{align*}
	For any $t \geq 2$, we have 
	$$
	Pr[\cG_t] \geq 1- \frac{10 d^2}{t^2}.
	$$
\end{lemma}

\begin{proof}
	According to Proposition 2 in~\cite{covariance-adapting2020},
	we have that for any $t \geq 2$ and $i,j \in [d]$,
	\begin{align*}
	\Pr \Bigg[  |\Sigma^*_{ij}-\hat{\Sigma}_{ij,t-1}|  \leq  & 16 \left( \frac{3 \ln t}{N_{ij}(t-1)}   \vee   \sqrt{\frac{3 \ln t}{N_{ij}(t-1)}} \right)  \\& +   \sqrt{\frac{61 \ln^2 t}{N_{ij}(t-1) N_{i}(t-1)}}   +   \sqrt{\frac{36 \ln^2 t}{N_{ij}(t-1) N_{j}(t-1)}} \Bigg] \leq 1- \frac{10}{ t^2}.
	\end{align*}
	Using a union bound on $i,j \in [d]$, we obtain Lemma~\ref{lemma:semi_concentration_covariance}.
\end{proof}

\begin{lemma}[Concentration of Means for CMCB-SB] \label{lemma:semi_concentration_means}
	Consider the CMCB-SB problem and algorithm $\algsemibandit$ (Algorithm~\ref{alg:semi_bandit}). 
	Let $0<\lambda<1$, and define $\delta_t=\frac{1}{t \ln^2 t}$ and $\beta(\delta_t)=\ln(1/\delta_t) + d\ln \ln t + \frac{d}{2} \ln(1+e/\lambda)$ for $t\geq2$.
	Then, for any $t \geq 2$ and $\bw \in \triangle^c_{d}$, with probability at least $1-\delta_t$, we have
	$$
	\left | \bw^\top \bs{\theta}^*-\bw^\top \bs{\hat{\theta}}_{t-1} \right| \leq \sqrt{2\beta(\delta_t)} \sqrt{\bw^\top D_{t-1}^{-1} \sbr{\lambda \Lambda_{\Sigma^*} D_{t-1} + \sum_{s=1}^{t-1} \Sigma^*_{\bw_s}}  D_{t-1}^{-1} \bw} .
	$$
	Further define $E_t(\bw) = \sqrt{2\beta(\delta_t)} \sqrt{\bw^\top D_{t-1}^{-1} (\lambda \Lambda_{\bar{\Sigma}_{t}} D_{t-1} + \sum_{s=1}^{t-1} \bar{\Sigma}_{s, \bw_s} ) D_{t-1}^{-1} \bw}$. Then, for any $t \geq 2$ and $\bw \in \triangle^c_{d}$, the event $\cH_t \triangleq \{ |\bw^\top \bs{\theta}^*-\bw^\top \bs{\hat{\theta}}_{t-1}| \leq E_t(\bw) \}$ satisfies $\Pr[\cH_t\left.|\right. \cG_t] \geq 1 - \delta_t$.
\end{lemma}

\begin{proof}
	The proof of Lemma~\ref{lemma:semi_concentration_means} follows the analysis procedure in \cite{known_covariance2016}.
	Specifically, assuming that event $\cG_t$ occurs,
	we have $\bar{\Sigma}_{t, ij} \geq {\Sigma}^*_{ij}$ for any $ i,j \in [d]$ and 
	\begin{align*} 
	E_t(\bw) \geq \sqrt{2\beta(\delta_t)} \sqrt{\bw^\top D_{t-1}^{-1} \sbr{\lambda \Lambda_{\Sigma^*} D_{t-1} + \sum_{s=1}^{t-1} \Sigma^*_{\bw_s}}  D_{t-1}^{-1} \bw} .
	\end{align*}
	Hence, to prove Lemma~\ref{lemma:semi_concentration_means}, it suffices to prove that 
	\begin{align} \label{eq:semi_con_mean_first_prove}
	\Pr \mbr{ \left | \bw^\top \bs{\theta}^*-\bw^\top \bs{\hat{\theta}}_{t-1} \right| > \sqrt{2\beta(\delta_t)} \sqrt{\bw^\top D_{t-1}^{-1} \sbr{\lambda \Lambda_{\Sigma^*} D_{t-1} + \sum_{s=1}^{t-1} \Sigma^*_{\bw_s}}  D_{t-1}^{-1} \bw} } \leq  \delta_t .
	\end{align}
	
	Recall that $N_{i}(t)=\sum_{s=1}^t \I\{ w_{s,i} \geq c \}$ and $D_t$ is a diagonal matrix such that $D_{t,ii}=N_{i}(t)$ for any $t > 0$. For any $\bw \in \triangle^c_{d}$, let $I_{\bw}$ denote the diagonal matrix such that $I_{ii}=1$ for any $w_i \geq c$ and  $I_{jj}=0$ for any $w_j = 0$, and let $\Sigma^*_{\bw}=I_{\bw} \Sigma^* I_{\bw}$.
	Let $\bs{\varepsilon}_t$ be the vector such that $\bs{\eta}_t=(\Sigma^*)^{\frac{1}{2}} \bs{\varepsilon}_t$ for any $t>0$.
	
	Let $D$ be a positive definite matrix such that $D \preceq \lambda \Lambda_{\Sigma^*} D_{t-1}$.
	Then, we have for any $\bw \in \triangle^c_{d}$ that 
	\begin{align*}
	\abr{ \bw^\top \sbr{ \bs{\theta}^* - \bs{\hat{\theta}}_{t-1} } } = & \abr{  - \bw^\top D_{t-1}^{-1} \sum_{s=1}^{t-1} I_{\bw_s} (\Sigma^*)^{\frac{1}{2}} \bs{\varepsilon}_s }
	\\
	= & \abr{- \bw^\top D_{t-1}^{-1}  \sbr{D + \sum_{s=1}^{t-1} \Sigma^*_{\bw_s} }^{\frac{1}{2}} \sbr{D + \sum_{s=1}^{t-1} \Sigma^*_{\bw_s} }^{-\frac{1}{2}}  \sum_{s=1}^{t-1} I_{\bw_s} (\Sigma^*)^{\frac{1}{2}} \bs{\varepsilon}_s}
	\\
	\leq & \sqrt{\bw^\top D_{t-1}^{-1}  \sbr{D + \sum_{s=1}^{t-1} \Sigma^*_{\bw_s} } D_{t-1}^{-1} \bw } \cdot 
	\left \| \sum_{s=1}^{t-1} I_{\bw_s} (\Sigma^*)^{\frac{1}{2}} \bs{\varepsilon}_s \right \|_{\sbr{D + \sum_{s=1}^{t-1} \Sigma^*_{\bw_s} }^{-1} }
	\end{align*}
	Let $S_t=\sum_{s=1}^{t-1} I_{\bw_s} (\Sigma^*)^{\frac{1}{2}} \bs{\varepsilon}_s$, $V_t=\sum_{s=1}^{t-1} \Sigma^*_{\bw_s}$ and $I_{D+V_t}=\frac{1}{2}\| S_t \|^2_{\sbr{D+V_t}^{-1}}$. 
	We get 
	$$
	\left \| \sum_{s=1}^{t-1} I_{\bw_s} (\Sigma^*)^{\frac{1}{2}} \bs{\varepsilon}_s \right \|_{\sbr{D + \sum_{s=1}^{t-1} \Sigma^*_{\bw_s} }^{-1} } = \| S_t \|_{\sbr{D+V_t}^{-1} } = \sqrt{2 I_{D+V_t}} .
	$$
	
	Since $ D \preceq \lambda \Lambda_{\Sigma^*} D_{t-1} $, we have 
	\begin{align*}
	\abr{ \bw^\top \sbr{ \bs{\theta}^* - \bs{\hat{\theta}}_{t-1} }} 
	\leq & \sqrt{\bw^\top D_{t-1}^{-1} D D_{t-1}^{-1} \bw^\top + \bw^\top D_{t-1}^{-1} \sbr{\sum_{s=1}^{t-1} \Sigma^*_{\bw_s} } D_{t-1}^{-1} \bw } \cdot \sqrt{2 I_{D+V_t}}
	\\
	\leq & \sqrt{\lambda \bw^\top D_{t-1}^{-1}  \Lambda_{\Sigma^*}  \bw^\top + \bw^\top D_{t-1}^{-1} \sbr{\sum_{s=1}^{t-1} \Sigma^*_{\bw_s} } D_{t-1}^{-1} \bw } \cdot \sqrt{2 I_{D+V_t}}
	\\
	= & \sqrt{ \bw^\top D_{t-1}^{-1}  \sbr{\lambda \Lambda_{\Sigma^*}  D_{t-1} + \sum_{s=1}^{t-1} \Sigma^*_{\bw_s} } D_{t-1}^{-1} \bw } \cdot \sqrt{2 I_{D+V_t}}
	\end{align*}
	Thus, 
	\begin{align*}
	& \Pr \mbr{ \abr{\bw^\top \sbr{ \bs{\theta}^* - \bs{\hat{\theta}}_{t-1} }}  > \sqrt{2\beta(\delta_t)} \sqrt{\bw^\top D_{t-1}^{-1} (\lambda \Lambda_{\bar{\Sigma}_{t}} D_{t-1} + \sum_{s=1}^{t-1} \bar{\Sigma}_{s, \bw_s} ) D_{t-1}^{-1} \bw} }
	\\
	\leq & \Pr \Bigg[ \sqrt{ \bw^\top D_{t-1}^{-1}  \sbr{\lambda \Lambda_{\Sigma^*}  D_{t-1} + \sum_{s=1}^{t-1} \Sigma^*_{\bw_s} } D_{t-1}^{-1} \bw } \cdot \sqrt{2 I_{D+V_t}} \\& \hspace*{8em} > \sqrt{2\beta(\delta_t)} \sqrt{\bw^\top D_{t-1}^{-1} \sbr{\lambda \Lambda_{\Sigma^*} D_{t-1} + \sum_{s=1}^{t-1} \Sigma^*_{\bw_s}}  D_{t-1}^{-1} \bw}  \Bigg]
	\\
	= & \Pr \mbr{  I_{D+V_t} > \beta(\delta_t) }
	\end{align*}
	
	Hence, to prove Eq.~\eqref{eq:semi_con_mean_first_prove}, it suffices to prove
	\begin{align} \label{eq:semi_I_D+V_t_>_beta}
	\Pr \mbr{  I_{D+V_t} > \beta(\delta_t) } \leq \delta_t .
	\end{align}
	
	To do so, we introduce some notions.
	Let $\cJ_t$ be the $\sigma$-algebra $\sigma(\bw_1, \varepsilon_1, \dots, \bw_{t-1}, \varepsilon_{t-1}, \bw_{t})$.
	Let $\bu \in \R^d$ be a multivariate Gaussian random variable with mean $\textbf{0}$ and covariance $D^{-1}$, which is independent of all the other random variables, and  use $\varphi(\bu)$ denote its probability density function. Define  
	$$ P^{\bu}_s = \exp \sbr{ \bu^\top  I_{\bw_s} (\Sigma^*)^{\frac{1}{2}} \bs{\varepsilon}_s - \frac{1}{2} \bu^\top  \Sigma^*_{\bw_s} \bu } ,
	$$
	$$
	M^{\bu}_t \triangleq \exp \sbr{ \bu^\top S_t - \frac{1}{2} \| \bu \|_{V_t}^2 } ,
	$$
	and 
	$$
	M_t \triangleq \ex_{\bu}[ M^{\bu}_t ] = \int_{\R^d} \exp \sbr{ \bu^\top S_t - \frac{1}{2} \| \bu \|_{V_t}^2 } \varphi(\bu) du.
	$$
	We have $M^{\bu}_t = \Pi_{s=1}^{t-1} P^{\bu}_s $.
	In the following, we prove $\ex[M_t] \leq 1$.
	
	For any $s>0$, according to the sub-Gaussian property,
	$\eta_s=(\Sigma^*)^{\frac{1}{2}} \bs{\varepsilon}_s$ satisfies
	$$
	\forall \bv \in \R^d, \  \ex \mbr{e^{\bv^\top (\Sigma^*)^{\frac{1}{2}} \bs{\varepsilon}_s } } \leq e^{\frac{1}{2} \bv^\top \Sigma^* \bv},
	$$
	which is equivalent to 
	$$
	\forall \bv \in \R^d, \  \ex \mbr{e^{\bv^\top (\Sigma^*)^{\frac{1}{2}} \bs{\varepsilon}_s - \frac{1}{2} \bv^\top \Sigma^* \bv} } \leq 1 .
	$$
	Thus, we have
	$$ \ex \mbr{  P^{\bu}_{s} | \cJ_{s} } =  \ex \mbr{ \exp \sbr{ \bu^\top  I_{\bw_s} (\Sigma^*)^{\frac{1}{2}} \bs{\varepsilon}_s - \frac{1}{2} \bu^\top  \Sigma^*_{\bw_s} \bu } | \cJ_{s} }  \leq 1 .
	$$

	Then, we can obtain
	\begin{align*}
	\ex[M^{\bu}_t| \cJ_{t-1}] = & \ex \mbr{ \Pi_{s=1}^{t-1} P^{\bu}_s | \cJ_{t-1} }
	\\
	= & \sbr{\Pi_{s=1}^{t-2} P^{\bu}_s} \ex \mbr{  P^{\bu}_{t-1} | \cJ_{t-1} }
	\\
	\leq & M^{\bu}_{t-1},
	\end{align*}
	which implies that $M^{\bu}_t$ is a super-martingale and $\ex[M^{\bu}_t | \bu] \leq 1$. 
	Thus, 
	$$
	\ex[M_t]=\ex_{\bu} [\ex[M^{\bu}_t | \bu] ] \leq 1 .
	$$
	
	According to Lemma 9 in~\cite{improved_linear_bandit2011}, we have 
	\begin{align*}
	M_t \triangleq \int_{\R^d} \exp \sbr{ \bu^\top S_t - \frac{1}{2} \| \bu \|_{V_t}^2 } \varphi(\bu) du = \sqrt{\frac{\det D}{\det (D+V_t)}} \exp \sbr{ I_{D+V_t} } .
	\end{align*}
	Thus,
	\begin{align*}
	\ex \mbr{ \sqrt{\frac{\det D}{\det (D+V_t)}} \exp \sbr{ I_{D+V_t} } } \leq 1.
	\end{align*}
	
	Now we prove Eq.~\eqref{eq:semi_I_D+V_t_>_beta}. First, we have
	\begin{align}
	\Pr \mbr{  I_{D+V_t} > \beta(\delta_t) }  = & \Pr \mbr{  \sqrt{\frac{\det D}{\det (D+V_t)}} \exp \sbr{ I_{D+V_t} } > \sqrt{\frac{\det D}{\det (D+V_t)}} \exp \sbr{ \beta(\delta_t) }  }
	\nonumber\\
	= & \Pr \mbr{  M_t > \frac{1}{\sqrt{\det ( I+D^{-\frac{1}{2}} V_t D^{-\frac{1}{2}} )}} \exp \sbr{ \beta(\delta_t) }  }
	\nonumber\\
	\leq & \frac{ \ex[M_t] \sqrt{\det ( I+D^{-\frac{1}{2}} V_t D^{-\frac{1}{2}} )} }{ \exp \sbr{ \beta(\delta_t) } } 
	\nonumber\\
	\leq & \frac{ \sqrt{\det ( I+D^{-\frac{1}{2}} V_t D^{-\frac{1}{2}} )} }{ \exp \sbr{ \beta(\delta_t) } } \label{eq:semi_det_exp_beta}
	\end{align}

	Then, for some constant $\gamma>0$ and for any $\ba = (a_1, \dots, a_d) \in \mathbb{N}^d$, we define the set of timesteps $\cK_{\ba} \subseteq [T]$ such that 
	$$
	t \in \cK_{\ba} \Leftrightarrow \forall i \in d, \   (1+\gamma)^{a_i} \leq N_i(t) < (1+\gamma)^{a_i+1} .
	$$
	Define $D_{\ba}$  a diagonal matrix with  $D_{\ba,ii}=(1+\gamma)^{a_i}$. 
	
	Suppose  $t \in \cK_{\ba}$ for some fixed $\ba$. We have
	$$
	\frac{1}{1+\gamma} D_t \preceq D_{\ba} \preceq D_t .
	$$
	Let $D= \lambda \Lambda_{\Sigma^*} D_{\ba} \succeq  \frac{\lambda}{1+\gamma} \Lambda_{\Sigma^*}  D_t$. Then, we have
	$$
	D^{-\frac{1}{2}} V_t D^{-\frac{1}{2}} \preceq \frac{1+\gamma}{\lambda} D_t^{-\frac{1}{2}} \Lambda_{\Sigma^*}^{-\frac{1}{2}}   V_t \Lambda_{\Sigma^*}^{-\frac{1}{2}}  D_t^{-\frac{1}{2}} , 
	$$
	where matrix $D_t^{-\frac{1}{2}} \Lambda_{\Sigma^*}^{-\frac{1}{2}}   V_t \Lambda_{\Sigma^*}^{-\frac{1}{2}}  D_t^{-\frac{1}{2}}$ has $d$ ones on the diagonal. Since the determinant of a positive definite matrix is smaller than the product of its diagonal terms, we have
	\begin{align}
	\det ( I+D^{-\frac{1}{2}} V_t D^{-\frac{1}{2}} ) \leq & \det ( I+ \frac{1+\gamma}{\lambda} D_t^{-\frac{1}{2}} \Lambda_{\Sigma^*}^{-\frac{1}{2}}   V_t \Lambda_{\Sigma^*}^{-\frac{1}{2}}  D_t^{-\frac{1}{2}} )
	\nonumber\\
	\leq & \sbr{1+ \frac{1+\gamma}{\lambda}}^d \label{eq:semi_det_1+gamma}
	\end{align}
	
	Let $0<\lambda<1$ and $\gamma=e-1$.
	Using Eqs.~\eqref{eq:semi_det_exp_beta} and \eqref{eq:semi_det_1+gamma}, $\beta(\delta_t)=\ln(1/\delta_t) + d\ln \ln t + \frac{d}{2} \ln(1+e/\lambda) = \ln(t \ln^2 t) + d\ln \ln t + \frac{d}{2} \ln(1+e/\lambda)$,  and a union bound over $\ba$,  we have
	\begin{align*}
	\Pr \mbr{  I_{D+V_t} > \beta(\delta_t) }  \leq & \sum_{\ba } \Pr \mbr{  I_{D+V_t} > \beta(\delta_t) | t \in \cK_{\ba}, D= \lambda \Lambda_{\Sigma^*} D_{\ba} }  
	\\
	\leq &  \sum_{\ba } \frac{ \sqrt{\det ( I+D^{-\frac{1}{2}} V_t D^{-\frac{1}{2}} )} }{ \exp \sbr{ \beta(\delta_t) } }
	\\
	\leq & \sbr{\frac{ \ln t }{ \ln(1+\gamma) }}^d \cdot \frac{  \sbr{1+ \frac{1+\gamma}{\lambda}}^\frac{d}{2} }{ \exp \sbr{ \ln(t \ln^2 t) + d\ln \ln t + \frac{d}{2} \ln(1+\frac{e}{\lambda}) } }
	\\
	= & \sbr{\ln t}^d \cdot \frac{  \sbr{1+ \frac{e}{\lambda}}^\frac{d}{2} }{ t \ln^2 t  \cdot (\ln t)^d \cdot \sbr{1+ \frac{e}{\lambda}}^\frac{d}{2} }
	\\
	= &  \frac{  1 }{ t \ln^2 t }
	\\
	= &  \delta_t 
	\end{align*}
	Thus, Eq.~\eqref{eq:semi_I_D+V_t_>_beta} holds and we complete the proof of Lemma~\ref{lemma:semi_concentration_means}.
\end{proof}

\section{Proof for CMCB-FI}

\subsection{Proof of Theorem~\ref{thm:ub_full_information}}
In order to prove Theorem~\ref{thm:ub_full_information}, we first have  the following Lemmas~\ref{lemma:full_info_con_covariance} and \ref{lemma:full_info_con_means}, which are adaptions of Lemmas~\ref{lemma:semi_concentration_covariance} and \ref{lemma:semi_concentration_means} to CMCB-FI.

\begin{lemma}[Concentration of Covariance for CMCB-FI]
	\label{lemma:full_info_con_covariance}
	Consider the CMCB-FI problem and algorithm $\empiricalmax$ (Algorithm~\ref{alg:empirical_max}). 
	For any $t \geq 2$, the event
	$$
	\cE_t \triangleq \lbr{ |\Sigma^*_{ij}-\hat{\Sigma}_{ij,t-1}| \leq 16 \left( \frac{3 \ln t}{t-1} \vee \sqrt{\frac{3 \ln t}{t-1}} \right) + \left( 6+4\sqrt{3} \right) \frac{\ln t}{t-1} , \forall i,j \in [d] }
	$$
	satisfies
	$$
	\Pr[\cE_t] \geq 1-\frac{10 d^2}{t^2}
	$$
\end{lemma}
\begin{proof}
	In CMCB-FI, we have $N_{ij}(t-1)=t-1$ for any $t \geq 2$ and $i,j \in [d]$. 
	Then, Lemma~\ref{lemma:full_info_con_covariance} can be obtained by applying Lemma~\ref{lemma:semi_concentration_covariance} with $N_{ij}(t-1)=t-1$ for any $i,j \in [d]$.
\end{proof}

\begin{lemma}[Concentration of Means for CMCB-FI]
	\label{lemma:full_info_con_means}
	Consider the CMCB-FI problem and algorithm $\empiricalmax$ (Algorithm~\ref{alg:empirical_max}). 
	Let $0<\lambda<1$. Define  $\delta_t=\frac{1}{t \ln^2 t}$ and $\beta(\delta_t)=\ln(1/\delta_t) + \ln \ln t + \frac{d}{2} \ln(1+e/\lambda)$ for $t \geq 2$. 
	Define $E_t(\bw)=\sqrt{2 \beta(\delta_t)} \sqrt{\bw^\top D_{t-1}^{-1}(\lambda \Lambda_{\Sigma^*} D_{t-1} + \sum_{s=1}^{t-1} \Sigma^* ) D_{t-1}^{-1} \bw}$. Then, for any $t \geq 2$ and $\bw \in \triangle_{d}$, the event $\cF_t \triangleq \{ |\bw^\top \bs{\theta}^*-\bw^\top \bs{\hat{\theta}}_{t-1}| \leq E_t(\bw) \}$ satisfies $\Pr[\cF_t] \geq 1- \delta_t$.
\end{lemma}

\begin{proof}
	In CMCB-FI, $D_t$ is a diagonal matrix such that $D_{t,ii}=N_{i}(t)=t$.
	Then, Lemma~\ref{lemma:full_info_con_means} can be obtained by applying  Lemma~\ref{lemma:semi_concentration_means} with $D_{t}=t I$ and that the union bound on the number of samples only needs to consider one dimension.
	Specifically, in the proof of Lemma~\ref{lemma:semi_concentration_means}, we replace the set of timesteps $\cK_{\ba}$ with $\cK_{a} \subseteq [T]$ for $a \in \N$, which stands for
	$$
	t \in \cK_{a} \Leftrightarrow  (1+\gamma)^{a} \leq t < (1+\gamma)^{a+1}. 
	$$
	This completes the proof. 
\end{proof}

Now we are ready to prove Theorem~\ref{thm:ub_full_information}.
\begin{proof} (Theorem~\ref{thm:ub_full_information}) 
	Let $\Delta_t = f(\bw^*)-f(\bw_t)$, $g(t) = 16 \left( \frac{3 \ln t}{t-1} \vee \sqrt{\frac{3 \ln t}{t-1}} \right) + \left( 6+4\sqrt{3} \right) \frac{\ln t}{t-1} $ denote the confidence radius of covariance $\Sigma^*_{ij}$ for any $i,j \in [d]$, and $G(t)$ be the matrix with all entries equal to $g(t)$. For any $\bw \in \triangle^c_{d}$, define $\hat{f}_{t}(\bw) = \bw^\top \bs{\hat{\theta}}_{t} - \rho \bw^\top \hat{\Sigma}_{t} \bw$ and $h_t(\bw) = E_t(\bw)+ \rho {\bw}^\top G_t {\bw}$. 
	
	For any $t \geq 2$, suppose that event $\cE_t \cap \cF_t$ occurs. Then, 
	$$|\hat{f}_{t-1}(\bw)-f(\bw)|\leq h_t(\bw).$$ 
	Therefore, we have
	$$\Delta_t \leq |\hat{f}_{t-1}(\bw^*)-f(\bw^*)| +|\hat{f}_{t-1}(\bw_t)-f(\bw_t)| .$$ 
	This is because if instead $\Delta_t > |\hat{f}_{t-1}(\bw^*)-f(\bw^*)| +|\hat{f}_{t-1}(\bw_t)-f(\bw_t)|$, we have
	\begin{align*}
	&\hat{f}_{t-1}(\bw^*)- \hat{f}_{t-1}(\bw_t)
	\\
	= & \hat{f}_{t-1}(\bw^*)- \hat{f}_{t-1}(\bw_t) + (f(\bw^*)- f(\bw_t)) - (f(\bw^*)- f(\bw_t))
	\\
	\geq & \Delta_t - (f(\bw^*)-\hat{f}_{t-1}(\bw^*)) -(\hat{f}_{t-1}(\bw_t)-f(\bw_t))
	\\
	\geq & \Delta_t - |(f(\bw^*)-\hat{f}_{t-1}(\bw^*))| - |(\hat{f}_{t-1}(\bw_t)-f(\bw_t))|
	\\
	> & 0 ,
	\end{align*} 
	which contradicts the selection strategy of $\bw_t$ in algorithm $\empiricalmax$.
Thus, we obtain 
\begin{align}
\Delta_t \leq & |\hat{f}_{t-1}(\bw^*)-f(\bw^*)| +|\hat{f}_{t-1}(\bw_t)-f(\bw_t)| \nonumber
\\
\leq & h_{t}(\bw^*)+ h_{t}(\bw_t) \nonumber
\\
= & E_{t}(\bw^*)+ \rho {\bw^*}^\top G_{t} {\bw^*} + E_{t}(\bw_t)+ \rho {\bw_t}^\top G_{t} {\bw_t} \label{eq:Delta_radius}
\end{align}

Now, for any $\bw \in \triangle^c_{d}$, we have
\begin{align*}
\bw^\top G_t \bw = & \sum_{i,j \in [d]} g(t) w_i w_j
\\
= & g(t) \sum_{i,j \in [d]} w_i w_j
\\
= & g(t) \left( \sum_{i} w_i \right)^2
\\
= & g(t)
\end{align*}
and
\begin{align*}
E_t(\bw) = & \sqrt{2 \beta(\delta_t)} \sqrt{\bw^\top D_{t-1}^{-1} \sbr{ \lambda \Lambda_{\Sigma^*} D_{t-1} + \sum_{s=1}^{t-1} \Sigma^* } D_{t-1}^{-1} \bw}
\\
= & \sqrt{2 \beta(\delta_t)} \sqrt{\lambda \bw^\top D_{t-1}^{-1} \Lambda_{\Sigma^*} \bw + \bw^\top D_{t-1}^{-1} \sbr{ \sum_{s=1}^{t-1} \Sigma^* } D_{t-1}^{-1} \bw}
\\
= & \sqrt{2 \beta(\delta_t)} \sqrt{\lambda \bw^\top D_{t-1}^{-1} \Lambda_{\Sigma^*} \bw  + \bw^\top D_{t-1}^{-1} \Sigma^* \bw }
\\
= & \sqrt{2 \beta(\delta_t)} \sqrt{ \frac{1}{t-1} \lambda \bw^\top  \Lambda_{\Sigma^*} \bw  + \frac{1}{t-1} \bw^\top  \Sigma^* \bw } 
\\
\leq & \sqrt{ \frac{2 \beta(\delta_t)}{t-1} } \sqrt{  \lambda \Sigma^*_{\textup{max}}  +  \bw^\top  \Sigma^* \bw } ,
\end{align*}
where $\Sigma^*_{\textup{max}}$ denotes the maximum diagonal entry of $\Sigma^*$.

For any $t \geq 7$, $\frac{3 \ln t}{t-1} < \sqrt{\frac{3 \ln t}{t-1}}$ and $g(t) \leq \left( 6+20\sqrt{3} \right) \frac{\ln t}{\sqrt{t-1}}$, Eq.~\eqref{eq:Delta_radius} can be written as
\begin{align}
\Delta_t \leq & E_{t}(\bw^*)+ \rho {\bw^*}^\top G_{t} {\bw^*} + E_{t}(\bw_t)+ \rho {\bw_t}^\top G_{t} {\bw_t} 
\nonumber\\
\leq & \sqrt{ \frac{2 \beta(\delta_t)}{t-1} } \sbr{ \sqrt{  \lambda \Sigma^*_{\textup{max}}  +  \bw_t^\top  \Sigma^* \bw_t } + \sqrt{  \lambda \Sigma^*_{\textup{max}}  +  {\bw^*}^\top  \Sigma^* \bw^* }  } + 82 \rho \frac{\ln t}{\sqrt{t-1}}
\nonumber\\
\leq & \sqrt{ \frac{2 \beta(\delta_t)}{t-1} } \sbr{ 2 \sqrt{\lambda \Sigma^*_{\textup{max}}  } + \sqrt{    \bw_t^\top  \Sigma^* \bw_t } + \sqrt{    {\bw^*}^\top  \Sigma^* \bw^* }  } + 82 \rho \frac{\ln t}{\sqrt{t-1}}
\nonumber\\
= & \sqrt{ \frac{2 \beta(\delta_t)}{t-1} } \sbr{ 2 \sqrt{\lambda \Sigma^*_{\textup{max}}  } +  \sqrt{    {\bw^*}^\top  \Sigma^* \bw^* }  } + 82 \rho \frac{\ln t}{\sqrt{t-1}} + \sqrt{ \frac{2 \beta(\delta_t)}{t-1} } \cdot \sqrt{    \bw_t^\top  \Sigma^* \bw_t }   
\label{eq:Delta_leq}
\end{align}

Next, we investigate the upper bound of ${\bw_t}^\top \Sigma^* \bw_t$.
According to Eq.~\eqref{eq:Delta_radius}, we have that  ${\bw_t}^\top \Sigma^* \bw_t$ satisfies
\begin{align}
\Delta_t \leq  h_{t}(\bw^*)+ h_{t}(\bw_t) \label{eq:optimization_constraint}
\end{align}

In Eq.~\eqref{eq:optimization_constraint}, we have 
\begin{align*}
\Delta_t = & f(\bw^*)-f(\bw_t)
\\
\geq & {\theta}^*_{\textup{min}} - \rho {\bw^*}^\top \Sigma^* {\bw^*} -  {\theta}^*_{\textup{max}} + \rho \bw_t^\top \Sigma^* \bw_t
\end{align*}
and
\begin{align*}
  h_{t}(\bw^*)+ h_{t}(\bw_t)= & E_{t}(\bw^*)+ \rho {\bw^*}^\top G_{t} {\bw^*} + E_{t}(\bw_t)+ \rho {\bw_t}^\top G_{t} {\bw_t}
\\
\leq & \sqrt{ \frac{2 \beta(\delta_t)}{t-1} } \sbr{ 2 \sqrt{\lambda \Sigma^*_{\textup{max}}  } + \sqrt{    \bw_t^\top  \Sigma^* \bw_t } + \sqrt{    {\bw^*}^\top  \Sigma^* \bw^* }  } + 82 \rho \frac{\ln t}{\sqrt{t-1}}.
\end{align*}

Thus, ${\bw_t}^\top \Sigma^* \bw_t$ satisfies
\begin{align*}
{\theta}^*_{\textup{min}} - \rho {\bw^*}^\top \Sigma^* {\bw^*} -  {\theta}^*_{\textup{max}} + \rho \bw_t^\top \Sigma^* \bw_t
\leq  &  \sqrt{ \frac{2 \beta(\delta_t)}{t-1} } \sbr{ 2 \sqrt{\lambda \Sigma^*_{\textup{max}}  } + \sqrt{    \bw_t^\top  \Sigma^* \bw_t } + \sqrt{    {\bw^*}^\top  \Sigma^* \bw^* }  } \\& + 82 \rho \frac{\ln t}{\sqrt{t-1}}
\end{align*}
Rearranging the terms, we have 
\begin{align}
 \rho  \bw_t^\top  \Sigma \bw_t & -  \sqrt{ \frac{2 \beta(\delta_t)}{t-1} } \sqrt{ {\bw_t}^\top \Sigma \bw_t}   - \Bigg( {\theta}^*_{\textup{max}}-{\theta}^*_{\textup{min}} +\rho {\bw^*}^\top \Sigma {\bw^*} \nonumber\\& + \sqrt{ \frac{2 \beta(\delta_t)}{t-1} } \sbr{ 2 \sqrt{\lambda \Sigma_{\textup{max}}  } + \sqrt{    {\bw^*}^\top  \Sigma \bw^* }  } + 82 \rho \frac{\ln t}{\sqrt{t-1}} \Bigg) \leq 0 \label{eq:relaxed_feasible_region}
\end{align}

Let $x=\bw_t^\top \Sigma^* \bw_t$ and $0<\lambda<1$.
Define function 
\begin{align*}
y(x)= & \rho x  -  c_1 \sqrt{x} - c_2 \leq 0,
\end{align*}
where  $c_1=\sqrt{ \frac{2 \beta(\delta_t)}{t-1} } >0$ and $c_2= {\theta}^*_{\textup{max}}-{\theta}^*_{\textup{min}} +\rho {\bw^*}^\top \Sigma^* {\bw^*} + \sqrt{ \frac{2 \beta(\delta_t)}{t-1} } \sbr{ 2 \sqrt{\lambda \Sigma^*_{\textup{max}}  } + \sqrt{    {\bw^*}^\top  \Sigma^* \bw^* }  } + 82 \rho \frac{\ln t}{\sqrt{t-1}} >0 $.

Now since
\begin{align*}
y(x)= \rho x  -  c_1 \sqrt{x} - c_2 
= \rho \sbr{\sqrt{x} - \frac{c_1}{2 \rho}}^2 - \frac{c_1^2}{4\rho} -c_2 ,
\end{align*}
by letting $y(x)\leq 0$, we have
\begin{align*}
x \leq & \sbr{\frac{c_1}{2 \rho} + \sqrt{ \frac{c_1^2}{4\rho^2} +\frac{c_2}{\rho} } }^2 
\\
\leq &  2 \frac{c_1^2}{4\rho^2} + 2 \frac{c_1^2}{4\rho^2} + 2 \frac{c_2}{\rho}
\\
= &   \frac{c_1^2}{\rho^2} +  \frac{2 c_2}{\rho}
\end{align*}
Therefore 
\begin{align*}
\bw_t^\top \Sigma^* \bw_t \leq &  \frac{1}{\rho^2} \cdot \frac{2 \beta(\delta_t)}{t-1} +  \frac{2}{\rho} \Bigg( {\theta}^*_{\textup{max}}-{\theta}^*_{\textup{min}} +\rho {\bw^*}^\top \Sigma^* {\bw^*} + \sqrt{ \frac{2 \beta(\delta_t)}{t-1} } \sbr{ 2 \sqrt{\lambda \Sigma^*_{\textup{max}}  } + \sqrt{    {\bw^*}^\top  \Sigma^* \bw^* }  }  \\& + 82 \rho \frac{\ln t}{\sqrt{t-1}} \Bigg)
\\
\leq & 2 {\bw^*}^\top \Sigma^* {\bw^*}+  \frac{2}{\rho} \sbr{ {\theta}^*_{\textup{max}}-{\theta}^*_{\textup{min}} }   +  \frac{2}{\rho} \sqrt{ \frac{2 \beta(\delta_t)}{t-1} } \sbr{ 2 \sqrt{\lambda \Sigma^*_{\textup{max}}  } + \sqrt{    {\bw^*}^\top  \Sigma^* \bw^* }  } + 164  \frac{\ln t}{\sqrt{t-1}} \\& + \frac{1}{\rho^2} \cdot \frac{2 \beta(\delta_t)}{t-1}  
\end{align*}

Thus, we have that, $\bw_t^\top \Sigma^* \bw_t$ satisfies 
\begin{align}
\bw_t^\top \Sigma^* \bw_t \leq  \min \Bigg\{& 2 {\bw^*}^\top \Sigma^* {\bw^*}+  \frac{2}{\rho} \sbr{ {\theta}^*_{\textup{max}}-{\theta}^*_{\textup{min}} }   +  \frac{2}{\rho} \sqrt{ \frac{2 \beta(\delta_t)}{t-1} } \sbr{ 2 \sqrt{\lambda \Sigma^*_{\textup{max}}  } + \sqrt{    {\bw^*}^\top  \Sigma^* \bw^* }  } \nonumber\\& + 164  \frac{\ln t}{\sqrt{t-1}} +\frac{1}{\rho^2} \cdot \frac{2 \beta(\delta_t)}{t-1} , \  \bw_{\textup{max}}^\top \Sigma^* \bw_{\textup{max}} \Bigg \} ,  \label{eq:upper_bound_w_Sigma_w}
\end{align}
where $\bw_{\textup{max}}^\top \triangleq \argmax_{\bw \in \triangle^c_{d}}{\bw^\top \Sigma^* \bw} $.

Below we discuss the two terms in  Eq.~\eqref{eq:upper_bound_w_Sigma_w} separately.

Case (i): Plugging the first term of the upper bound of $\bw_t^\top \Sigma^* \bw_t$ in Eq.~\eqref{eq:upper_bound_w_Sigma_w} into Eq.~\eqref{eq:Delta_leq}, we have that for $t \geq t_0$,
\begin{align*}
	\Delta_t \leq & \sqrt{ \frac{2 \beta(\delta_t)}{t-1} } \sbr{ 2 \sqrt{\lambda \Sigma^*_{\textup{max}}  } +  \sqrt{    {\bw^*}^\top  \Sigma^* \bw^* }  } + 82 \rho \frac{\ln t}{\sqrt{t-1}} + \sqrt{ \frac{2 \beta(\delta_t)}{t-1} } \cdot \sqrt{    \bw_t^\top  \Sigma^* \bw_t }   
	\\
	\leq & \sqrt{ \frac{2 \beta(\delta_t)}{t-1} }  \sbr{  2 \sqrt{\lambda \Sigma^*_{\textup{max}}  } +  \sqrt{    {\bw^*}^\top  \Sigma^* \bw^* }  } + 82 \rho \frac{\ln t}{\sqrt{t-1}} + \sqrt{ \frac{2 \beta(\delta_t)}{t-1} } \cdot  \\&   \sqrt{ 2 {\bw^*}^\top \Sigma^* {\bw^*} \!+\!  \frac{2}{\rho}  \sbr{ {\theta}^*_{\textup{max}} \!-\! {\theta}^*_{\textup{min}} } \!+\! \frac{2}{\rho} \sqrt{ \frac{2 \beta(\delta_t)}{t-1} } \!\! \sbr{ \! 2 \sqrt{\lambda \Sigma^*_{\textup{max}}  } \!+\! \sqrt{    {\bw^*}^\top  \Sigma^* \bw^* }  } \!+\! 164 \frac{\ln t}{\sqrt{t-1}} \!+\! \frac{1}{\rho^2} \cdot \frac{2 \beta(\delta_t)}{t-1} } 
	\\
	\leq & \sqrt{ \frac{2 \beta(\delta_t)}{t-1} } \sbr{ 2 \sqrt{\lambda \Sigma^*_{\textup{max}}  } +  \sqrt{    {\bw^*}^\top  \Sigma^* \bw^* }  } + 82 \rho \frac{\ln t}{\sqrt{t-1}} + \sqrt{ \frac{2 \beta(\delta_t)}{t-1} }   \cdot  \\&  \Bigg( \sqrt{ 2 {\bw^*}^\top \Sigma^* {\bw^*} } + \frac{\sqrt{2}}{\sqrt{\rho}} \sqrt{  {\theta}^*_{\textup{max}}-{\theta}^*_{\textup{min}} }  + \frac{\sqrt{2}}{\sqrt{\rho}} \sbr{\frac{2 \beta(\delta_t)}{t-1}}^{\frac{1}{4}}  \sbr{\sqrt{2}   \sbr{\lambda \Sigma^*_{\textup{max}} }^{\frac{1}{4}} + \sbr{ {\bw^*}^\top  \Sigma^* \bw^* }^{\frac{1}{4}} } \\& + 13 \frac{ \sqrt{\ln t} }{ (t-1)^{\frac{1}{4}}} +\frac{1}{\rho}  \sqrt{\frac{2 \beta(\delta_t)}{t-1}}  \Bigg)    
	\\
	\leq & \sqrt{ \frac{2 \beta(\delta_t)}{t-1} } \sbr{ 2 \sqrt{\lambda \Sigma^*_{\textup{max}}  } +  \sqrt{    {\bw^*}^\top  \Sigma^* \bw^* } + \sqrt{ 2 {\bw^*}^\top \Sigma^* {\bw^*} } + \frac{\sqrt{2}}{\sqrt{\rho}} \sqrt{  {\theta}^*_{\textup{max}}-{\theta}^*_{\textup{min}} }   } + 82 \rho \frac{\ln t}{\sqrt{t-1}}  \\& 
	+   \frac{\sqrt{2}}{\sqrt{\rho}} \sbr{\frac{2 \beta(\delta_t)}{t-1}}^{\frac{3}{4}}  \sbr{\sqrt{2}   \sbr{\lambda \Sigma^*_{\textup{max}} }^{\frac{1}{4}} + \sbr{ {\bw^*}^\top  \Sigma^* \bw^* }^{\frac{1}{4}} } + 42 \frac{ \beta(\delta_t)}{ (t-1)^{\frac{3}{4}} } +\frac{1}{\rho} \cdot \frac{2 \beta(\delta_t)}{t-1}   
	\\
	\leq & \sqrt{ \frac{2 \beta(\delta_t)}{t-1} } \sbr{ 2 \sqrt{\lambda \Sigma^*_{\textup{max}}  } +  \sqrt{    {\bw^*}^\top  \Sigma^* \bw^* } + \sqrt{ 2 {\bw^*}^\top \Sigma^* {\bw^*} } + \frac{\sqrt{2}}{\sqrt{\rho}} \sqrt{  {\theta}^*_{\textup{max}}-{\theta}^*_{\textup{min}} }   } + 82 \rho \frac{\ln t}{\sqrt{t-1}}  \\& 
	+   \frac{\sqrt{2}}{\sqrt{\rho}} \sbr{\frac{2 \beta(\delta_t)}{t-1}}^{\frac{3}{4}}  \sbr{\sqrt{2}   \sbr{\lambda \Sigma^*_{\textup{max}} }^{\frac{1}{4}} + \sbr{ {\bw^*}^\top  \Sigma^* \bw^* }^{\frac{1}{4}} } + 42 \frac{ \beta(\delta_t)}{ (t-1)^{\frac{3}{4}} } +\frac{1}{\rho} \cdot \frac{2 \beta(\delta_t)}{t-1}
\end{align*}

According to Lemmas~\ref{lemma:full_info_con_covariance} and \ref{lemma:full_info_con_means}, for any $t \geq 2$, we bound the probability of event $\neg(\cE_t \cap \cF_t)$ as follows.
\begin{align*}
\Pr \mbr{ \neg(\cE_t \cap \cF_t) } \leq & \frac{10 d^2}{t^2} + \frac{1}{t \ln^2 t}
\\
\leq & \frac{10 d^2}{t \ln^2 t} + \frac{1}{t \ln^2 t}
\\
= &  \frac{11 d^2}{t \ln^2 t}
\end{align*}

Recall that
$$
\beta(\delta_t)=\ln(t \ln^2 t) + \ln \ln t + \frac{d}{2} \ln(1+e/\lambda) = O (\ln t + d \ln(1+\lambda^{-1})) .
$$

Let $t_0=7$. For any horizon $T$ which satisfies $T \geq 1+\lambda^{-1}$, summing over $t=1,\dots,T$, we obtain the regret upper bound 
\begin{align*}
	\ex[\cR(T)] = & O(t_0) +  \sum_{t=t_0}^{T} O \left(  \Delta_{\textup{max}} \cdot \Pr \mbr{ \neg(\cE_t \cap \cF_t) } + \Delta_t \cdot \I \lbr{\cE_t \cap \cF_t}  \right) 
	\\
	= & \sum_{t=t_0}^{T} O \left( \Delta_{\textup{max}} \cdot \frac{d^2}{t \ln^2 t} + \Delta_t \cdot \I \lbr{\cE_t \cap \cF_t} \right) 
	\\
	= & \sum_{t=t_0}^{T} O \Bigg( \sqrt{ \frac{ \ln t + d\ln(1+\lambda^{-1})}{t-1} } \sbr{  \sqrt{\lambda \Sigma^*_{\textup{max}}  } +  \sqrt{    {\bw^*}^\top  \Sigma^* \bw^* } + \frac{1}{\sqrt{\rho}} \sqrt{  {\theta}^*_{\textup{max}}-{\theta}^*_{\textup{min}} }   } +  \rho \frac{\ln t}{\sqrt{t-1}} \Bigg) 
	\\
	= & O \Bigg( \ln T \sqrt{d T}  \sbr{  \sqrt{\lambda \Sigma^*_{\textup{max}}  } +  \sqrt{    {\bw^*}^\top  \Sigma^* \bw^* } + \frac{1}{\sqrt{\rho}} \sqrt{  {\theta}^*_{\textup{max}}-{\theta}^*_{\textup{min}} } +\rho  } \Bigg) 
\end{align*}

Case (ii): Plugging the second term of the upper bound of $\bw_t^\top \Sigma^* \bw_t$ in Eq.~\eqref{eq:upper_bound_w_Sigma_w} into Eq.~\eqref{eq:Delta_leq}, we have that for $t \geq t_0$,
\begin{align*}
\Delta_t \leq & \sqrt{ \frac{2 \beta(\delta_t)}{t-1} } \sbr{ 2 \sqrt{\lambda \Sigma^*_{\textup{max}}  } + \sqrt{  \bw_{\textup{max}}^\top \Sigma^* \bw_{\textup{max}} } + \sqrt{    {\bw^*}^\top  \Sigma^* \bw^* }  } + 82 \rho \frac{\ln t}{\sqrt{t-1}}
\\
\leq & 2 \sqrt{ \frac{ 2\beta(\delta_t)}{t-1} }  \sbr{  \sqrt{\lambda \Sigma^*_{\textup{max}}  } + \sqrt{  \bw_{\textup{max}}^\top \Sigma^* \bw_{\textup{max}} } } + 82 \rho \frac{\ln t}{\sqrt{t-1}}
\end{align*}

For any horizon $T$ which satisfies $T \geq 1+\lambda^{-1}$, summing over $t=1,\dots,T$, we obtain the regret upper bound  
\begin{align*}
\ex[\cR(T)] = & O(t_0) + \sum_{t=t_0}^{T} O \left( \Delta_{\textup{max}} \cdot \Pr \mbr{ 
\neg(\cE_t \cap \cF_t) } + \Delta_t \cdot \I \lbr{\cE_t \cap \cF_t}  
\right) 
\\
= & \sum_{t=t_0}^{T} O \Bigg( \Delta_{\textup{max}} \cdot \frac{d^2}{t \ln^2 t} + \sqrt{ \frac{ \beta(\delta_t)}{t-1} }  \sbr{  \sqrt{\lambda \Sigma^*_{\textup{max}}  } + \sqrt{  \bw_{\textup{max}}^\top \Sigma^* \bw_{\textup{max}} } } +  \rho \frac{\ln t}{\sqrt{t-1}} \Bigg) 
\\
= & \sum_{t=t_0}^{T} O \Bigg(  \sqrt{ \frac{ \ln t + d\ln(1+\lambda^{-1})}{t-1} }  \sbr{  \sqrt{\lambda \Sigma^*_{\textup{max}}  } + \sqrt{  \bw_{\textup{max}}^\top \Sigma^* \bw_{\textup{max}} } } +  \rho \frac{\ln t}{\sqrt{t-1}} \Bigg) 
\\
= & O \Bigg( \ln T \sqrt{d T} \sbr{  \sqrt{\lambda \Sigma^*_{\textup{max}}  } + \sqrt{  \bw_{\textup{max}}^\top \Sigma^* \bw_{\textup{max}} } + \rho } \Bigg) 
\end{align*}

Combining cases (i) and (ii), we can obtain
\begin{align*}
\ex[\cR(T)] 
= & O \Bigg( \ln T \sqrt{d T} \bigg( \min \lbr{ \sqrt{{\bw^*}^\top  \Sigma^* \bw^*} + \rho^{-\frac{1}{2}} \sqrt{  {\theta}^*_{\textup{max}}-{\theta}^*_{\textup{min}} }, \sqrt{ \bw_{\textup{max}}^\top \Sigma^* \bw_{\textup{max}} } }
+ \sqrt{\lambda \Sigma^*_{\textup{max}}  }   +\rho \bigg) \Bigg) 
\end{align*}

Let $\Sigma^*_{\max} = \max_{i \in [d]} \Sigma^*_{ii}$.
Letting $\lambda=\frac{{\bw^*}^\top \Sigma^* {\bw^*}}{\Sigma^*_{\textup{max}}}$ and $T \geq 1+\frac{\Sigma^*_{\textup{max}}}{{\bw^*}^\top \Sigma^* {\bw^*}}$, we obtain
\begin{align*}
\ex[\cR(T)] = & O \Bigg( \ln T \sqrt{d T} \bigg( \min \lbr{ \sqrt{{\bw^*}^\top  \Sigma^* \bw^*} + \rho^{-\frac{1}{2}} \sqrt{  {\theta}^*_{\textup{max}}-{\theta}^*_{\textup{min}} }, \sqrt{ \bw_{\textup{max}}^\top \Sigma^* \bw_{\textup{max}} } }
 +\rho \bigg) \Bigg) 
\\
= & O \Bigg( \ln T \sqrt{d T} \bigg( \min \lbr{ \sqrt{{\bw^*}^\top  \Sigma^* \bw^*} + \rho^{-\frac{1}{2}} \sqrt{  {\theta}^*_{\textup{max}}-{\theta}^*_{\textup{min}} }, \sqrt{ \Sigma^*_{\textup{max}} } }
+\rho \bigg) \Bigg) ,
\end{align*}
which completes the proof.
\end{proof}

\subsection{Proof of Theorem~\ref{thm:lb_full_information}}

In order to prove Theorem~\ref{thm:lb_full_information}, we first analyze the offline problem of CMCB-FI. Suppose that the covariance matrix $\Sigma^*$ is positive definite.

\paragraph{(Offline Problem of CMCB-FI)} 
We define the quadratic optimization $\quadopt(\bs{\theta}^*, \Sigma^*)$ as 
\begin{align*}
\min_{\bw} & \quad f(\bw)=\rho \bw^\top \Sigma^* \bw - \bw^\top \bs{\theta}^*
\\
s.t. & \quad w_i \geq 0, \quad \forall i \in [d]
\\
& \quad \sum_{i=1}^{d} w_i = 1
\end{align*}
and $\bw^*$ as the optimal solution to $\quadopt(\bs{\theta}^*, \Sigma^*)$.
We consider the KKT condition for this quadratic optimization as follows:
\begin{align*}
2 \rho \Sigma^* \bw - \bs{\theta}^* - \bu -v \unitvec &= 0
\\
w_i u_i &=0, \quad \forall i \in [d]
\\
u_i &\geq 0, \quad \forall i \in [d]
\\
w_i &\geq 0, \quad \forall i \in [d]
\\
\sum_{i=1}^{d} w_i &= 1
\end{align*}
Let $S \subseteq [d]$ be a subset of indexes for $\bw$ such that $S=\{i \in [d]: w_i>0\}$. Let $ \bar{S}=[d] \setminus S$ and we have $\bar{S}=\{i \in [d]: w_i=0\}$. Then, from the KKT condition, we have
\begin{align}
\bw_S = &\frac{1}{2 \rho} (\Sigma^*_S)^{-1} \bs{\theta}^*_S + \frac{1-\| \frac{1}{2 \rho} (\Sigma^*_S)^{-1} \bs{\theta}^*_S \|}{\| (\Sigma^*_S)^{-1} \unitvec \|} (\Sigma^*_S)^{-1} \unitvec \succ \textbf{0}
\label{eq:w_S}
\\
\bw_{\bar{S}} = & \textbf{0} \nonumber
\\
v= & \frac{2 \rho (1-\| \frac{1}{2 \rho} (\Sigma^*_S)^{-1} \bs{\theta}^*_S \|)}{\| (\Sigma^*_S)^{-1} \unitvec \|}  \nonumber
\\
\bu = & 2 \rho \Sigma^* \bw - \bs{\theta}^* - v \unitvec \succeq  \textbf{0} \nonumber
\end{align}
Since this problem is a quadratic optimization and the covariance matrix $\Sigma^*$ is positive-definite, there is a unique feasible $S$ satisfying the above inequalities and the solution $\bw_S, \bw_{\bar{S}}$ is the optimal solution $\bw^*$.

\paragraph{Main Proof.}
Now, we give the proof of Theorem~\ref{thm:lb_full_information}.
\begin{proof} (Theorem~\ref{thm:lb_full_information}) 
	First, we choose prior distributions for $\bs{\theta}^*$ and $\Sigma^*$. We assume that $\bs{\theta}^* \sim \cN(0, \frac{1}{\omega}I)$ and $\Sigma^* \sim \pi_{I}$, where $\omega>0$ and $\pi_{I}$ takes probability $1$ at the support $I$ and probability $0$ anywhere else.
	Define $\bs{\hat{\theta}}_t \triangleq \frac{1}{t} \sum_{i=1}^{t} \theta_i$ and $\bs{\mu}_t=\frac{t}{t+\omega}\bs{\hat{\theta}}_t$. Then, we see that $\bs{\hat{\theta}}_t \sim \cN(0, \frac{\omega+t}{t \omega}I)$  and  $\bs{\mu}_t \sim \cN(0, \frac{t}{(t+\omega)\omega}I)$.

	Thus, the posterior of $\bs{\theta}^*$ is given by 
	$$ 
	\bs{\theta}^* | \theta_1, \dots, \theta_t, \Sigma^* \sim \cN \sbr{\frac{t}{t+\omega}\bs{\hat{\theta}}_t, \frac{1}{t+\omega}I} = \cN \sbr{\bs{\mu}_t, \frac{1}{t+\omega}I}.
	$$ 
	The posterior of $\Sigma^*$ is still $\Sigma^* \sim \pi_{I}$, i.e., $\Sigma^*$ is always a fixed identity matrix.
	
	Under the Bayesian setting, the expected regret is givenn by 
	\begin{align*}
	\sum_{t=1}^{T} \ex_{\bs{\mu}_t \sim \cN(0, \frac{t}{(t+\omega)\omega}I)} \mbr{ \ex_{\bs{\theta^*}|\bs{\mu}_t  \sim \cN(\bs{\mu}_t, \frac{1}{t+\omega}I)} \mbr{  f(\bw^*)-f(\bw_t) } } .
	\end{align*}
	Recall that $\bw^*$ is the optimal solution to $\quadopt(\bs{\theta}^*, \Sigma^*)$.
	It can be seen that the best strategy of $\bw_t$ at timestep $t$ is to select the optimal solution to $\quadopt(\bs{\mu}_t, \Sigma^*)$ and we use algorithm $\cA$ to denote this strategy. Thus, to obtain a regret lower bound for the problem, it suffices to  prove a regret lower bound of algorithm $\cA$ for the problem.
	
	Below we prove a regret lower bound of algorithm $\cA$ for the problem.
	\paragraph{Step (i).} We consider the case when $\bw^*$ and $\bw_t$ both lie in the interior of the $d$-dimensional probability simplex $\triangle_{d}$, i.e., $w^*_i >0, \forall i \in [d]$ and $w_{t,i}>0, \forall i \in [d]$.
	From Eq.~\eqref{eq:w_S}, $\bw^*$ satisfies
	\begin{align*}
	\frac{1}{2 \rho} I^{-1} \bs{\theta}^* + \frac{1-\| \frac{1}{2 \rho} I^{-1} \bs{\theta}^* \|}{\| I^{-1} \unitvec \|} I^{-1} \unitvec \succ \textbf{0} .
	\end{align*}
	Rearranging the terms, we have
	\begin{align*}
	\| \bs{\theta}^* \|\unitvec - d \bs{\theta}^* \prec  2 \rho \unitvec ,
	\end{align*}
	which is equivalent to
	\begin{align} \label{eq:theta_star_interior}
	\left\{\begin{matrix}
	\theta^*_2+ \dots + \theta^*_d-(d-1)\theta^*_1 & < 2 \rho
	\\
	\vdots
	\\ 
	\theta^*_1+ \dots + \theta^*_{d-1}-(d-1)\theta^*_d & < 2 \rho
	\end{matrix}\right.
	\end{align}
	
	Similarly, $\bw_t$ satisfies
	\begin{align} \label{eq:mu_t_interior}
	\left\{\begin{matrix} 
	\mu_2+ \dots + \mu_d-(d-1)\mu_1 & < 2 \rho
	\\
	\vdots
	\\ 
	\mu_1+ \dots + \mu_{d-1}-(d-1)\mu_d & < 2 \rho
	\end{matrix}\right.
	\end{align}
	
	We first derive a condition that makes $\bs{\mu}_t$ lie in the interior of $\triangle_{d}$.  
	Recall that $\bs{\mu}_t \sim \cN(0, \frac{t}{(t+\omega)\omega}I)$. Define event 
	$$
	\cE_t \triangleq  \lbr{ -3 \sqrt{\frac{t}{(t+\omega)\omega}} \leq \mu_t \leq 3 \sqrt{\frac{t}{(t+\omega)\omega}} }.
	$$
	According to the $3-\sigma$ principle for Gaussian distributions, we have
	$$
	\Pr \mbr{\cE_t} \geq (99.7\%)^d .
	$$ 
	Conditioning on $\cE_t$, under which  Eq.~\eqref{eq:mu_t_interior} hold, it suffices to let
	$$
	3 (d-1) \sqrt{\frac{t}{(t+\omega)\omega}} - \sbr{ -3 (d-1) \sqrt{\frac{t}{(t+\omega)\omega}} } < \rho,
	$$ 
	which is equivalent to
	\begin{align} \label{eq:mu_t_interior_inequality}
	\sbr{1+\frac{\omega}{t}} \omega> \frac{36(d-1)^2 }{ \rho^2} ,
	\end{align}
	when $t>0, m>0, t+\omega>0$. Let $t_1>0$ be the smallest timestep that satisfies Eq.~\eqref{eq:mu_t_interior_inequality}. Thus, when $\cE_t$ occurs and $t \geq t_1$, $\bs{\mu}_t$ lie in the interior of $\triangle_{d}$. 
	
	Next, we derive some condition that make $\bs{\theta}^*$ lie in the interior of $\triangle_{d}$.  
	Recall that $\bs{\theta^*}|\bs{\mu}_t  \sim \cN(\bs{\mu}_t, \frac{1}{t+\omega}I)$.
	
	Fix $\bs{\mu}_t$, and then we define event 
	$$
	\cF_t \triangleq  \lbr{ -3 \sqrt{\frac{t}{(t+\omega)\omega}} \leq \bs{\theta}^*-\bs{\mu}_t \leq 3 \sqrt{\frac{t}{(t+\omega)\omega}} }.
	$$
	According to the $3-\sigma$ principle for Gaussian distributions, we have 
	$$
	\Pr \mbr{\cF_t} \geq (99.7\%)^d .
	$$ 
	Conditioning on $\cF_t$, in order to let Eq.~\eqref{eq:theta_star_interior} hold, it suffices to let
	$$
	3 (d-1) \frac{1}{\sqrt{t+\omega}} - \sbr{ -3 (d-1) \frac{1}{\sqrt{t+\omega}}  } < \rho,
	$$ 
	which is equivalent to
	\begin{align} \label{eq:theta_star_interior_inequality}
	t  > \frac{36(d-1)^2}{\rho^2} -\omega,
	\end{align}
	when $t+\omega>0$. Let $t_2>0$ be the smallest timestep that satisfies Eq.~\eqref{eq:theta_star_interior_inequality}. Thus, when $\cF_t$ occurs and $t \geq t_2$, $\bs{\theta}^*$ lie in the interior of $\triangle_{d}$. 
	
	\paragraph{Step (ii).} Suppose that $\cE_t \cap \cF_t \cap \cG_t$ occurs and consider $t \geq \tilde{t} \triangleq \max \{t_1, t_2\}$. Then, $\bw^*$ and $\bw_t$ both lie in the interior of $\triangle_{d}$, i.e., $w^*_i >0, \forall i \in [d]$ and $w_{t,i}>0, \forall i \in [d]$. 
	We have
	\begin{align*}
	\bw^* = &\frac{1}{2 \rho} (\Sigma^*)^{-1} \bs{\theta}^* + \frac{1-\| \frac{1}{2 \rho} (\Sigma^*)^{-1} \bs{\theta}^* \|}{\| (\Sigma^*)^{-1} \unitvec \|} (\Sigma^*)^{-1} \unitvec 
	\\
	\bw_t = &\frac{1}{2 \rho} (\Sigma^*)^{-1} \bs{\mu}_t + \frac{1-\| \frac{1}{2 \rho} (\Sigma^*)^{-1} \bs{\mu}_t \|}{\| (\Sigma^*)^{-1} \unitvec \|} (\Sigma^*)^{-1} \unitvec 
	\end{align*}
	
	Let $\Delta \bs{\theta}_t \triangleq  \bs{\mu}_t-\bs{\theta}^*$ and thus $ \Delta \bs{\theta}_t|\bs{\mu}_t  \sim \cN(0, \frac{1}{t+\omega}I)$. 
	Let $\Delta \bw_t \triangleq \bw_t-\bw^*= \frac{1}{2 \rho} (\Sigma^*)^{-1} \Delta \bs{\theta}_t - \frac{ \| \frac{1}{2 \rho} (\Sigma^*)^{-1} \Delta \bs{\theta}_t \| }{\| (\Sigma^*)^{-1} \unitvec \|} (\Sigma^*)^{-1} \unitvec = \frac{1}{2 \rho}  \Delta \bs{\theta}_t - \frac{1}{2 \rho d} \| \Delta \bs{\theta}_t \| \unitvec $. 
	Then, we have
	\begin{align*}
	f(\bw^*)-f(\bw_t) =& f(\bw^*)-f(\bw^*+\Delta \bw_t)
	\\
	=& \sbr{ (\bw^*)^\top \bs{\theta}^*-\rho (\bw^*)^\top \Sigma^* \bw^* } 
	\\& - \sbr{ (\bw^*)^\top \bs{\theta}^*+ (\Delta \bw_t)^\top \bs{\theta}^* -\rho (\bw^*)^\top \Sigma^* \bw^* - 2\rho (\Delta \bw_t)^\top \Sigma^* \bw^*-\rho (\Delta \bw_t)^\top \Sigma^* \Delta \bw_t }
	\\
	=& - (\Delta \bw_t)^\top \bs{\theta}^* + 2\rho (\Delta \bw_t)^\top \Sigma^* \bw^* + \rho (\Delta \bw_t)^\top \Sigma^* \Delta \bw_t
	\\
	=& (\Delta \bw_t)^\top \nabla f( \bs{\theta}^* ) + \rho (\Delta \bw_t)^\top \Sigma^* \Delta \bw_t
	\\
	\geq & \rho (\Delta \bw_t)^\top \Sigma^* \Delta \bw_t
	\\
	= & \rho (\Delta \bw_t)^\top  \Delta \bw_t
	\\
	= & \rho \sbr{ \frac{1}{4 \rho^2}  (\Delta \bs{\theta}_t)^\top  \Delta \bs{\theta}_t - \frac{1}{4 \rho^2 d^2} \| \Delta \bs{\theta}_t \|^2 d }
	\\
	= &   \frac{1}{4 \rho} \sbr{ \sum_{i=1}^d \Delta \theta_{t,i}^2 - \frac{\sbr{ \sum_{i=1}^d \Delta \theta_{t,i} }^2 }{d} }
	\\
	= & \frac{1}{4 \rho d} \sum_{1 \leq i<j \leq d} \sbr{\Delta \theta_{t,i}-\Delta \theta_{t,j} }^2
	\end{align*}
	Let $i_1=1, j_1=2, i_2=3, j_2=4, \dots, i_{\left \lceil \frac{d}{2} \right \rceil}=2 \left \lceil \frac{d}{2} \right \rceil-1, j_{\left \lceil \frac{d}{2} \right \rceil}=2 \left \lceil \frac{d}{2} \right \rceil $.
	For any $i_k, j_k$ ($k \in [\left \lceil \frac{d}{2} \right \rceil] $), $\Delta \theta_{t,i_k}-\Delta \theta_{t,j_k} | \bs{\mu}_t \sim \cN(0, \frac{2}{t+\omega})$ and they are mutually independent among $k$.
\yihan{Revised to ``fix $\bs{\mu}_t$''.}
	
	Fix $\bs{\mu}_t$, and then we define event
	$$
	\cG_t \triangleq \lbr{ |\Delta \theta_{t,i_k}-\Delta \theta_{t,j_k}| \geq 0.3 \sqrt{ \frac{2}{t+\omega} }, \forall k \in \mbr{\left \lceil \frac{d}{2} \right \rceil}  } .
	$$
	From the c.d.f. of Gaussian distributions, we have 
	$$
	\Pr[\cG_t] \geq (75\%)^{\left \lceil \frac{d}{2} \right \rceil} .
	$$
	From this, we get 
	\begin{align*}
	\Pr\mbr{\cF_t \cap \cG_t} \geq & 1- \Pr\mbr{\bar{\cF_t}} - \Pr\mbr{\bar{\cG_t}}
	\\
	\geq & 1 - \sbr{ 1-(99.7\%)^d } - \sbr{ 1-(75\%)^{\left \lceil \frac{d}{2} \right \rceil} } 
	\\
	\geq & (99.7\%)^d + (75\%)^{\left \lceil \frac{d}{2} \right \rceil} - 1 .
	\end{align*}
	When $d \leq 18$, $\Pr\mbr{\cF_t \cap \cG_t} \geq (99.7\%)^d + (75\%)^{\left \lceil \frac{d}{2} \right \rceil} - 1 > 0$.

	\paragraph{Step (iii).} We bound the expected regret by considering the event $\cE_t \cap \cF_t \cap \cG_t$ and  $t \geq \tilde{t} $. Specifically, 
	\begin{align*}
	& \sum_{t=1}^{T} \ex_{\bs{\mu}_t \sim \cN(0, \frac{t}{(t+\omega)\omega}I)} \mbr{ \ex_{\bs{\theta^*}|\bs{\mu}_t  \sim \cN(\bs{\mu}_t, \frac{1}{t+\omega}I)} \mbr{  f(\bw^*)-f(\bw_t) } } 
	\\
	\geq & \sum_{t=\tilde{t}}^{T} \ex_{\bs{\mu}_t \sim \cN(0, \frac{t}{(t+\omega)\omega}I)} \mbr{ \ex_{\bs{\theta^*}|\bs{\mu}_t  \sim \cN(\bs{\mu}_t, \frac{1}{t+\omega}I)} \mbr{  f(\bw^*)-f(\bw_t) } | \cE_t} \Pr\mbr{\cE_t}
	\\
	\geq & \sum_{t=\tilde{t}}^{T} \ex_{\bs{\mu}_t \sim \cN(0, \frac{t}{(t+\omega)\omega}I)} \mbr{ \ex_{\bs{\theta^*}|\bs{\mu}_t  \sim \cN(\bs{\mu}_t, \frac{1}{t+\omega}I)} \mbr{  f(\bw^*)-f(\bw_t) | \cF_t \cap \cG_t } \Pr\mbr{\cF_t \cap \cG_t} | \cE_t} \Pr\mbr{\cE_t}
	\\
	\geq & \sum_{t=\tilde{t}}^{T} \ex_{\bs{\mu}_t \sim \cN(0, \frac{t}{(t+\omega)\omega}I)} \! \Bigg[ \ex_{\bs{\theta^*}|\bs{\mu}_t  \sim \cN(\bs{\mu}_t, \frac{1}{t+\omega}I)} \! \Bigg[  \frac{1}{4 \rho d} \sum_{1 \leq i<j \leq d} \sbr{\Delta \theta_{t,i}-\Delta \theta_{t,j} }^2 | \cF_t \cap \cG_t \cap \cE_t \!\Bigg] \cdot  \Pr\mbr{\cF_t \cap \cG_t} \! \Bigg] \!\! \Pr\mbr{\cE_t}
	\\
	\geq & \sum_{t=\tilde{t}}^{T} \ex_{\bs{\mu}_t \sim \cN(0, \frac{t}{(t+\omega)\omega}I)} \!\Bigg[ \ex_{\bs{\theta^*}|\bs{\mu}_t  \sim \cN(\bs{\mu}_t, \frac{1}{t+\omega}I)} \!\mbr{  \frac{  {\left \lceil \frac{d}{2} \right \rceil} }{4 \rho d}  \cdot \frac{0.3^2 \cdot 2}{ t+\omega }   | \cF_t \cap \cG_t \cap \cE_t \!} \cdot  \sbr{ (99.7\%)^d + (75\%)^{\left \lceil \frac{d}{2} \right \rceil} - 1 } \!\Bigg]  (99.7\%)^d 
	\\
	= & \sum_{t=\tilde{t}}^{T} \frac{ 0.01125 }{ \rho  (t+\omega)}  \sbr{ (99.7\%)^d + (75\%)^{\left \lceil \frac{d}{2} \right \rceil} - 1 } (99.7\%)^d 
	\\
	= &  \frac{ 0.01125 \sbr{ (99.7\%)^d + (75\%)^{\left \lceil \frac{d}{2} \right \rceil} - 1 } (99.7\%)^d }{ \rho  } \ln \sbr{ \frac{t+\omega}{\tilde{t}+\omega} }  .
	\end{align*}
	
	In the following, we consider an intrinsic bound of the expected regret and set the problem parameters to  proper quantities.
	Since the expected regret is upper bounded by $\Delta_{\textup{max}} T$ and
	\begin{align*}
	\Delta_{\textup{max}} = & f_{\textup{max}} - f_{\textup{min}}
	\\
	\geq & f\sbr{ \frac{1}{d} \unitvec } - \sbr{ \theta^*_{\textup{min}} - \rho }
	\\
	= & \frac{1}{d} \sum_{i=1}^{d} \theta^*_i - \rho \frac{1}{d} - \sbr{ \theta^*_{\textup{min}} - \rho }
	\\
	= & \frac{1}{d} \sum_{i=1}^{d} \theta^*_i  -  \theta^*_{\textup{min}} + \frac{d-1}{d} \rho
	\\ 
	> & 0 ,
	\end{align*}
	we conclude that the expected regret is lower bounded by 
	\begin{align*}
	 \min \Bigg\{ \frac{ 0.01125 \sbr{ (99.7\%)^d + (75\%)^{\left \lceil \frac{d}{2} \right \rceil} - 1 } (99.7\%)^d }{ \rho  }  \ln \sbr{ \frac{t+\omega}{\tilde{t}+m} },
	 \sbr{\frac{1}{d} \sum_{i=1}^{d} \theta^*_i  -  \theta^*_{\textup{min}} + \frac{d-1}{d} \rho} T \Bigg\} .
	\end{align*}
	Choose $\omega=36(d-1)^2 T$ and  $\rho=\frac{1}{\sqrt{T}}$. According to Eqs.~\eqref{eq:mu_t_interior_inequality} and \eqref{eq:theta_star_interior_inequality}, we have $t_1=t_2=\tilde{t}=1$. Therefore, for $d \leq 18$, the expected regret is lower bounded by $\Omega(\sqrt{T})$.
	
\end{proof}

\section{Proof for CMCB-SB}
\subsection{Proof of Theorem~\ref{thm:ub_semi_bandit}}

\begin{proof} (Theorem~\ref{thm:ub_semi_bandit}) 
	Denote $\Delta_t = f(\bw^*)-f(\bw_t)$, $\bar{f}_{t}(\bw) = \bw^\top \bs{\hat{\theta}}_{t-1} + E_t(\bw) - \rho \bw^\top \underline{\Sigma}_{t-1} \bw $ and $G_t$ the matrix whose $ij$-th entry is $g_{ij}(t)$.
	
	For $t \geq d^2+1$, suppose that event $\cG_t \cap \cH_t$ occurs. Define $h_t(\bw) \triangleq E_t(\bw)+ \rho {\bw}^\top G_t {\bw}$ for any $\bw \in \triangle^c_{d}$.
	Then, we have
	$$0 \leq \bar{f}_t(\bw)-f(\bw) \leq 2 h_t(\bw) .$$
	
	According to the selection strategy of $\bw_t$, we obtain 
	\begin{align*}
	f(\bw^*)-f(\bw_t) \leq & \bar{f}_t(\bw^*)-f(\bw_t)
	\\
	\leq & \bar{f}_t(\bw_t)-f(\bw_t)
	\\
	\leq & 2 h_t(\bw_t)
	\end{align*}
	Thus, for any $T \geq d^2+1$, we have that 
	\begin{align*}
	& \sum_{t=d^2 +1}^{T} \sbr{f(\bw^*)-f(\bw_t)} 
	\\
	\leq & 2 \sum_{t=d^2 +1}^{T} h_t(\bw_t)
	\\
	= & 2 \sum_{t=d^2 +1}^{T} \sbr{ \sqrt{2\beta(\delta_t)} \sqrt{\bw^\top D_{t-1}^{-1} (\lambda \Lambda_{\bar{\Sigma}_{t}} D_{t-1} + \sum_{s=1}^{t-1} \bar{\Sigma}_{s, \bw_s} ) D_{t-1}^{-1} \bw} + \rho {\bw}^\top G_t {\bw} }
	\\
	\leq & 2 \sqrt{2\beta(\delta_T)} \sum_{t=d^2 +1}^{T} \sqrt{\bw^\top D_{t-1}^{-1} (\lambda \Lambda_{\bar{\Sigma}_{t}} D_{t-1} + \sum_{s=1}^{t-1} \bar{\Sigma}_{s, \bw_s} ) D_{t-1}^{-1} \bw} + 2 \rho \sum_{t=d^2 +1}^{T} {\bw}^\top G_t {\bw} 
	\\
	\leq & 2 \sqrt{2\beta(\delta_T)} \sqrt{ T \cdot  \sum_{t=d^2 +1}^{T} \sbr{ \bw^\top D_{t-1}^{-1} (\lambda \Lambda_{\bar{\Sigma}_{t}} D_{t-1} + \sum_{s=1}^{t-1} \bar{\Sigma}_{s, \bw_s} ) D_{t-1}^{-1} \bw } } \\& + 2 \rho \sum_{t=d^2 +1}^{T} {\bw}^\top G_t {\bw} 
	\\
	\leq & 2 \sqrt{2\beta(\delta_T)} \sqrt{ T } \sqrt{ \lambda \underbrace{ \sum_{t=d^2 +1}^{T} \sbr{  \bw^\top D_{t-1}^{-1} \Lambda_{\bar{\Sigma}_{t}}  \bw } }_{\Gamma_1} + \underbrace{ \sum_{t=d^2 +1}^{T} \sbr{ \bw^\top D_{t-1}^{-1} \sum_{s=1}^{t-1} \bar{\Sigma}_{s, \bw_s}  D_{t-1}^{-1} \bw } }_{ \Gamma_2} } \\& + 2 \rho \underbrace{ \sum_{t=d^2 +1}^{T} {\bw}^\top G_t {\bw} }_{\Gamma_3}
	\end{align*}
	
	Let $(\Sigma^*_{ij})^+=\Sigma^*_{ij} \vee 0$ for any $i,j, \in [d]$.
	We first address $\Gamma_3$.
	Since for any $t \geq d^2+1$ and $i,j \in [d]$,
	\begin{align*}
	g_{ij}(t) = & 16 \left( \frac{3 \ln t}{N_{ij}(t-1)} \vee \sqrt{\frac{3 \ln t}{N_{ij}(t-1)}} \right) + \sqrt{\frac{48 \ln^2 t}{N_{ij}(t-1) N_{i}(t-1)}} + \sqrt{\frac{36 \ln^2 t}{N_{ij}(t-1) N_{j}(t-1)}} 
	\\
	\leq & 48  \frac{ \ln t}{\sqrt{N_{ij}(t-1)}}  + \sqrt{\frac{48 \ln^2 t}{N_{ij}(t-1) }} + \sqrt{\frac{36 \ln^2 t}{N_{ij}(t-1) }} 
	\\
	\leq & 61 \frac{ \ln t}{ \sqrt{N_{ij}(t-1)} } ,
	\end{align*}
	we can bound $\Gamma_3$ as follows 
	\begin{align*}
	\Gamma_3 = & \sum_{t=d^2 +1}^{T} {\bw}^\top G_t {\bw}
	\\
	= & \sum_{t=d^2 +1}^{T} \sum_{i,j \in [d]} g_{ij}(t) w_{t,i} w_{t,j}
	\\
	\leq & 61 \sum_{i,j \in [d]} \sum_{t=d^2 +1}^{T} \frac{ \ln t}{ \sqrt{N_{ij}(t-1)} }  w_{t,i} w_{t,j}
	\\
	\leq & 61 \ln T \sum_{i,j \in [d]} \sum_{t=d^2 +1}^{T} \frac{ w_{t,i} w_{t,j} }{ \sqrt{ \sum_{s=1}^{t-1} w_{s,i} w_{s,j}  } }  
	\\
	\leq & 61 \ln T \sum_{i,j \in [d]} \sum_{t=d^2 +1}^{T} \frac{ w_{t,i} w_{t,j} }{ \sqrt{ \sum_{s=1}^{t-1} w_{s,i} w_{s,j}  } } 
	\\
	\leq & 122 \ln T \sum_{i,j \in [d]} \sqrt{ \sum_{t=1}^{T} w_{t,i} w_{t,j} }
	\\
	\leq & 122 \ln T\sqrt{ d^2 \sum_{i,j \in [d]}  \sum_{t=1}^{T} w_{t,i} w_{t,j} }
	\\
	= & 122 \ln T\sqrt{ d^2   \sum_{t=1}^{T} \sum_{i,j \in [d]} w_{t,i} w_{t,j} }
	\\
	= & 122 \ln T\sqrt{ d^2 \sum_{t=1}^{T} \sbr{\sum_{i \in [d]} w_{t,i}}^2 }
	\\
	\leq & 122  d \ln T  \sqrt{ T }
	\end{align*}
	
	Next, we obtain a bound for   $\Gamma_1$.
	\begin{align*}
	\Gamma_1 = & \sum_{t=d^2 +1}^{T} \sbr{  \bw^\top D_{t-1}^{-1} \Lambda_{\bar{\Sigma}_{t}}  \bw }
	\\
	= & \sum_{t=d^2 +1}^{T} \sum_{i \in [d]} \frac{\bar{\Sigma}_{t,ii}}{N_{i}(t-1)} w_{t,i}^2
	\\
	\leq & \sum_{t=d^2 +1}^{T} \sum_{i \in [d]} \frac{\bar{\Sigma}_{t,ii}}{\sum_{s=1}^{t-1} w_{s,i}} w_{t,i}^2
	\\
	\leq &  \sum_{i \in [d]} \sum_{t=d^2 +1}^{T} \frac{\Sigma^*_{ii}+ 2g_{ii}(t)}{\sum_{s=1}^{t-1} w_{s,i}} w_{t,i}^2
	\\
	\leq &  \sum_{i \in [d]} \sbr{ \sum_{t=d^2 +1}^{T} \frac{ \Sigma^*_{ii} }{\sum_{s=1}^{t-1} w_{s,i}} w_{t,i} + 122 \sum_{t=d^2 +1}^{T} \frac{  \frac{ \ln t}{ \sqrt{N_{i}(t-1)} } }{\sum_{s=1}^{t-1} w_{s,i}} w_{t,i} }
	\\
	\leq &  \sum_{i \in [d]} \sbr{ \Sigma^*_{ii} \sum_{t=d^2 +1}^{T} \frac{ 1 }{\sum_{s=1}^{t-1} w_{s,i}} w_{t,i} + 122 \ln T \sum_{t=d^2 +1}^{T} \frac{ 1 }{(\sum_{s=1}^{t-1} w_{s,i})^{\frac{3}{2}} } w_{t,i} }
	\\
	\leq &  \sum_{i \in [d]} \sbr{ (\Sigma^*_{ii})^+ \ln \sbr{ \sum_{t=1}^{T} w_{t,i} } + 244 \ln T  \frac{ 1 }{ \sqrt{\sum_{t=1}^{d^2} w_{t,i}} }  }
	\\
	\leq &  \sum_{i \in [d]} \sbr{ (\Sigma^*_{ii})^+ \ln T  + 244 \ln T  \frac{ 1 }{ \sqrt{\sum_{t=1}^{d^2} w_{t,i}} }  }
	\end{align*}
	
	Finally, we bound $\Gamma_2$.
	\begin{align*}
	& \Gamma_2  
	\\
	= & \sum_{t=d^2 +1}^{T} \sbr{ \bw_t^\top D_{t-1}^{-1} \sbr{ \sum_{s=1}^{t-1} \bar{\Sigma}_{s, \bw_s} }  D_{t-1}^{-1} \bw_t }
	\\
	= &  \sum_{i,j \in [d]} \sum_{t=d^2 +1}^{T} \frac{ \sum_{s=1}^{t-1} \bar{\Sigma}_{s, ij} \I\{w_{s,i}, w_{s,j}>0\} }{ N_i(t-1) N_j(t-1) }  w_{t,i} w_{t,j}
	\\
	\leq &  \sum_{i,j \in [d]} \sum_{t=d^2 +1}^{T} \frac{ \sum_{s=1}^{t-1} \bar{\Sigma}_{s, ij} \I\{w_{s,i}, w_{s,j}>0\} }{ N_{ij}^2(t-1) }  w_{t,i} w_{t,j}
	\\
	\leq & \sum_{t=d^2 +1}^{T} \sum_{i,j \in [d]}  \frac{ \sum_{s=1}^{t-1} \sbr{ \Sigma^*_{ij}+ 2g_{ij}(s) } \I\{w_{s,i}, w_{s,j}>0\} }{ N_{ij}^2(t-1) }  w_{t,i} w_{t,j}
	\\
	= &  \sum_{i,j \in [d]} \sbr{ \sum_{t=d^2 +1}^{T}  \frac{ \Sigma^*_{ij} \sum_{s=1}^{t-1}  \I\{w_{s,i}, w_{s,j}>0\} }{ N_{ij}^2(t-1) }  w_{t,i} w_{t,j} + \sum_{t=d^2 +1}^{T}  \frac{ 2 \sum_{s=1}^{t-1}  g_{ij}(s) \I\{w_{s,i}, w_{s,j}>0\} }{ N_{ij}^2(t-1) }  w_{t,i} w_{t,j} }
	\\
	\leq &  \sum_{i,j \in [d]} \sbr{ \Sigma^*_{ij} \sum_{t=d^2 +1}^{T}  \frac{ 1  }{  N_{ij}(t-1)  }  w_{t,i} w_{t,j} + 122 \sum_{t=d^2 +1}^{T}  \frac{  \sum_{s=1}^{t-1}   \frac{ \ln s }{ \sqrt{N_{ij}(s-1)} } \I\{w_{s,i}, w_{s,j}>0\} }{ N_{ij}^2(t-1) }  w_{t,i} w_{t,j} }
	\\
	\leq &  \sum_{i,j \in [d]} \!\! \sbr{ \Sigma^*_{ij} \!\!\!\sum_{t=d^2 +1}^{T} \!   \frac{ 1  }{  \sum_{s=1}^{t-1}  w_{s,i} w_{s,j}  }  w_{t,i} w_{t,j} + 122 \!\! \sum_{t=d^2 +1}^{T} \!\!\!  \frac{  \ln t \sum_{s=1}^{t-1} \!\!   \frac{ 1 }{ \sqrt{ \sum_{\ell=1}^{s-1}  \I\{w_{\ell,i}, w_{\ell,j}>0\} } } \I\{w_{s,i}, w_{s,j}>0\} }{ N_{ij}^2(t-1) }  w_{t,i} w_{t,j} }
	\\
	\leq &  \sum_{i,j \in [d]} \sbr{ \Sigma^*_{ij} \sum_{t=d^2 +1}^{T}  \frac{ 1  }{  \sum_{s=1}^{t-1} w_{s,i} w_{s,j}  }  w_{t,i} w_{t,j} + 244 \sum_{t=d^2 +1}^{T}  \frac{  \ln t \sqrt{ \sum_{s=1}^{t-1} \I\{w_{s,i}, w_{s,j}>0\} } }{ N_{ij}^2(t-1) }  w_{t,i} w_{t,j} }
	\\
	= &  \sum_{i,j \in [d]} \sbr{ \Sigma^*_{ij} \sum_{t=d^2 +1}^{T}  \frac{ 1  }{  \sum_{s=1}^{t-1} w_{s,i} w_{s,j}  }  w_{t,i} w_{t,j} + 244 \ln T \sum_{t=d^2 +1}^{T}  \frac{  1  }{  N_{ij}^\frac{3}{2}(t-1)  }  w_{t,i} w_{t,j} }
	\\
	\leq &  \sum_{i,j \in [d]} \sbr{ \Sigma^*_{ij} \sum_{t=d^2 +1}^{T}  \frac{ 1  }{  \sum_{s=1}^{t-1} w_{s,i} w_{s,j}  }  w_{t,i} w_{t,j} + 244 \ln T \sum_{t=d^2 +1}^{T}  \frac{  1  }{ \sbr{ \sum_{s=1}^{t-1} w_{s,i} w_{s,j} }^\frac{3}{2} }  w_{t,i} w_{t,j} }
	\\
	\leq &  \sum_{i,j \in [d]} \sbr{ (\Sigma^*_{ij})^+ \ln \sbr{ \sum_{t=1}^{T}  w_{t,i} w_{t,j} } + 488 \ln T   \frac{  1  }{ \sqrt{ \sum_{t=1}^{d^2} w_{t,i} w_{t,j} } }  }
	\\
	\leq &  \sum_{i,j \in [d]} \sbr{ (\Sigma^*_{ij})^+ \ln T + 488 \ln T   \frac{  1  }{ \sqrt{ \sum_{t=1}^{d^2} w_{t,i} w_{t,j} } }  }
	\\
	\end{align*}
	Recall that $\beta(\delta_T)=\ln(T \ln^2 T) + d\ln \ln T + \frac{d}{2} \ln(1+e/ \lambda)=O(\ln T + d\ln\ln T+d \ln(1+ \lambda^{-1}) )$. 
	Combining the bounds of $\Gamma_1, \Gamma_2$ and $\Gamma_3$, we obtain
	\begin{align*}
	& \sum_{t=d^2 +1}^{T} \sbr{f(\bw^*)-f(\bw_t)} 
	\\
	\leq & 2 \sqrt{2\beta(\delta_T)} \sqrt{ T } \sqrt{ \lambda \underbrace{ \sum_{t=d^2 +1}^{T} \sbr{  \bw^\top D_{t-1}^{-1} \Lambda_{\bar{\Sigma}_{t}}  \bw } }_{\Gamma_1} + \underbrace{ \sum_{t=d^2 +1}^{T} \sbr{ \bw^\top D_{t-1}^{-1} \sum_{s=1}^{t-1} \bar{\Sigma}_{s, \bw_s}  D_{t-1}^{-1} \bw } }_{ \Gamma_2} }+ 2 \rho \underbrace{ \sum_{t=d^2 +1}^{T} {\bw}^\top G_t {\bw} }_{\Gamma_3}
	\\
	\leq & 2 \sqrt{2\beta(\delta_T)} \sqrt{ T } \sqrt{ \lambda \sum_{i \in [d]} \sbr{ (\Sigma^*_{ii})^+ \ln T  + 244 \ln T  \frac{ 1 }{ \sqrt{\sum_{t=1}^{d^2} w_{t,i}} }  } + \sum_{i,j \in [d]} \sbr{ (\Sigma^*_{ij})^+ \ln T + 488 \ln T   \frac{  1  }{ \sqrt{ \sum_{t=1}^{d^2} w_{t,i} w_{t,j} } }  } } \\& + 244 \rho  d \ln T  \sqrt{ T }	
	\\
	\leq & 2 \sqrt{2\beta(\delta_T)} \sqrt{ T } \sqrt{ \lambda \ln T \sum_{i \in [d]}  (\Sigma^*_{ii})^+     + \ln T \sum_{i,j \in [d]}  (\Sigma^*_{ij})^+  + (244\lambda+488) d^2 \ln T     }  + 244 \rho  d \ln T  \sqrt{ T }
	\\
	= &O \sbr{  \sqrt{ \sbr{ \ln T + d\ln\ln T+d \ln(1+ \lambda^{-1}) } \cdot T } \sqrt{ \lambda \ln T \sum_{i \in [d]}  (\Sigma^*_{ii})^+     + \ln T \sum_{i,j \in [d]}  (\Sigma^*_{ij})^+   +  (\lambda+1) d^2 \ln T     }  + \rho  d \ln T  \sqrt{ T } }
	\\
	= &O \sbr{  \sqrt{   d (\ln T+ \ln(1+ \lambda^{-1}))  \cdot T } \sqrt{ \lambda \ln T \sum_{i \in [d]}  (\Sigma^*_{ii})^+     + \ln T \sum_{i,j \in [d]}  (\Sigma^*_{ij})^+   +  (\lambda+1) d^2 \ln T     }  + \rho  d \ln T  \sqrt{ T } }
	\\
	= &O \sbr{  \sqrt{ d \ln T \sbr{ \ln T +  \ln(1+ \lambda^{-1}) } \cdot T } \sqrt{ \lambda \sum_{i \in [d]}  (\Sigma^*_{ii})^+     +  \sum_{i,j \in [d]}  (\Sigma^*_{ij})^+   +  (\lambda+1) d^2  }  + \rho  d \ln T  \sqrt{ T } }
	\end{align*}
	According to Lemmas~\ref{lemma:semi_concentration_covariance} and \ref{lemma:semi_concentration_means}, for any $t \geq 2$,  the probability of event $\neg(\cG_t \cap \cH_t)$ satisfies 
	\begin{align*}
	\Pr \mbr{ \neg(\cG_t \cap \cH_t) } \leq & \frac{10 d^2}{t^2} + \frac{1}{t \ln^2 t}
	\\
	\leq & \frac{10 d^2}{t \ln^2 t} + \frac{1}{t \ln^2 t}
	\\
	= &  \frac{11 d^2}{t \ln^2 t}
	\end{align*}
	
	Therefore, for any horizon $T$, we obtain the regret upper bound  
	\begin{align*}
	\ex[\cR(T)] = & O( \Delta_{\textup{max}} ) + \sum_{t=2}^{T} O \left( \Delta_{\textup{max}} \cdot \Pr \mbr{ \neg(\cG_t \cap \cH_t) } + \Delta_t \cdot \I \lbr{\cG_t \cap \cH_t}  \right) 
	\\
	= & O( \Delta_{\textup{max}} ) + \sum_{t=2}^{T} O \left(\Delta_{ \textup{max}} \cdot \frac{ d^2}{t \ln^2 t}  \right) +  O \Bigg(  \sqrt{ d \ln T \sbr{ \ln T +  \ln(1+ \lambda^{-1}) }  T } \cdot \\& \sqrt{ \lambda \sum_{i \in [d]}  (\Sigma^*_{ii})^+   +  \sum_{i,j \in [d]}  (\Sigma^*_{ij})^+   + (\lambda+1)  d^2      }   + \rho  d \ln T  \sqrt{ T } \Bigg)
	\\
	= &   O \sbr{  \sqrt{ d \ln T \sbr{ \ln T +  \ln(1+ \lambda^{-1}) }  T } \sqrt{ \lambda \sum_{i \in [d]}  (\Sigma^*_{ii})^+     +  \sum_{i,j \in [d]}  (\Sigma^*_{ij})^+   + (\lambda+1) d^2      }  + \rho  d \ln T  \sqrt{ T } + d^2 \Delta_{ \textup{max}} }
	\\
	= &   O \sbr{  \sqrt{ d \sbr{\ln(1+ \lambda^{-1}) + 1} \ln^2 T \cdot T } \sqrt{    (\lambda+1) \sbr{  \sum_{i,j \in [d]}  (\Sigma^*_{ij})^+ + d^2}      }  + \rho  d \ln T  \sqrt{ T } }
	\\
	= & O \Bigg(  \sqrt{ (\lambda+1) \sbr{\ln(1+ \lambda^{-1}) + 1}  (\|\Sigma^*\|_{+}   +  d^2) d \ln^2 T    \cdot T }   + \rho  d \ln T  \sqrt{ T } \Bigg)
	\\
	= & O \Bigg(  \sqrt{ L(\lambda) (\|\Sigma^*\|_{+}   +  d^2) d \ln^2 T    \cdot T }   + \rho  d \ln T  \sqrt{ T } \Bigg)
	\end{align*}
	where $L(\lambda)= (\lambda+1) \sbr{\ln(1+ \lambda^{-1}) + 1} $ and $\|\Sigma^*\|_{+}=\sum_{i,j \in [d]} \sbr{\Sigma^*_{ij} \vee 0}$ for any $i,j \in [d]$.
\end{proof}

\subsection{Proof of Theorem~\ref{thm:lb_semi_bandit}}

\begin{proof}	
First, we construct some instances with $d \geq 4$, $\frac{2}{d} \leq c \leq \frac{1}{2}$, $\Sigma_*=I$ and $\boldsymbol{\theta}_t \sim N(\boldsymbol{\theta}^*, I)$. 

Let $I_J$ be a random instance constructed as follows: we uniformly choose a dimension $J$ from $[d]$, and the expected reward vector $\boldsymbol{\theta}^*_{J}$ has $\frac{1}{2}+\varepsilon$ on its $J$-th entry and $\frac{1}{2}$ elsewhere, where $\varepsilon \in (0, \frac{1}{2}]$ will be specified later. Let $I_u$ be a uniform instance, where $\boldsymbol{\theta}^*_u$ has all its entries to be $\frac{1}{2}$.
Let $\Pr_J[\cdot]$ and $\Pr_u[\cdot]$ denote the probabilities under instances $I_J$ and $I_u$, respectively, and let $\Pr_j[\cdot]=\Pr_J[\cdot|J=j]$. Analogously, $E_J[\cdot]$, $E_u[\cdot]$ and $E_j[\cdot]=E_J[\cdot|J=j]$ denote the expectation operations.

Fix an algorithm $\mathcal{A}$. Let $S_t \in \{\mathbb{R} \cup \{\perp\}\}^d$ be a random variable vector denoting the observations at timestep $t$, obtained by running $\mathcal{A}$. Here $\perp$ denotes no observation on this dimension. Let $Q_{\perp}$ denote the distribution on support $\{\perp\}$ which takes value $\perp$ with probability 1.

In CMCB-SB, if $w_{t,i}>0$, we can observe the reward on the $i$-th dimension, i.e., $S_{t,i}=\theta_{t,i}$; otherwise, if $w_{t,i}=0$, we cannot get observation on the $i$-th dimension, i.e., $S_{t,i}=\perp$. 
Let $D_J$ be the distribution of observation sequence $S_1, \dots S_t$ under instance $I_J$, and $D_j=D_{J|J=j}$ is the distribution conditioned on $J=j$. 
Let $D_u$ be the distribution of observation sequence $S_1, \dots S_t$ under instance $I_u$.  
For any $i \in [d]$, let $N_i=\sum_{t=1}^{T}\mathbb{I}\{w_{t,i}>0\}$ be the number of pulls that has a positive weight on the $i$-th dimension, i.e., the number of observations on the $i$-th dimension.

Following the analysis procedure of Lemma A.1 in \cite{auer2002nonstochastic}, we have

\begin{align*}
KL(D_j\|D_u) = & \sum_{t=1}^{T} KL(D_u[S_t|S_1,\dots,S_{t-1}]\|D_j[S_t|S_1,\dots,S_{t-1}])
\\
= & \sum_{t=1}^{T} \sum_{i=1}^{d} \left(\Pr[w_{t,i}>0] \cdot KL\left(N(\theta^*_{u,i},1) \| N(\theta^*_{j,i},1)\right)
+\Pr[w_{t,i}=0] \cdot KL( Q_{\perp} \| Q_{\perp}) \right)
\\
= & \sum_{t=1}^{T} \left( \Pr[w_{t,j}>0] \cdot KL\left(N(\frac{1}{2},1) \| N(\frac{1}{2}+\varepsilon,1)\right)
+\sum_{i \neq j}^{d} \Pr[w_{t,i}>0] \cdot KL\left(N(\frac{1}{2},1) \| N(\frac{1}{2},1)\right) \right)
\\
= & \frac{1}{2} \varepsilon^2 \cdot \sum_{t=1}^{T} \Pr[w_{t,j}>0] 
\\
=  & \frac{1}{2} \varepsilon^2 E_u[N_j] 
\end{align*}

Here the first equality comes from the chain rule of entropy~\cite{cover1999elements}. The second equality is due to that given $S_1,\dots,S_{t-1}$, if $w_{t,i}>0$, the conditional distribution of $S_t$ is $N(\theta^*_{\cdot,i},1)$, where "$\cdot$" refers to the subscript of instances; otherwise, if $w_{t,i}=0$, $S_t$ is  $\perp$ deterministically. The third equality is due to that $\boldsymbol{\theta}^*_u$ and $\boldsymbol{\theta}^*_j$ only have one different entry on the $j$-th dimension.

Let $\|\cdot\|$ with subscript $TV$ denote the total variance distance, and $KL(\cdot\|\cdot)$ denote the Kullback–Leibler divergence. 
Using Eq. (28) in the analysis of Lemma A.1 in \cite{auer2002nonstochastic} and Pinsker's inequality, we have

\begin{align*}
E_j[N_j] \leq & E_u[N_j] + T \|D_j-D_u\|_{TV}
\\
\leq & E_u[N_j] + T \sqrt{ \frac{1}{2} KL(D_j\|D_u) }
\\
= & E_u[N_j] +\frac{T  \varepsilon}{2} \sqrt{  E_u[N_j] }
\end{align*}

Let $m = \left \lfloor \frac{1}{c} \right \rfloor \leq \frac{d}{2}$ denote the maximum number of positive entries for a feasible action, i.e., the maximum number of observations for a pull.
Performing the above argument for all $j \in [d]$ and using $\sum_{j \in [d]} E_u[N_j] \leq mT$, we have

\begin{align*}
\sum_{j \in [d]} E_j[N_j] \leq & \sum_{j \in [d]} E_u[N_j] +\frac{T \varepsilon}{2} \sum_{j \in [d]} \sqrt{ E_u[N_j] }
\\
\leq & mT+\frac{T \varepsilon}{2}  \sqrt{d \sum_{j \in [d]} E_u[N_j]}
\\
\leq & mT+\frac{T \varepsilon}{2}  \sqrt{d mT}
\end{align*}

and thus
\begin{align*}
E_J[N_J] = \frac{1}{d} \sum_{j \in [d]} E_j[N_j]
\leq \frac{mT}{d} + \frac{T \varepsilon}{2} \sqrt{\frac{mT}{d}}
\end{align*}

Letting $\rho \leq \frac{\varepsilon}{2(1-c)}$, the expected reward (linear) term dominates $f(\boldsymbol{w})$, and the best action $\boldsymbol{w}^*$ under $I_J$ has the weight $1$ on the $J$-th entry and $0$ elsewhere. 

Recall that $m \leq \frac{1}{c}$. For each pull that has no weight on the $J$-th entry,  algorithm $\mathcal{A}$ must suffer a regret at least

\begin{align*}
	& (\frac{1}{2}+\varepsilon - \rho)-(\frac{1}{2}-\rho \cdot\frac{1}{m})
	\\
	\geq & \varepsilon-\frac{m-1}{m}\rho
	\\
	\geq & \varepsilon-\frac{m-1}{m} \cdot \frac{\varepsilon}{2(1-c)}
	\\
	\geq & \varepsilon-\frac{m-1}{m} \cdot \frac{\varepsilon}{2(1-\frac{1}{m})}
	\\
	= & \frac{\varepsilon}{2}
\end{align*}

Thus, the regret is lower bounded by

\begin{align*}
E[R(T)] \geq & ( T - E_J[N_J]) \cdot \frac{\varepsilon}{2}
\\
\geq & \sbr{T - \frac{mT}{d} - \frac{T \varepsilon}{2} \sqrt{\frac{mT}{d}} } \cdot \frac{\varepsilon}{2}
\\
= & \Omega \sbr{ T \varepsilon - T \varepsilon^2 \sqrt{\frac{mT}{d}}  } ,
\end{align*}
where the last equality is due to $m \leq \frac{d}{2}$.

Letting $\varepsilon=a_0 \sqrt{ \frac{d}{Tm} }$ for small enough constant $a_0$, we obtain the regret lower bound $\Omega(\sqrt{\frac{dT}{m}})=\Omega(\sqrt{cdT})$. 

\end{proof}

\section{Proof for CMCB-FB}

\subsection{Proof of Theorem~\ref{thm:ub_full_bandit}}

In order to prove Theorem~\ref{thm:ub_full_bandit}, we first prove  Lemmas~\ref{lemma:full_bandit_con_covariance} and \ref{lemma:full_bandit_con_means}, which give the concentrations of covariance and means for CMCB-FB, using different  techniques than  those for CMCB-SB (Lemmas~\ref{lemma:semi_concentration_covariance} and \ref{lemma:semi_concentration_means}) and CMCB-FI (Lemmas~\ref{lemma:full_info_con_covariance} and \ref{lemma:full_info_con_means}).

\begin{lemma}[Concentration of Covariance for CMCB-FB]
	\label{lemma:full_bandit_con_covariance}
	Consider the CMCB-FB problem and algorithm $\algfullbandit$ (Algorithm~\ref{alg:full_bandit}). 
	For any $t > 0$, the event 
	$$
	\cM_t \triangleq \lbr{ |\Sigma^*_{ij}-\hat{\Sigma}_{ij,t-1}| \leq 5 \| C^+_\pi \|  \sqrt{\frac{3 \ln t}{2 N_\pi(t)}} }
	$$
	satisfies
	$$
	Pr[\cM_t] \geq 1-\frac{6 d^2}{ t^2} ,
	$$
	where $\| C_\pi^+ \| \triangleq \max_{i \in [\tilde{d}]} \lbr{ \sum_{j \in [\tilde{d}]} \abr{C^+_{\pi, ij}} }$.
\end{lemma}
\begin{proof}
	Let $\bs{\sigma}=(\Sigma^*_{11}, \dots, \Sigma^*_{dd}, \Sigma^*_{12}, \dots, \Sigma^*_{1d}, \Sigma^*_{23}, \dots, \Sigma^*_{d,d-1})^\top \in \R^{\tilde{d}}$ denote the column vector that stacks the $\tilde{d}$ distinct entries in the covariance matrix $\Sigma^*$.
	
	Recall that the $\tilde{d} \times \tilde{d}$ matrix
	$$
	C_\pi=\begin{bmatrix}
	w_{1,1}^2 & \dots & w_{1,d}^2 & 2 w_{1,1} w_{1,2} & 2 w_{1,1} w_{1,3} & \dots & 2 w_{1,d-1} w_{1,d} \\ 
	w_{2,1}^2 & \dots & w_{2,d}^2 & 2 w_{2,1} w_{2,2} & 2 w_{2,1} w_{2,3} & \dots & 2 w_{2,d-1} w_{2,d} \\
	& \dots &  &  &  & \dots & \\ 
	&  &  &  &  &  & \\ 
	&  &  &  &  &  & \\ 
	& \dots &  &  &  & \dots & \\ 
	w_{\tilde{d},1}^2 & \dots & w_{\tilde{d},d}^2 & 2 w_{\tilde{d},1} w_{\tilde{d},2} & 2 w_{\tilde{d},1} w_{\tilde{d},3} & \dots & 2 w_{\tilde{d},d-1} w_{\tilde{d},d} \\
	\end{bmatrix} ,
	$$
	where $w_{i,j}$ denotes the $j$-th entry of portfolio vector $\bw_i$ in design set $W$. 
	We use $C_{\pi,k}$ to denote the $k$-th row in matrix $C_\pi$. 
	
	We recall the feedback structure in algorithm $\algfullbandit$ as follows.
	At each timestep $t$, the learner plays an action $\bw_t \in \triangle_{d}$, and observes the full-bandit feedback $y_t=\bw_t^\top \bs{\theta_t}$ with  $\ex[y_t]=\bw_t^\top \bs{\theta^*} $ and $\var[y_t]=\bw_t^\top \Sigma^* \bw_t=\sum_{i \in [d]} w_{t,i}^2 \Sigma^*_{ii} + \sum_{i,j \in [d], i<j} 2 w_{t,i} w_{t,j} \Sigma^*_{ij}$. 
	Then, during exploration round $s$, where each action in design set $\pi=\{ \bv_1, \dots, \bv_{\tilde{d}} \}$ is pulled once, the full-bandit feedback $\by_s$ has  mean $ \by(\pi) \triangleq (\bv_1^\top \bs{\theta^*}, \dots, \bv_{\tilde{d}}^\top \bs{\theta^*})^\top$ and variance $\bz(\pi) \triangleq (\bv_1^\top \Sigma^* \bv_1, \dots, \bv_{\tilde{d}}^\top \Sigma^* \bv_{\tilde{d}} )^\top =( C_{\pi,1}^\top \bs{\sigma}, \dots, C_{\pi,\tilde{d}}^\top \bs{\sigma} )^\top=C_\pi \bs{\sigma}$. 
	For any $t>0$, denote $\hat{\by}_{t}$  the empirical mean of $ \by(\pi) \triangleq (\bv_1^\top \bs{\theta^*}, \dots, \bv_{\tilde{d}}^\top \bs{\theta^*})^\top$ and $\hat{\bz}_t$ the empirical variance of $\bz(\pi) \triangleq ( C_{\pi,1}^\top \bs{\sigma}, \dots, C_{\pi,\tilde{d}}^\top \bs{\sigma} )^\top$. 
	
	Using the Chernoff-Hoeffding inequality for empirical variances (Lemma~1 in \cite{sani2012risk}), we have that for any $t>0$ and $k \in \tilde{d}$, with probability at least $1-\frac{6 d^2}{ t^2}$,
	$$
	\abr{\hat{z}_{t-1,k} - z_k}  =  \abr{\hat{z}_{t-1,k} - C_{\pi,k}^\top \bs{\sigma}}  \leq 5 \sqrt{\frac{3 \ln t}{2 N_\pi(t-1)}}.
	$$
	Let $C^+_{\pi,i}$ denote the $i$-th row of matrix $C^+_\pi$ and $C^+_{\pi,ik}$ denote the $ik$-th entry of matrix $C^+_\pi$. Since $C^+_\pi C_\pi=I$, for any $i \in [\tilde{d}]$, we have
	\begin{align*}
	\abr{ \hat{\sigma}_{t-1,i} - \sigma_i }  = & \abr{ (C^+_{\pi,i})^\top \hat{\bz}_{t} - (C^+_{\pi,i})^\top C_\pi \bs{\sigma} } 
	\\
	\leq & \sum_{k \in [\tilde{d}]} \abr{C^+_{\pi,ik}}  \abr{\hat{z}_{t-1,k} -  C_{\pi,k}^\top \bs{\sigma}} 
	\\
	\leq & 5 \sum_{k \in [\tilde{d}]} \abr{C^+_{\pi,ik} } \sqrt{\frac{3 \ln t}{2 N_\pi(t)}}
	\\
	\\
	\leq & 5 \| C^+_\pi \|  \sqrt{\frac{3 \ln t}{2 N_\pi(t)}}
	\end{align*}
	
	In addition, since $\bs{\sigma}=(\Sigma^*_{11}, \dots, \Sigma^*_{dd}, \Sigma^*_{12}, \dots, \Sigma^*_{1d}, \Sigma^*_{23}, \dots, \Sigma^*_{d,d-1})^\top \in \R^{\tilde{d}}$ is the column vector stacking the distinct entries in $\Sigma^*$, we obtain the lemma.
\end{proof}

\begin{lemma}[Concentration of Means for CMCB-FB]
	\label{lemma:full_bandit_con_means}
	Consider the CMCB-FB problem and algorithm $\algfullbandit$ (Algorithm~\ref{alg:full_bandit}).
	Let $\delta_t>0$, $\lambda>0$, $\beta(\delta_t)=\ln(1/\delta_t) + \ln \ln t + \frac{\tilde{d}}{2} \ln(1+e/\lambda)$ and $E_t(\bw) = \sqrt{2\beta(\delta_t)} \sqrt{\bw^\top B_\pi^+ D_{t-1}^{-1} (\lambda \Lambda_{\Sigma^*_\pi} D_{t-1} + \sum_{s=1}^{N_\pi(t-1)} \Sigma^*_\pi ) D_{t-1}^{-1} (B_\pi^{+}) \top \bw}$, where $\Sigma^*_\pi=\textup{diag}(\bv_1^\top \Sigma^* \bv_1, \dots, \bv_{\tilde{d}}^\top \Sigma^* \bv_{\tilde{d}} )$. 
	Then, the event $\cN_t \triangleq \{ |\bw^\top \bs{\theta}^*-\bw^\top \bs{\hat{\theta}}_{t-1}| \leq E_t(\bw), \forall \bw \in \cD \}$ satisfies $\Pr[\cN_t] \geq 1 - \delta_t$.
\end{lemma}

\begin{proof}
	Recall that the $\tilde{d} \times d$ matrix 	$B_\pi = [ \bv_1^\top; \dots; \bv_{\tilde{d}}^\top ]$.
	and $B_\pi^+$ is the Moore–Penrose pseudoinverse of $B_\pi$. Since $B_\pi$ is of full column rank, $B_\pi^+$ satisfies $B_\pi^+ B_\pi=I$.
	
	We recall the feedback structure in algorithm $\algfullbandit$ as follows.
	At each timestep $t$, the learner plays an action $\bw_t \in \triangle_{d}$, and observes the full-bandit feedback $y_t=\bw_t^\top \bs{\theta_t}$ such that $\ex[y_t]=\bw_t^\top \bs{\theta^*} $ and $\var[y_t]=\bw_t^\top \Sigma^* \bw_t=\sum_{i \in [d]} w_{t,i}^2 \Sigma^*_{ii} + \sum_{i,j \in [d], i<j} 2 w_{t,i} w_{t,j} \Sigma^*_{ij}$.
	Then, during exploration round $s$, where each action in design set $\pi=\{ \bv_1, \dots, \bv_{\tilde{d}} \}$ is pulled once, the full-bandit feedback $\by_s$ has  mean $ \by(\pi) \triangleq (\bv_1^\top \bs{\theta^*}, \dots, \bv_{\tilde{d}}^\top \bs{\theta^*})^\top$ and variance $\bz(\pi) \triangleq (\bv_1^\top \Sigma^* \bv_1, \dots, \bv_{\tilde{d}}^\top \Sigma^* \bv_{\tilde{d}} )^\top =( C_{\pi,1}^\top \bs{\sigma}, \dots, C_{\pi,\tilde{d}}^\top \bs{\sigma} )^\top=C_\pi \bs{\sigma}$.
	For any $t>0$, $\hat{\by}_{t}$ is the empirical mean of $ \by(\pi) \triangleq (\bv_1^\top \bs{\theta^*}, \dots, \bv_{\tilde{d}}^\top \bs{\theta^*})^\top$ and $\hat{\bz}_t$ is the empirical variance of $\bz(\pi) \triangleq ( C_{\pi,1}^\top \bs{\sigma}, \dots, C_{\pi,\tilde{d}}^\top \bs{\sigma} )^\top$. 
	
	Let $D_t$ be the $\tilde{d} \times \tilde{d}$ diagonal matrix such that $D_{t,ii}=N_\pi(t)$ for any $i \in [\tilde{d}]$ and $t > 0$. 
	Let $\Sigma^*_\pi$ be a $\tilde{d} \times \tilde{d}$ diagonal matrix such that $\Sigma^*_{\pi,ii}=\bw_i^\top \Sigma^* \bw_i$ for any $i \in [\tilde{d}]$, and thus $\Lambda_{\Sigma^*_\pi}=\Sigma^*_\pi$. 
	Let $\bs{\varepsilon}_t$ be the vector such that $\bs{\eta}_t=(\Sigma^*)^{\frac{1}{2}} \bs{\varepsilon}_t$ for any timestep $t>0$.
	Let $\zeta_{s}=(\bv_1\top (\Sigma^*)^{\frac{1}{2}} \bs{\varepsilon}_{s,1}, \dots,  \bv_{\tilde{d}}\top (\Sigma^*)^{\frac{1}{2}} \bs{\varepsilon}_{s,\tilde{d}} )^\top$, where $(\Sigma^*)^{\frac{1}{2}} \bs{\varepsilon}_{s,k}$ denotes the noise of the $k$-th sample in the $s$-th exploration round.
	
	Note that, the following analysis of $| \bw^\top ( \bs{\theta}^* - \bs{\hat{\theta}}_{t-1} ) |$ and the constructions (definitions) of the noise $\zeta_{s}$, matrices $S_t,V_t$ and super-martingale $M^{\bu}_t$ are different from those in CMCB-SB (Lemmas~\ref{lemma:semi_concentration_covariance} and \ref{lemma:semi_concentration_means}) and CMCB-FI (Lemmas~\ref{lemma:full_info_con_covariance} and \ref{lemma:full_info_con_means}).
	
	For any $\bw \in \triangle_{d}$, we have
	\begin{align*}
	\abr{ \bw^\top \sbr{ \bs{\theta}^* - \bs{\hat{\theta}}_{t-1} } } = & \abr{ \bw^\top \sbr{ B_\pi^+ B_\pi \bs{\theta}^* - B_\pi^{+} \hat{\by}_{t-1} } }
	\\
	= & \abr{ \bw^\top  B_\pi^+ \sbr{ \by(\pi) - \hat{\by}_{t-1} } }
	\\
	= & \abr{  - \bw^\top B_\pi^+ D_{t-1}^{-1} \sum_{s=1}^{N_\pi(t-1)} \zeta_{s} }
	\\
	= & \abr{- \bw^\top B_\pi^+ D_{t-1}^{-1}  \sbr{D + \sum_{s=1}^{N_\pi(t-1)} \Sigma^*_\pi }^{\frac{1}{2}} \sbr{D + \sum_{s=1}^{N_\pi(t-1)} \Sigma^*_\pi }^{-\frac{1}{2}}  \sum_{s=1}^{N_\pi(t-1)} \zeta_{s} }
	\\
	\leq & \sqrt{\bw^\top B_\pi^+ D_{t-1}^{-1}  \sbr{D + \sum_{s=1}^{N_\pi(t-1)} \Sigma^*_\pi } D_{t-1}^{-1} (B_\pi^+)^\top \bw } \cdot 
	\left \| \sum_{s=1}^{N_\pi(t-1)} \zeta_{s} \right \|_{\sbr{D + \sum_{s=1}^{N_\pi(t-1)} \Sigma^*_\pi }^{-1} }
	\end{align*}
	
	Let $ S_t=\sum_{s=1}^{N_\pi(t-1)} \zeta_{s} $, $ V_t=\sum_{s=1}^{N_\pi(t-1)} \Sigma^*_\pi $ and $I_{D+V_t}=\frac{1}{2}\| S_t \|^2_{\sbr{D+V_t}^{-1}}$. 
	Then, we have
	$$
	\left \| \sum_{s=1}^{N_\pi(t-1)} \zeta_{s} \right \|_{\sbr{D + \sum_{s=1}^{N_\pi(t-1)} \Sigma^*_\pi }^{-1} } = \| S_t \|_{\sbr{D+V_t}^{-1} } = \sqrt{2 I_{D+V_t}} .
	$$
	
	Since $ D \preceq \lambda \Lambda_{\Sigma^*_\pi} D_{t-1} $, we get 
	\begin{align*}
	\abr{ \bw^\top \sbr{ \bs{\theta}^* - \bs{\hat{\theta}}_{t-1} }} 
	\leq & \sqrt{\bw^\top B_\pi^+ D_{t-1}^{-1} D D_{t-1}^{-1} (B_\pi^+)^\top \bw + \bw^\top B_\pi^+ D_{t-1}^{-1} \sbr{\sum_{s=1}^{N_\pi(t-1)} \Sigma^*_\pi } D_{t-1}^{-1} (B_\pi^+)^\top \bw } \cdot \sqrt{2 I_{D+V_t}}
	\\
	\leq & \sqrt{\lambda \bw^\top B_\pi^+ D_{t-1}^{-1}  \Lambda_{\Sigma^*}  (B_\pi^+)^\top \bw + \bw^\top B_\pi^+ D_{t-1}^{-1} \sbr{\sum_{s=1}^{N_\pi(t-1)} \Sigma^*_\pi } D_{t-1}^{-1} (B_\pi^+)^\top \bw } \cdot \sqrt{2 I_{D+V_t}}
	\\
	= & \sqrt{ \bw^\top B_\pi^+ D_{t-1}^{-1} \sbr{\lambda \Lambda_{\Sigma^*}  D_{t-1} + \sum_{s=1}^{N_\pi(t-1)} \Sigma^*_\pi } D_{t-1}^{-1} (B_\pi^+)^\top \bw } \cdot \sqrt{2 I_{D+V_t}}
	\end{align*}
	Thus, 
	\begin{align*}
	& \Pr \mbr{ \abr{\bw^\top \sbr{ \bs{\theta}^* - \bs{\hat{\theta}}_{t-1} }}  > \sqrt{2\beta(\delta_t)} \sqrt{\bw^\top B_\pi^+ D_{t-1}^{-1} (\lambda \Lambda_{\Sigma^*_\pi} D_{t-1} + \sum_{s=1}^{N_\pi(t-1)} \Sigma^*_\pi ) D_{t-1}^{-1} (B_\pi^{+}) \top \bw} }
	\\
	\leq & \Pr \Bigg[ \sqrt{ \bw^\top B_\pi^+ D_{t-1}^{-1} \sbr{\lambda \Lambda_{\Sigma^*}  D_{t-1} + \sum_{s=1}^{N_\pi(t-1)} \Sigma^*_\pi } D_{t-1}^{-1} (B_\pi^+)^\top \bw } \cdot \sqrt{2 I_{D+V_t}} \\& \hspace*{8em} > \sqrt{2\beta(\delta_t)} \sqrt{\bw^\top B_\pi^+ D_{t-1}^{-1} \sbr{\lambda \Lambda_{\Sigma^*_\pi} D_{t-1} + \sum_{s=1}^{N_\pi(t-1)} \Sigma^*_\pi}  D_{t-1}^{-1} (B_\pi^{+}) \top \bw}  \Bigg]
	\\
	= & \Pr \mbr{  I_{D+V_t} > \beta(\delta_t) }
	\end{align*}
	
	Hence, to prove Eq.~\eqref{eq:semi_con_mean_first_prove}, it suffices to prove
	\begin{align} \label{eq:I_D+V_t_>_beta}
	\Pr \mbr{  I_{D+V_t} > \beta(\delta_t) } \leq \delta_t .
	\end{align}
	
	To prove Eq.~\eqref{eq:I_D+V_t_>_beta}, we introduce some notions.
	Let $\bu \in \R^d$ be a multivariate Gaussian random variable with mean $\textbf{0}$ and covariance $D^{-1}$, which is independent of all the other random variables and we use $\varphi(\bu)$ denote its probability density function. Let 
	$$ P^{\bu}_s = \exp \sbr{ \bu^\top  \zeta_s - \frac{1}{2} \bu^\top  \Sigma^*_\pi \bu } ,
	$$
	$$
	M^{\bu}_t \triangleq \exp \sbr{ \bu^\top S_t - \frac{1}{2} \| \bu \|_{V_t}^2 } ,
	$$
	and 
	$$
	M_t \triangleq \ex_{\bu}[ M^{\bu}_t ] = \int_{\R^d} \exp \sbr{ \bu^\top S_t - \frac{1}{2} \| \bu \|_{V_t}^2 } \varphi(\bu) du,
	$$
	where $s=1, \dots, N_\pi(t-1)$ is the index of exploration round.
	We have $M^{\bu}_t = \Pi_{s=1}^{N_\pi(t-1)} P^{\bu}_s $.
	In the following, we prove $\ex[M_t] \leq 1$.
	
	For any timestep $t$, $\eta_t=(\Sigma^*)^{\frac{1}{2}} \bs{\varepsilon}_t$ is $\Sigma^*$-sub-Gaussian and $\eta_t$ is independent among different timestep $t$. Then, $\zeta_{s}=(\bv_1\top (\Sigma^*)^{\frac{1}{2}} \bs{\varepsilon}_{s,1}, \dots,  \bv_{\tilde{d}}\top (\Sigma^*)^{\frac{1}{2}} \bs{\varepsilon}_{s,\tilde{d}} )^\top$ is $\Sigma^*_\pi$-sub-Gaussian.
	According to the sub-Gaussian property, $\zeta_{s}$ satisfies
	$$
	\forall \bv \in \R^d, \  \ex \mbr{e^{\bv^\top \zeta_{s} } } \leq e^{\frac{1}{2} \bv^\top \Sigma^*_\pi \bv},
	$$
	which is equivalent to 
	$$
	\forall \bv \in \R^d, \  \ex \mbr{e^{\bv^\top \zeta_{s} - \frac{1}{2} \bv^\top \Sigma^*_\pi \bv} } \leq 1 .
	$$
	Let $\cJ_s$ be the $\sigma$-algebra $\sigma(W, \zeta_1, \dots, W, \zeta_{s-1}, \pi)$.
	Thus, we have
	$$ \ex \mbr{  P^{\bu}_{s} | \cJ_{s} } =  \ex \mbr{ \exp \sbr{ \bu^\top \zeta_{s} - \frac{1}{2} \bu^\top  \Sigma^*_\pi \bu } | \cJ_{s} }  \leq 1 .
	$$

	Then, we can obtain
	\begin{align*}
	\ex[M^{\bu}_t| \cJ_{N_\pi(t-1)}] = & \ex \mbr{ \Pi_{s=1}^{N_\pi(t-1)} P^{\bu}_s | \cJ_{N_\pi(t-1)} }
	\\
	= & \sbr{\Pi_{s=1}^{N_\pi(t-2)} P^{\bu}_s} \ex \mbr{  P^{\bu}_{N_\pi(t-1)} | \cJ_{N_\pi(t-1)} }
	\\
	\leq & M^{\bu}_{t-1},
	\end{align*}
	which implies that $M^{\bu}_t$ is a super-martingale and $\ex[M^{\bu}_t | \bu] \leq 1$. 
	Thus, 
	$$
	\ex[M_t]=\ex_{\bu} [\ex[M^{\bu}_t | \bu] ] \leq 1 .
	$$
	
	According to Lemma 9 in \cite{improved_linear_bandit2011},
	we have 
	\begin{align*}
	M_t \triangleq \int_{\R^d} \exp \sbr{ \bu^\top S_t - \frac{1}{2} \| \bu \|_{V_t}^2 } \varphi(\bu) du = \sqrt{\frac{\det D}{\det (D+V_t)}} \exp \sbr{ I_{D+V_t} } .
	\end{align*}
	Thus,
	\begin{align*}
	\ex \mbr{ \sqrt{\frac{\det D}{\det (D+V_t)}} \exp \sbr{ I_{D+V_t} } } \leq 1.
	\end{align*}
	
	Now we prove Eq.~\eqref{eq:I_D+V_t_>_beta}. First, we have
	\begin{align}
	\Pr \mbr{  I_{D+V_t} > \beta(\delta_t) }  = & \Pr \mbr{  \sqrt{\frac{\det D}{\det (D+V_t)}} \exp \sbr{ I_{D+V_t} } > \sqrt{\frac{\det D}{\det (D+V_t)}} \exp \sbr{ \beta(\delta_t) }  }
	\nonumber\\
	= & \Pr \mbr{  M_t > \frac{1}{\sqrt{\det ( I+D^{-\frac{1}{2}} V_t D^{-\frac{1}{2}} )}} \exp \sbr{ \beta(\delta_t) }  }
	\nonumber\\
	\leq & \frac{ \ex[M_t] \sqrt{\det ( I+D^{-\frac{1}{2}} V_t D^{-\frac{1}{2}} )} }{ \exp \sbr{ \beta(\delta_t) } } 
	\nonumber\\
	\leq & \frac{ \sqrt{\det ( I+D^{-\frac{1}{2}} V_t D^{-\frac{1}{2}} )} }{ \exp \sbr{ \beta(\delta_t) } } \label{eq:det_exp_beta}
	\end{align}

	Then, for some constant $\gamma>0$ and for any $a \in \mathbb{N}$, we define the set of timesteps $\cK_{a} \subseteq [T]$ such that 
	$$
	t \in \cK_{a} \Leftrightarrow \    (1+\gamma)^{a} \leq N_\pi(t-1) < (1+\gamma)^{a+1} .
	$$
	Define $D_{a}$ as a diagonal matrix such that $D_{a,ii}=(1+\gamma)^{a}, \   \forall i \in \tilde{d}$. 
	Suppose $t \in \cK_{a}$ for some fixed $a$. Then, we have
	$$
	\frac{1}{1+\gamma} D_t \preceq D_{a} \preceq D_t .
	$$
	Let $D= \lambda \Lambda_{\Sigma^*_\pi} D_{a} \succeq  \frac{\lambda}{1+\gamma} \Lambda_{\Sigma^*_\pi}  D_t$. Then, we have
	$$
	D^{-\frac{1}{2}} V_t D^{-\frac{1}{2}} \preceq \frac{1+\gamma}{\lambda} D_t^{-\frac{1}{2}} \Lambda_{\Sigma^*_\pi}^{-\frac{1}{2}}   V_t \Lambda_{\Sigma^*_\pi}^{-\frac{1}{2}}  D_t^{-\frac{1}{2}} , 
	$$
	where matrix $D_t^{-\frac{1}{2}} \Lambda_{\Sigma^*_\pi}^{-\frac{1}{2}}   V_t \Lambda_{\Sigma^*_\pi}^{-\frac{1}{2}}  D_t^{-\frac{1}{2}}$ has $\tilde{d}$ ones on the diagonal. Since the determinant of a positive definite matrix is smaller than the product of its diagonal terms, we have
	\begin{align}
	\det ( I+D^{-\frac{1}{2}} V_t D^{-\frac{1}{2}} ) \leq & \det ( I+ \frac{1+\gamma}{\lambda} D_t^{-\frac{1}{2}} \Lambda_{\Sigma^*_\pi}^{-\frac{1}{2}}   V_t \Lambda_{\Sigma^*_\pi}^{-\frac{1}{2}}  D_t^{-\frac{1}{2}} )
	\nonumber\\
	\leq & \sbr{1+ \frac{1+\gamma}{\lambda}}^{\tilde{d}} \label{eq:det_1+gamma}
	\end{align}
	
	Let $\gamma=e-1$.
	Using Eqs.~\eqref{eq:det_exp_beta} and \eqref{eq:det_1+gamma}, $\beta(\delta_t)=\ln(1/\delta_t) + \ln \ln t + \frac{\tilde{d}}{2} \ln(1+e/\lambda) = \ln(t \ln^2 t) + \ln \ln t + \frac{\tilde{d}}{2} \ln(1+\frac{e}{\lambda})$  and a union bound over $a$,  we have
	\begin{align*}
	\Pr \mbr{  I_{D+V_t} > \beta(\delta_t) }  \leq & \sum_{a } \Pr \mbr{  I_{D+V_t} > \beta(\delta_t) | t \in \cK_{a}, D= \lambda \Lambda_{\Sigma^*_\pi} D_{a} }  
	\\
	\leq &  \sum_{a } \frac{ \sqrt{\det ( I+D^{-\frac{1}{2}} V_t D^{-\frac{1}{2}} )} }{ \exp \sbr{ \beta(\delta_t) } }
	\\
	\leq & \frac{ \ln t }{ \ln(1+\gamma) } \cdot \frac{  \sbr{1+ \frac{1+\gamma}{\lambda}}^\frac{ \tilde{d}}{2} }{ \exp \sbr{ \ln(t \ln^2 t) + \ln \ln t + \frac{d}{2} \ln(1+\frac{e}{\lambda}) } }
	\\
	= & \ln t \cdot \frac{  \sbr{1+\frac{e}{\lambda}}^\frac{\tilde{d}}{2} }{ t \ln^2 t  \cdot \ln t \cdot \sbr{1+\frac{e}{\lambda}}^\frac{\tilde{d}}{2} }
	\\
	= &  \frac{  1 }{ t \ln^2 t }
	\\
	= &  \delta_t 
	\end{align*}
	Thus, Eq.~\eqref{eq:I_D+V_t_>_beta} holds and we complete the proof of Lemma~\ref{lemma:semi_concentration_means}.
\end{proof}

Now, we give the proof of Theorem~\ref{thm:ub_full_bandit}.

\begin{proof} (Theorem~\ref{thm:ub_full_bandit}) 
	First, we bound the number of exploration rounds up to time $T$.
	Let $\psi(t) = t^{\frac{2}{3}} / d$.
	According to the condition of exploitation (Line~\ref{line:ete_check_exploration_times} in Algorithm~\ref{alg:full_bandit}), we have that if at timestep $t$ algorithm $\algfullbandit$ starts an exploration round, then $t$ satisfies 
	$$
	N_\pi(t-1) \leq \psi(t).
	$$ 
	Let $t_0$ denote the timestep at which algorithm $\algfullbandit$ starts the last exploration round. Then, we have
	$$
	N_\pi(t_0-1) \leq \psi(t_0)
	$$
	and thus
	\begin{align*}
	N_\pi(T)= & N_\pi(t_0)
	\\
	= & N_\pi(t_0-1)+1 
	\\
	\leq & \psi(t_0)+1
	\\
	\leq & \psi(T)+1 .
	\end{align*}
	
	Next, for each timestep $t$, we bound the estimation error of $f(\bw)$.
	For any $t>0$, let $\Delta_t \triangleq f(\bw^*)-f(\bw_t)$ and $\hat{f}_{t}(\bw) \triangleq \bw^\top \bs{\hat{\theta}}_{t} - \rho \bw^\top \hat{\Sigma}_{t} \bw$.
	Suppose that event $\cM_t \cap \cN_t$ occurs. Then, according to Lemmas~\ref{lemma:full_bandit_con_covariance} and \ref{lemma:full_bandit_con_means}, we have that for any $\bw \in \triangle_{d}$,
	\begin{align*}
	\abr{ f(\bw) - \hat{f}_{t-1}(\bw)  } \leq & \sqrt{2\beta(\delta_t)} \sqrt{\bw^\top B_\pi^+ D_{t-1}^{-1} \sbr{\lambda \Lambda_{\Sigma^*_\pi} D_{t-1} + \sum_{s=1}^{N_\pi(t-1)} \Sigma^*_\pi}  D_{t-1}^{-1} (B_\pi^{+}) \top \bw} \\& + 5 \rho   \| C^+_\pi \|  \sqrt{\frac{ 3 \ln t}{2 N_\pi(t-1)}} 
	\\
	\leq & \sqrt{ 2 \ln(t \ln^2 t) + \ln \ln t + \frac{\tilde{d}}{2} \ln(1+\frac{e}{\lambda}) } \cdot \sqrt{ \bw^\top B_\pi^+ D_{t-1}^{-1} \sbr{\lambda \Lambda_{\Sigma^*_\pi} + \Sigma^*_\pi}  (B_\pi^{+}) \top \bw  } \\& + 5 \rho   \| C^+_\pi \|  \sqrt{\frac{3 \ln t}{2 N_\pi(t-1)}} 
	\\
	\leq & 7  \sqrt{  \ln t + \frac{\tilde{d}}{2} \ln(1+\frac{e}{\lambda}) } \cdot \sbr{ \sqrt{ \bw^\top B_\pi^+  \sbr{\lambda \Lambda_{\Sigma^*_\pi} + \Sigma^*_\pi}  (B_\pi^{+}) \top \bw  } + \rho \| C^+_\pi \| } \cdot \sqrt{\frac{1}{ N_\pi(t-1)}}
	\end{align*}
	Let $Z(\rho, \pi) \triangleq \max_{\bw \in \triangle_{d}} \sbr{ \sqrt{ \bw^\top B_\pi^+  \sbr{\lambda \Lambda_{\Sigma^*_\pi} + \Sigma^*_\pi}  (B_\pi^{+}) \top \bw  } + \rho \| C^+_\pi \| }$. Then, we have
	\begin{align*}
	\abr{ f(\bw) - \hat{f}_{t-1}(\bw)  } 
	\leq  7  \cdot Z(\rho, \pi)  \sqrt{ \frac{\ln t + \tilde{d} \ln(1+e/\lambda) }{ N_\pi(t-1)} } .
	\end{align*}
	
	Let $\cL^{\textup{exploit}}(t)$ denote the event that algorithm $\algfullbandit$ does the exploitation at timestep $t$. 
	For any $t>0$, if $\cL^{\textup{exploit}}(t)$ occurs, we have $N_\pi(t-1)>\psi(t)$. Thus 
	\begin{align*}
	\abr{ f(\bw) - \hat{f}_{t-1}(\bw)  } 
	< & 7  \cdot Z(\rho, \pi)  \sqrt{ \frac{\ln t + \tilde{d} \ln(1+e/\lambda)}{ \psi(t) } }
	\\
	= & 7  \cdot Z(\rho, \pi)  \sqrt{ d \sbr{\ln t + \tilde{d} \ln(1+e/\lambda)} } \cdot t^{-\frac{1}{3}}
	\end{align*}
	
	According to Lemmas~\ref{lemma:full_bandit_con_covariance},\ref{lemma:full_bandit_con_means}, for any $t \geq 2$, we bound the probability of event $\neg(\cM_t \cap \cN_t)$ by 
	\begin{align*}
	\Pr \mbr{ \neg(\cM_t \cap \cN_t) } \leq &  \frac{6 d^2}{t^2} + \frac{1}{t \ln^2 t}
	\\
	\leq &  \frac{6 d^2}{t \ln^2 t} + \frac{1}{t \ln^2 t}
	\\
	= &  \frac{7 d^2}{t \ln^2 t}
	\end{align*}
	
	The expected regret of algorithm $\algfullbandit$ can be divided into two parts, one  due to exploration and the other due to exploitation. Then, we can obtain
	\begin{align*}
	\ex[\cR(T)] \leq & N_\pi(T) \cdot \tilde{d} \Delta_{\textup{max}} + \sum_{t=1}^T \ex[ \Delta_t | \cL^{\textup{exploit}}(t)]
	\\
	\leq & (\psi(T)+1) \cdot \tilde{d} \Delta_{\textup{max}} + \sum_{t=1}^T \sbr{ \ex[ \Delta_t | \cL^{\textup{exploit}}(t), \cM_t \cap \cN_t ] + \ex[ \Delta_t | \cL^{\textup{exploit}}(t),  \neg(\cM_t \cap \cN_t)] \cdot \Pr \mbr{ \neg(\cM_t \cap \cN_t) } }
	\\
	= & O \sbr{\sbr{ \frac{T^{\frac{2}{3}}}{d} + 1} \cdot \tilde{d} \Delta_{\textup{max}} + \sum_{t=1}^T \sbr{  Z(\rho, \pi)  \sqrt{ d(\ln t + \tilde{d}\ln(1+e/\lambda) )} \cdot t^{-\frac{1}{3}}  + \Delta_{\textup{max}} \cdot  \frac{d^2}{t \ln^2 t}  } }
	\\
	= & O \sbr{T^{\frac{2}{3}} d \Delta_{\textup{max}} + d^2 \Delta_{\textup{max}} +  Z(\rho, \pi)  \sqrt{ d(\ln T + d^2 \ln(1+\lambda^{-1}))} \cdot  T^{\frac{2}{3}}  +  d^2 \Delta_{\textup{max}}}
	\\
	= & O \sbr{ Z(\rho, \pi)  \sqrt{ d(\ln T + d^2 \ln(1+\lambda^{-1})) } \cdot  T^{\frac{2}{3}}+  d \Delta_{\textup{max}} \cdot  T^{\frac{2}{3}} + d^2 \Delta_{\textup{max}} } 
	\\
	= & O \sbr{ Z(\rho, \pi)  \sqrt{ d(\ln T + d^2 \ln(1+\lambda^{-1})) } \cdot  T^{\frac{2}{3}}+   d \Delta_{\textup{max}} \cdot  T^{\frac{2}{3}} }
	\end{align*}

Choosing $\lambda=\frac{1}{2}$ and using $\Lambda_{\Sigma^*_\pi}=\Sigma^*_\pi$, we obtain
\begin{align*}
\ex[\cR(T)]  =  O \sbr{ Z(\rho, \pi)  \sqrt{ d(\ln T + d^2 ) } \cdot  T^{\frac{2}{3}}+   d \Delta_{\textup{max}} \cdot  T^{\frac{2}{3}} },
\end{align*}
where $Z(\rho, \pi) = \max_{\bw \in \triangle_{d}} \sbr{ \sqrt{ \bw^\top B_\pi^+  \Sigma^*_\pi  (B_\pi^{+}) \top \bw  } + \rho \| C^+_\pi \| }$.
\end{proof}

\restoregeometry

\end{document}